\documentclass[11pt]{article}

\usepackage{latexsym}
\usepackage{amsmath}
\usepackage{amssymb}
\usepackage{ifsym}
\usepackage{bbm}
\usepackage{amstext}
\usepackage{mathrsfs}
\usepackage{gastex}
\usepackage{color}
\usepackage{graphics}
\usepackage{graphicx}
\usepackage{algorithm}
\usepackage{algorithmic}
\usepackage{url}
\usepackage{setspace}

\newcommand{\mboxit}[1]{\mbox{\textit{#1}}}

\newcommand{\comment}[1]{}
\newcommand{\ie}{\mbox{i.e.}}
\newcommand{\cf}{\mbox{cf.}}
\newcommand{\eg}{\mbox{e.g.}}
\newcommand{\viz}{\mbox{viz.}}
\newcommand{\wrt}{\mbox{w.r.t.}}
\newcommand{\suchthat}{\mbox{s.t.}}

\newcommand{\etal}{\mboxit{et al.}}
\newcommand{\qed}{\rule{2mm}{2mm}}
\newcommand{\tuple}[1]{\langle #1 \rangle}
\newcommand{\defined}{\ensuremath{=_{\text{def}}}}
\newcommand{\stackscript}[2]{\stackrel{#1}{\scriptscriptstyle #2}}
\newcommand{\bigM}{{\ensuremath{\text{\it big}}}}

\newtheorem{definition}{Definition}[section]

\newtheorem{lemma}{Lemma}[section]
\newtheorem{theorem}{Theorem}[section] 

\newtheorem{corollary}{Corollary}[section]

\newenvironment{proof}{\noindent\textbf{Proof:}}
{\hfill\qed}

\newcommand{\et}{\wedge}
\newcommand{\ou}{\vee}
\newcommand{\imp}{\rightarrow}
\newcommand{\sii}{\leftrightarrow}
\newcommand{\symdiff}{\dot{-}} % symmetric difference

\newcommand{\valuations}[1]{{\ensuremath{\text{\it val}(#1)}}}
\newcommand{\base}[1]{{\ensuremath{\text{\it base}(#1)}}}
\newcommand{\atm}[1]{
{\ensuremath{\text{\it atm}(#1)}}
}
\newcommand{\card}[1]{
{\ensuremath{\text{\it card}(#1)}}
}

\newcommand{\Cn}[1]{{\ensuremath{\text{\it Cn}(#1)}}}
\newcommand{\IP}[1]{
{\ensuremath{\text{\it IP}(#1)}}
}
\newcommand{\RelTarget}[1]{%
{\ensuremath{\text{\it RelTarget}(#1)}}
}

\newcommand{\ElemAt}[1]{\ensuremath{\text{\it atm}(#1)}}
\newcommand{\EssentAt}[1]{\ensuremath{\text{\it atm}!(#1)}}

\newcommand{\sat}{\Vdash}
\newcommand{\notsat}{\not\sat}
\newcommand{\satMod}[1]{\sat_{\raisebox{-0.2ex}{$\!\!\!_{#1}$}}}

\newcommand{\satEssent}[1]{\sat_{\raisebox{-0.2ex}{$\!\!\!_{#1}$}}^
{\raisebox{-0.3ex}{$^{!}$}}}

\newcommand{\wMvalid}[2]{\models_{\raisebox{-0.2ex}{$\!\!\!_{#1}$}}^
{\raisebox{-0.3ex}{$\!\!\!^{#2}$}}\!\!}
\newcommand{\notwMvalid}{\not\wMvalid}
\newcommand{\Mvalid}[1]{\models^{\raisebox{-0.3ex}{$\!\!\!^{#1}$}}\!\!}
\newcommand{\notMvalid}{\not\Mvalid}
\newcommand{\Cvalid}[1]{\models_{\raisebox{-0.25ex}{$\!\!\!_{#1}$}}\!\!}

\newcommand{\provable}[1]{\vdash_{\raisebox{-0.2ex}{$\!\!\!_{#1}$}}\!\!}

\newcommand{\CPL}{{\sf CPL}}
\newcommand{\CPLvalid}{\Cvalid{\CPL}}
\newcommand{\notCPLvalid}{\not\CPLvalid}
\newcommand{\CPLtheorem}{\provable{\CPL}}
\newcommand{\notCPLtheorem}{\not\CPLtheorem}

\newcommand{\Kn}{{\sf K}_{n}}

\newcommand{\PDL}{{\sf PDL}}

\newcommand{\PDLvalid}{\Cvalid{\PDL}}
\newcommand{\notPDLvalid}{\not\PDLvalid}
\newcommand{\PDLtheorem}{\provable{\PDL}}
\newcommand{\notPDLtheorem}{\not\PDLtheorem}

\newcommand{\act}{\ensuremath{\text{\it a}}} % an action
\newcommand{\prp}{\ensuremath{\text{\it p}}} % an atomic proposition
\newcommand{\lit}{\ell} % a literal
\newcommand{\fml}{\ensuremath{\varphi}} % a classical formula
\newcommand{\modfml}{\ensuremath{\varPhi}} % a modal formula
\newcommand{\term}{\ensuremath{\pi}} % a term
\newcommand{\compatterm}{\overline{\atm{\term}}}

\newcommand{\poss}[1]{\langle #1 \rangle}
\newcommand{\nec}[1]{[ #1 ]}
\newcommand{\antec}{\ensuremath{\varphi}}  % the antecedent in a generic formula
\newcommand{\conseq}{\ensuremath{\psi}} % the consequent in a generic formula

\newcommand{\Worlds}{\ensuremath{\text{\it W}}}
\newcommand{\AccRel}{\ensuremath{\text{\it R}}}
\newcommand{\AccRelS}{\ensuremath{\text{\it S}}}
\newcommand{\TransStruct}[1]{\tuple{\Worlds_{#1},\AccRel_{#1}}}
\newcommand{\TransStructP}[1]{\tuple{\Worlds'_{#1},\AccRel'_{#1}}}
\newcommand{\TransStructS}[1]{\tuple{\Worlds''_{#1},\AccRel''_{#1}}}
\newcommand{\ract}{\AccRel_{\act}}
\newcommand{\model}{\mathscr{M}}
\newcommand{\ModelSet}{\ensuremath{\mathcal{M}}}
\newcommand{\val}{\ensuremath{\text{\it v}}} % propositional valuation
\newcommand{\submodel}[1]{\preceq_{#1}}

\newcommand{\Act}{\ensuremath{\mathfrak{Act}}}
\newcommand{\Prp}{\ensuremath{\mathfrak{Prop}}}
\newcommand{\Lit}{\ensuremath{\mathfrak{Lit}}}
\newcommand{\Fml}{\ensuremath{\mathfrak{Fml}}}

\newcommand{\STAT}[2]{\ensuremath{\mathcal{S}^{#1}_{#2}}}
\newcommand{\EFF}[2]{\ensuremath{\mathcal{E}^{#1}_{#2}}}
\newcommand{\EXE}[2]{\ensuremath{\mathcal{X}^{#1}_{#2}}}
\newcommand{\Theory}[1]{\ensuremath{\mathcal{T}_{#1}}}

\newcommand{\buy}{\ensuremath{\text{\it buy}}}
\newcommand{\Token}{\ensuremath{\text{\it token}}}
\newcommand{\Coffee}{\ensuremath{\text{\it coffee}}}
\newcommand{\Hot}{\ensuremath{\text{\it hot}}}

\newcommand{\RevisionModels}[2]{{#1}^{\proprevision}_{#2}} % models of revision
\newcommand{\ErasureModels}[2]{{#1}^{-}_{#2}} % models of erasure
\newcommand{\ErasureSet}[2]{{#1}^{-}_{#2}} % erased set of formulas

\newcommand{\revise}[2]{\ensuremath{\text{\it revise}(#1,#2)}}
\newcommand{\erasure}[2]{\ensuremath{\text{\it contract}(#1,#2)}}
\newcommand{\propcontract}{\ominus}
\newcommand{\proprevision}{\star}

\newcommand{\Abuy}{\ensuremath{\text{\it b}}}
\newcommand{\AToken}{\ensuremath{\text{\it t}}}
\newcommand{\ACoffee}{\ensuremath{\text{\it c}}}
\newcommand{\AHot}{\ensuremath{\text{\it h}}}

\newcommand{\myskip}{\medskip} % space interlines
\newcommand{\Set}{\ensuremath{\text{\it A}}}

\newcommand{\pspace}{{\sc pspace}}

\begin{document}

\title{Action Theory Evolution}
%\title{Evolving Action Domain Descriptions}
%\title{On the Evolution of Action Domain Descriptions}
\date{}

\author{Ivan Jos\'{e} Varzinczak\\
       Meraka Institute\\
       CSIR Pretoria, South Africa \\
       \url{ivan.varzinczak@meraka.org.za}}

% For research notes, remove the comment character in the line below.
% \researchnote

\maketitle

\begin{abstract}
Like any other logical theory, domain descriptions in reasoning about actions may evolve, and thus need revision methods to adequately accommodate new information about the behavior of actions. The present work is about changing action domain descriptions in propositional dynamic logic. Its contribution is threefold: first we revisit the semantics of action theory contraction that has been done in previous work, giving more robust operators that express minimal change based on a notion of distance between Kripke-models. Second we give algorithms for syntactical action theory contraction and establish their correctness \wrt\ our semantics. Finally we state postulates for action theory contraction and assess the behavior of our operators \wrt\ them. Moreover, we also address the revision counterpart of action theory change, showing that it benefits from our semantics for contraction.
\end{abstract}

\newpage

\thispagestyle{empty}
\vspace*{-1.2cm}

\tableofcontents

\newpage

%%%%%%%%%%%%%%%%%%%%%%%%%%%%%%%%%%%%%%%%%%%%%%%%%%%
\section{Introduction}\label{Introduction}
%%%%%%%%%%%%%%%%%%%%%%%%%%%%%%%%%%%%%%%%%%%%%%%%%%%

	Consider an intelligent agent designed to perform rationally in a dynamic world, and suppose she should reason about the dynamics of an automatic coffee machine (Figure~\ref{RobotCoffeeMachine}). Suppose, for example, that the agent believes that coffee is always a hot beverage. Suppose now that some day she gets a coffee and observes that it is cold. In such a case, the agent must change her beliefs about the relation between the propositions ``I hold a coffee'' and ``I hold a hot beverage''. This example is an instance of the problem of changing propositional belief bases and is largely addressed in the literature about belief change~\cite{Gardenfors88} and belief update~\cite{KatsunoMendelzon92}.

\begin{figure}[htbp]
\begin{center}
\includegraphics[height=5cm]{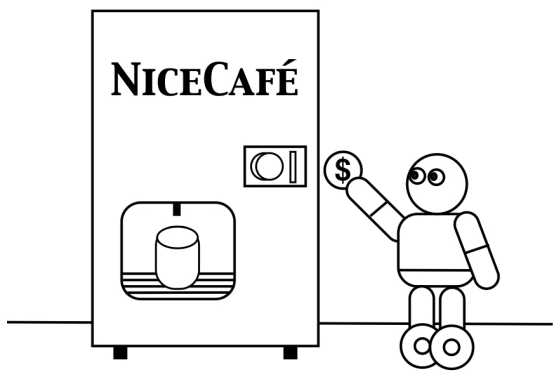}
\caption{The coffee deliverer agent.}
\label{RobotCoffeeMachine}
\end{center}
\end{figure}

	Next, let our agent believe that whenever she buys a coffee from the machine, she gets a hot beverage. This means that in every state of the world that follows the execution of buying a coffee, the agent possesses a hot beverage. Then, in a situation where the machine is running out of cups, after buying, the coffee runs through the shelf and the agent does not hold a hot beverage in her hands.

	Imagine now that the agent never considered any relation between buying a coffee on the machine and its service availability, in the sense that the agent always believed that buying does not prevent other users from using the machine. Nevertheless, someday our agent is queuing to buy a coffee and observes that just after the agent before her has bought, the machine went out of order (maybe due to a lack of coffee powder).

	Completing our agent's struggle in discovering the intricacies of a coffee machine, suppose she always believed that if she has a token, then it is possible to buy coffee, provided that some other conditions like being close enough to the button, having a free hand, etc, are satisfied. However, during a blackout,  the agent, even with a token, does not manage to buy her coffee.

\myskip

	The last three examples illustrate situations where changing the beliefs about the {\em behavior} of the action of buying coffee is mandatory. In the first one, buying coffee, once believed to be deterministic, has now to be seen as nondeterministic, or alternatively to have a different outcome in a more specific context (\eg\ if there is no cup in the machine). In the second example, buying a coffee is now known to have side-effects (ramifications) one was not aware of. Finally, in the last example, the executability of the action under concern is questioned in the light of new information showing a context that was not known to preclude its execution.

	Such cases of theory change are very important when one deals with logical descriptions of 
dynamic domains: it may always happen that one discovers that an action actually has a behavior 
that is different from that one has always believed it had.

\myskip

	Up to now, theory change has been studied mainly for knowledge bases in classical logics, both in terms of revision and update. Since the work by Fuhrmann~\cite{Fuhrmann89}, only in a few recent studies has it been considered in the realm of modal logics, \viz\ in epistemic logic~\cite{Hansson99} and in dynamic logics~\cite{HerzigEtAl-ECAI06}. Recently some studies have investigated revision of beliefs about {\em facts} of the world~\cite{ShapiroEtAl00,JinThielscher05} or the agent's goals~\cite{ShapiroEtAl-IJCAI05}. In our scenario, this would concern for instance the truth of $\Token$ in a given state: the agent believes that she has a token, but is actually wrong about that. Then she might subsequently be forced to revise her beliefs about the current state of affairs or change her goals according to what she can perform in that state. Such belief revision operations do not modify the agent's beliefs about the {\em action laws}. In opposition to that, here we are interested exactly in such modifications. Starting with Baral and Lobo's work~\cite{BaralLobo-IJCAI97}, some recent studies have been done on that issue~\cite{EiterEtAl-IJCAI05,EiterEtAl-ECAI06} for domain descriptions in action languages~\cite{GelfondLifschitz93}.
	
	We here take a step further in this direction and propose a method based on that given by Herzig \etal~\cite{HerzigEtAl-ECAI06} that is more robust by integrating a notion of {\em minimal change} and complying with {\em postulates} of theory change.

\myskip

	The present text is structured as follows: in Section~\ref{Preliminaries} we establish the formal background that will be used throughout this work.
Sections~\ref{SemanticsErasure}--\ref{Correctness} are the core of the work: in Section~\ref{SemanticsErasure} we present the central definitions for a semantics of action theory change, 
Section~\ref{SyntaxErasure} is devoted to its syntactical counterpart while
Section~\ref{Correctness} to the proof of its correspondence with the semantics. In 
Section~\ref{Postulates} we discuss some postulates for contraction/erasure and then present a semantics for action theory revision (Section~\ref{SemanticsRevision}). In Section~\ref{RelatedWork} we address existing work in the field. After making some comments on our method (Section~\ref{Comments}), we finish with some conclusions and future directions of research.

%%%%%%%%%%%%%%%%%%%%%%%%%%%%%%%%%%%%%%%%%%%%%%%%%%%
\section{Logical Preliminaries}\label{Preliminaries}
%%%%%%%%%%%%%%%%%%%%%%%%%%%%%%%%%%%%%%%%%%%%%%%%%%%

	Following the tradition in the reasoning about actions (RAA) community, we consider action theories to be finite collections of statements that have the particular form:
\begin{itemize}
\item if {\em context}, then {\em effect} after {\em every execution} of {\em action} (effect laws);
\item if {\em precondition}, then {\em action executable} (executability laws).
\end{itemize}
Statements mentioning no action at all represent laws about the underlying structure of the world, \ie, its possible states (static laws).

	Several logical frameworks have been proposed to formalize such statements. Among the most prominent ones are the Situation Calculus~\cite{McCarthyHayes69,Reiter91}, the family of Action Languages~\cite{GelfondLifschitz93,KarthaLifschitz94,GiunchigliaEtAl97}, the Fluent Calculus~\cite{Thielscher95,Thielscher97}, and the dynamic logic-based approaches~\cite{DeGiacomoLenzerini95,CastilhoEtAl99,ZhangFoo01}. Here we opt to formalize action theories using a version of Propositional Dynamic Logic (\PDL)~\cite{HarelEtAl00}.

\subsection{Action Theories in Dynamic Logic}\label{ActTheoriesPDL}

	Let $\Act=\{\act_{1},\act_{2},\ldots\}$ be the set of all {\em atomic action constants} of a given domain. An example of atomic action is $\buy$. To each atomic action $\act$ there is associated a modal operator $\nec\act$.\footnote{We here suppose that our multimodal logic is independently axiomatized~\cite{KrachtWolter91}, \ie, the logic is a fusion and there is no interaction between the modal operators. This is a requirement to achieve modularity of action theories~\cite{HerzigVarzinczak-Aiml04Proc05} (see further).}

	$\Prp=\{\prp_{1},\prp_{2},\ldots\}$ denotes the set of all {\em propositional constants}, also called {\em fluents} or {\em atoms}. Examples of those are $\Token$ (``the agent has a token'') and $\Coffee$ (``the agent holds a coffee''). The set of all literals is $\Lit=\{\lit_{1},\lit_{2},\ldots\}$, where each $\lit_{i}$ is either $\prp$ or $\neg\prp$, for some $\prp\in\Prp$. If $\lit=\neg\prp$, then we identify $\neg\lit$ with $\prp$. By $|\lit|$ we denote the atom in $\lit$.

	We use small Greek letters $\fml,\psi,\ldots$ to denote {\em Boolean formulas}. They are 
recursively defined in the usual way:
\[
\fml\ ::= p\ |\ \top\ |\ \bot\ |\ \neg\fml\ |\ \fml\et\fml\ |\ \fml\ou\fml\ |\
\fml\imp\fml\ |\ \fml\sii\fml\
\]
$\Fml$ is the set of all Boolean formulas. An example of a Boolean formula is $\Coffee\imp\Hot$. A propositional valuation $\val$ is a {\em maximally consistent} set of literals. We denote by $\val\sat\fml$ the fact that $\val$ satisfies a propositional formula $\fml$. By $\valuations{\fml}$ we denote the set of all valuations satisfying $\fml$. $\CPLvalid$ denotes the classical consequence relation. $\Cn{\fml}$ denotes all logical consequences of $\fml$ in classical propositional logic.

	If $\fml$ is a propositional formula, $\ElemAt{\fml}$ denotes the set of elementary atoms {\em actually} occurring in $\fml$. For example, $\ElemAt{\neg\prp_{1}\et(\neg\prp_{1}\ou\prp_{2})}=\{\prp_{1},\prp_{2}\}$.

\myskip

	For $\fml$ a Boolean formula, $\IP{\fml}$ denotes the set of its {\em prime implicants}~
\cite{Quine52}, \ie, the weakest terms (conjunctions of literals) that imply $\fml$. As an example, $\IP{\prp_{1}\oplus\prp_{2}}=\{\prp_{1}\et\neg\prp_{2},\neg\prp_{1}\et\prp_{2}\}$. For more on prime implicants, their properties and how to compute them see the chapter by Marquis~\cite{Marquis2000}. By $\term$ we denote a prime implicant, and given $\lit$ and $\term$, $\lit\in\term$ abbreviates `$\lit$ is a literal of $\term$'.

\myskip

	We denote complex formulas (possibly with modal operators) by $\modfml,\varPsi,\ldots$ They are 
recursively defined in the following way:
\[
\modfml\ ::= \fml\ |\ \nec\act\modfml\ |\ \neg\modfml\ |\ \modfml\et\modfml\ |\
\modfml\ou\modfml\ |\ \modfml\imp\modfml\ |\ \modfml\sii\modfml\ 
\]$\poss\act$ is the dual operator of $\nec\act$, defined as $\poss\act\modfml\defined\neg\nec\act
\neg\modfml$. An example of a complex formula is $\neg\Coffee\imp\nec\buy\Coffee$.

\myskip

  The semantics is that of \PDL\ without the $*$ operator, which amounts to multimodal logic~$\Kn$~\cite{Popkorn94}. In the following we will refer to \PDL\ but our underlying logical formalism is essentially the simpler multimodal logic $\Kn$, which turns out to be expressive enough for our purposes here.

\begin{definition}[\PDL-model]
A {\em \PDL-model} is a tuple $\model=\TransStruct{}$ where $\Worlds$ is a set of valuations (also called possible worlds), and $\AccRel$ maps action constants $\act$ to accessibility relations $\ract\subseteq\Worlds\times\Worlds$.
\end{definition}

  As an example, for $\Act=\{\act_{1},\act_{2}\}$ and $\Prp=\{\prp_{1},
\prp_{2}\}$, we have the \PDL-model $\model=\TransStruct{}$, where
\[
\Worlds=\{\{\prp_{1},\prp_{2}\},\{\prp_{1},\neg\prp_{2}\},\{\neg\prp_{1},
\prp_{2}\}\},
\]\[
\AccRel(\act_{1})=
\left\{
   \begin{array}{c}
     (\{\prp_{1},\prp_{2}\},\{\prp_{1},\neg\prp_{2}\}),
     (\{\prp_{1},\prp_{2}\},\{\neg\prp_{1},\prp_{2}\}), \\
     (\{\prp_{1},\neg\prp_{2}\},\{\prp_{1},\neg\prp_{2}\}),
     (\{\prp_{1},\neg\prp_{2}\},\{\neg\prp_{1},\prp_{2}\})
   \end{array}                     
\right\}
\]\[
\AccRel(\act_{2})=\{(\{\prp_{1},\prp_{2}\},\{\neg\prp_{1},\prp_{2}\}),
(\{\neg\prp_{1},\prp_{2}\},\{\neg\prp_{1},\prp_{2}\})\}
\]

Figure~\ref{PDL-model} gives a graphical representation of $\model$.\footnote{Notice that our notion of \PDL-model does not follow the standard notion from modal logics: here no two worlds satisfy the same valuation. This is a pragmatic choice (see~Section~\ref{SyntaxErasure}). Nevertheless, all we are about to state in the sequel can be straightforwardly formulated for standard \PDL\ models as well.}

\begin{figure}[h]
\begin{center}
\parbox{2cm}{
$\model$ :
}\parbox{4cm}{
\scriptsize{
\begin{picture}(40,42)(0,0) % (0,0) Ž embaixo ˆ esquerda, (50,50) Ž acima ˆ direita
    \gasset{Nw=9,Nh=4,linewidth=0.3}
    \thinlines
    \node(A1)(0,30){$\prp_{1},\prp_{2}$}
    \node[Nw=10](A2)(30,30){$\prp_{1},\neg\prp_{2}$}
    \node[Nw=10](A3)(15,5){$\neg\prp_{1},\prp_{2}$}
    \drawedge[ELside=l](A1,A2){$\act_{1}$}
    \drawedge[curvedepth=-3,ELside=r](A1,A3){$\act_{1}$}
    \drawedge[curvedepth=3,ELside=l](A1,A3){$\act_{2}$}
    \drawedge[ELside=l](A2,A3){$\act_{1}$}
    \drawloop[loopangle=90](A2){$\act_{1}$}
    \drawloop[loopangle=270](A3){$\act_{2}$}
\end{picture}
} % End \scriptsize
} % End \parbox
\end{center}
\caption{Example of a \PDL-model for $\Act=\{\act_{1},\act_{2}\}$, and $\Prp=\{\prp_{1},\prp_{2}\}
$.}
\label{PDL-model}
\end{figure}
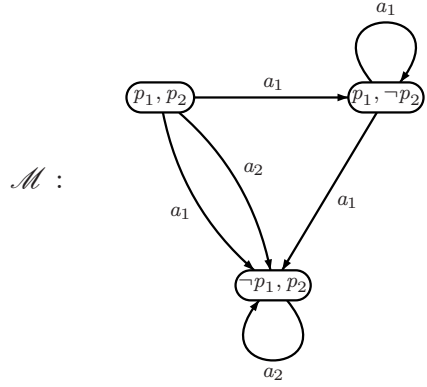

\begin{definition}[Truth conditions]
Given a \PDL-model $\model=\TransStruct{}$,
\begin{itemize}
\item $\wMvalid{w}{\model}\prp$ ($\prp$ is true at world $w$ of model $\model$) if $w\sat\prp$ (the valuation $w$ satisfies $\prp$, \ie, $\prp\in w$);
\item $\wMvalid{w}{\model}\nec\act\modfml$ if $\wMvalid{w'}{\model}\modfml$ for every $w'$\ \suchthat\ $(w,w')\in\ract$;
\item $\wMvalid{w}{\model}\modfml\et\varPsi$ if $\wMvalid{w}{\model}\modfml$ and  $\wMvalid{w}{\model}\varPsi$;
\item $\wMvalid{w}{\model}\modfml\ou\varPsi$ if $\wMvalid{w}{\model}\modfml$ or  $\wMvalid{w}{\model}\varPsi$, or both;
\item $\wMvalid{w}{\model}\neg\modfml$ if $\notwMvalid{w}{\model}\modfml$, \ie, not $\wMvalid{w}{\model}\modfml$;
\item truth conditions for the other connectives are as usual.
\end{itemize}
\end{definition}

	By $\ModelSet$ we will denote a set of \PDL-models.

\myskip

	A \PDL-model $\model$ is a model of $\modfml$ (denoted $\Mvalid{\model}\modfml$) if and only if for all $w\in\Worlds$, $\wMvalid{w}{\model}\modfml$. In the model depicted in Figure~\ref{PDL-model}, we have $\Mvalid{\model}\prp_{1}\imp\nec{\act_{2}}\prp_{2}$ and $\Mvalid{\model}\prp_{1}\ou\prp_{2}$.

\begin{definition}[Global consequence]
$\model$ is a model of a set of formulas $\Sigma$ (noted $\Mvalid{\model}\Sigma$) if and only if 
$\Mvalid{\model}\modfml$ for every $\modfml\in\Sigma$. A formula $\modfml$ is a {\em 
consequence of a set of global axioms} $\Sigma$ in the class of all \PDL-models (noted $\Sigma\PDLvalid\modfml$) if and only if for every \PDL-model $\model$, if $\Mvalid{\model}\Sigma$, then $\Mvalid{\model}\modfml$.
\end{definition}

\myskip

	With \PDL\ we can state laws describing the behavior of actions. One way of doing this is by stating some formulas as global axioms.\footnote{An alternative to that is given by Castilho~\etal~\cite{CastilhoEtAl99}, with laws being stated with the aid of an extra universal modality and local consequence being thus considered.} As usually done in the RAA community, we here distinguish three types of laws. The first kind of statements are {\em static laws}, which are Boolean formulas that must hold in every possible state of the world.

\begin{definition}[Static Law]
A {\em static law} is a formula $\fml\in\Fml$.
\end{definition}

\noindent An example of a static law is $\Coffee\imp\Hot$, saying that if the agent holds a coffee, 
then 
she holds a hot beverage. The set of all static laws of a domain is denoted by $\STAT{}{}
\subseteq
\Fml$. In our example we will have $\STAT{}{}=\{\Coffee\imp\Hot\}$.

\myskip

	The second kind of action law we consider is given by the {\em effect laws}. These are 
formulas relating an action to its effects, which can be~conditional.

\begin{definition}[Effect Law]
An {\em effect law for action} $\act$ is of the form $\antec\imp\nec\act\conseq$, where $\antec,
\conseq\in\Fml$.
\end{definition}

\noindent The consequent $\conseq$ is the effect which always obtains when action $\act$ is executed in a state where the antecedent $\antec$ holds. If \act\ is a nondeterministic action, then the consequent $\conseq$ is typically a disjunction. An example of an effect law is $\neg\Coffee\imp\nec\buy\Coffee$, saying that in a situation where the agent has no coffee, after buying, the agent has a coffee. If $\conseq$ is inconsistent, then we have a special kind of effect law that we call an {\em inexecutability law}. For example, we could also have $\neg\Token\imp\nec\buy\bot$, expressing that $\buy$ cannot be executed if the agent has no token.

	The set of effect laws of a domain is denoted by $\EFF{}{}$. In our coffee machine scenario, we could have for example:
\[
\EFF{}{}=\left\{
  \begin{array}{c}
  	\neg\Coffee\imp\nec\buy\Coffee, \Token\imp\nec\buy\neg\Token, \neg\Token\imp\nec\buy
\bot
  \end{array}
  \right\}
\]

\myskip

	Finally, we also define {\em executability laws}, which stipulate the context where an action is guaranteed to be executable. In \PDL, the operator $\poss\act$ is used to express executability. $\poss\act\top$ thus reads ``the execution of~$\act$~is~possible''.

\begin{definition}[Executability Law]
An {\em executability law for action} $\act$ is of the form $\antec\imp\poss\act\top$, where $\antec\in\Fml$.
\end{definition}

\noindent For instance, $\Token\imp\poss\buy\top$ says that buying can be executed whenever the agent has a token. The set of all executability laws of a given domain is denoted by $\EXE{}{}$. In our scenario example we would have $\EXE{}{}=\{\Token\imp\poss\buy\top\}$.

\myskip

	With our three basic types of laws, we are able to define action theories:

\begin{definition}[Action Theory]
Given a domain and any (possibly empty) sets of laws $\STAT{}{}$, $\EFF{}{}$, and $\EXE{}{}$, $\Theory{}=\STAT{}{}\cup\EFF{}{}\cup\EXE{}{}$ is an {\em action theory}.
\end{definition}

	For given action $\act$, $\EFF{}{\act}$ (resp.\ $\EXE{}{\act}$) will denote the set of only those effect (resp.\ executability) laws about $\act$. $\Theory{\act}=\STAT{}{}\cup\EFF{}{\act}\cup\EXE{}{\act}$ is then the action theory for $\act$.\footnote{Notice that for $\act_{1},\act_{2}\in\Act$, $\act_{1}\neq\act_{2}$, the intuition is indeed that $\Theory{\act_{1}}$ and $\Theory{\act_{2}}$ overlap only on $\STAT{}{}$, \ie, the only laws that are common to both $\Theory{\act_{1}}$ and $\Theory{\act_{2}}$ are the laws about the structure of the world. This requirement is somehow related with the logic being independently axiomatized (see above).}

\myskip

	For the sake of clarity, we abstract here from the frame and ramification problems, and suppose the agent's theory already entails all the relevant frame axioms. We could have used any suitable solution to the frame problem, like \eg\ the dependence relation~\cite{CastilhoEtAl99}, which is used in the work of Herzig \etal~\cite{HerzigEtAl-ECAI06}, or a kind of successor state axioms in a slightly modified setting~\cite{DemolombeEtAl-JANCL03}. To make the presentation more clear to the reader, here we do not bother with a solution to the frame problem and just assume all frame axioms can be inferred from the theory. Actually we can suppose that all intended frame axioms are automatically recovered and stated in the theory, more specific, in the set of effect laws.\footnote{Frame axioms are a special type of effect law, having the form $\lit\imp\nec\act\lit$, for $\lit\in\Lit$.} Hence the action theory of our example will be:
\[
\Theory{}=\left\{
  \begin{array}{c}
	 \Coffee\imp\Hot, \Token\imp\poss\buy\top, \\
	  \neg\Coffee\imp\nec\buy\Coffee, \\
	 \Token\imp\nec\buy\neg\Token, \neg\Token\imp\nec\buy\bot, \\
	 \Coffee\imp\nec\buy\Coffee, \Hot\imp\nec\buy\Hot
  \end{array}
  \right\}
\]
(We have not stated the frame axiom $\neg\Token\imp\nec\buy\neg\Token$ because it can be trivially deduced from the inexecutability law $\neg\Token\imp\nec\buy\bot$.)

Figure~\ref{ExModelCoffeeMachine} below shows a \PDL-model for the theory $\Theory{}$.

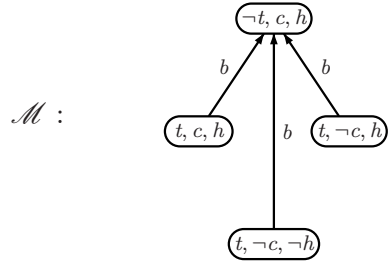
\begin{figure}[h]
\begin{center}
\parbox{1.5cm}{
$\model$ :
}\parbox{4cm}{
\scriptsize{
\begin{picture}(35,35)(0,0) % (0,0) Ž embaixo ˆ esquerda, (50,50) Ž acima ˆ direita
    \gasset{Nw=9,Nh=4,linewidth=0.3}
    \thinlines
    \node(A1)(10,15){$\AToken,\ACoffee,\AHot$}
    \node[Nw=10](A2)(20,30){$\neg\AToken,\ACoffee,\AHot$}
    \node[Nw=10](A3)(30,15){$\AToken,\neg\ACoffee,\AHot$}
    %\node[Nw=13](A4)(0,0){$\neg\AToken,\neg\ACoffee,\neg\AHot$}
    %\node[Nw=12](A5)(40,0){$\neg\AToken,\neg\ACoffee,\AHot$}
    \node[Nw=12](A6)(20,0){$\AToken,\neg\ACoffee,\neg\AHot$}
    \drawedge[ELside=l](A1,A2){$\Abuy$}
    \drawedge[ELside=r](A3,A2){$\Abuy$}
    \drawedge[ELside=r](A6,A2){$\Abuy$}
%    \drawedge[curvedepth=-3,linewidth=0.4,ELside=r](A1,A3){$\act_{1}$}
%    \drawedge[curvedepth=3,linewidth=0.4,ELside=l](A1,A3){$\act_{2}$}
%    \drawedge[linewidth=0.4,ELside=l](A2,A3){$\act_{1}$}
%    \drawloop[linewidth=0.4,loopangle=90](A2){$\act_{1}$}
%    \drawloop[linewidth=0.4,loopangle=270](A3){$\act_{2}$}
\end{picture}
} % End \scriptsize
} % End \parbox
\end{center}
\caption{A model for our coffee machine scenario: \Abuy, \AToken, \ACoffee, and \AHot\ stand for, respectively, \buy, \Token, \Coffee, and \Hot.}
\label{ExModelCoffeeMachine}
\end{figure}

	Given an action theory $\Theory{}$, sometimes it will be useful to consider models whose possible worlds are {\em all} the possible worlds allowed by $\Theory{}$:

\begin{definition}[Big Model]
Let $\Theory{}=\STAT{}{}\cup\EFF{}{}\cup\EXE{}{}$ be an action theory. $\model_{\bigM}=\tuple{\Worlds_{\bigM},\AccRel_{\bigM}}$ is the {\em big model} of $\Theory{}$ if and only if:
\begin{itemize}
\item $\Worlds_{\bigM}=\valuations{\STAT{}{}}$; and
\item $\AccRel_{\bigM}=\bigcup_{\act\in\Act}\ract$ \suchthat\ $\ract=\{(w,w') : \text{ for all }\antec\imp\nec\act\conseq\in\EFF{}{\act},\text{ if } \wMvalid{w}{\model}\antec,\text{ then }\wMvalid{w'}{\model}\conseq\}$.
\end{itemize}
\end{definition}

	Figure~\ref{ExBigModelCoffeeMachine} below shows the big model of $\Theory{}$.

\begin{figure}[h]
\begin{center}
\parbox{1.5cm}{
$\model$ :
}\parbox{4cm}{
\scriptsize{
\begin{picture}(35,35)(0,0) % (0,0) Ž embaixo ˆ esquerda, (50,50) Ž acima ˆ direita
    \gasset{Nw=9,Nh=4,linewidth=0.3}
    \thinlines
    \node(A1)(10,15){$\AToken,\ACoffee,\AHot$}
    \node[Nw=10](A2)(20,30){$\neg\AToken,\ACoffee,\AHot$}
    \node[Nw=10](A3)(30,15){$\AToken,\neg\ACoffee,\AHot$}
    \node[Nw=13](A4)(0,0){$\neg\AToken,\neg\ACoffee,\neg\AHot$}
    \node[Nw=12](A5)(40,0){$\neg\AToken,\neg\ACoffee,\AHot$}
    \node[Nw=12](A6)(20,0){$\AToken,\neg\ACoffee,\neg\AHot$}
    \drawedge[ELside=l](A1,A2){$\Abuy$}
    \drawedge[ELside=r](A3,A2){$\Abuy$}
    \drawedge[ELside=r](A6,A2){$\Abuy$}
%    \drawedge[curvedepth=-3,linewidth=0.4,ELside=r](A1,A3){$\act_{1}$}
%    \drawedge[curvedepth=3,linewidth=0.4,ELside=l](A1,A3){$\act_{2}$}
%    \drawedge[linewidth=0.4,ELside=l](A2,A3){$\act_{1}$}
%    \drawloop[linewidth=0.4,loopangle=90](A2){$\act_{1}$}
%    \drawloop[linewidth=0.4,loopangle=270](A3){$\act_{2}$}
\end{picture}
} % End \scriptsize
} % End \parbox
\end{center}
\caption{The big model for the coffee machine scenario.}
\label{ExBigModelCoffeeMachine}
\end{figure}
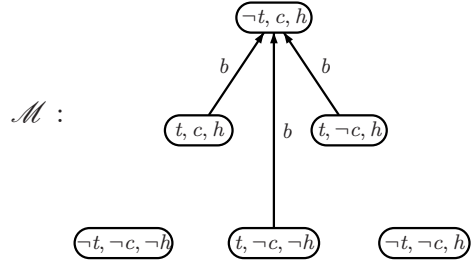

\subsection{Essential Atoms}

	An atom $\prp$ is {\em essential} to a formula $\fml$ if and only if $\prp\in
\ElemAt{\fml'}$ for every $\fml'$ such that $\CPLvalid\fml\sii\fml'$. For instance, $\prp_{1}$ is 
essential to $\neg\prp_{1}\et(\neg\prp_{1}\ou\prp_{2})$. Given $\fml$, $\EssentAt{\fml}$ denotes the set of essential atoms of $\fml$. (If $\fml$ is not contingent, \ie, $\fml$ is a tautology or a contradiction, then $\EssentAt{\fml}=\emptyset$.)

	Given $\fml$ a Boolean formula, $\fml*$ is the set of all formulas $\fml'$ such that $\fml\CPLvalid\fml'$ and $\ElemAt{\fml'}\subseteq\EssentAt{\fml}$. For instance, $\prp_{1}\ou\prp_{2}\notin\prp_{1}*$, as $\prp_{1}\CPLvalid\prp_{1}\ou\prp_{2}$ but $\ElemAt{\prp_{1}\ou\prp_{2}}\not\subseteq\EssentAt{\prp_{1}}$. Clearly, $\ElemAt{\bigwedge\fml*}=\EssentAt{\bigwedge\fml*}$, moreover whenever $\CPLvalid\fml\sii\fml'$ is the case, then $\EssentAt{\fml}=\EssentAt{\fml'}$ and also $\fml*=\fml'*$.
	
\begin{theorem}[Least atom-set theorem~\cite{Parikh99}]
Given $\fml$ a propositional formula, $\CPLvalid\fml\sii\bigwedge\fml*$, and for every $\fml'$\ \suchthat\ $\CPLvalid\fml\sii\fml'$, $\ElemAt{\fml*}\subseteq\ElemAt{\fml'}$.
\end{theorem}

	A proof of this theorem is given by Makinson~\cite{Makinson07} and we do not state it here. Essentially, the theorem establishes that for every formula $\fml$, there is a unique least set of elementary atoms such that $\fml$ may equivalently be expressed using only letters from that set.\footnote{The dual notion, \ie, that of redundant atoms is also addressed in the literature~\cite{HerzigRifi-AIJ99}, with similar purposes.} Hence, $\Cn{\fml}=\Cn{\fml*}$.

\subsection{Prime Valuations}

	Given a valuation $\val$, $\val'\subseteq\val$ is a {\em subvaluation}. Given a set of valuations $\Worlds$, a subvaluation $\val'$ {\em satisfies} a propositional formula $\fml$ modulo $\Worlds$ (noted $\val'\satMod{\Worlds}\fml$) if and only if $\val\sat\fml$ for all $\val\in\Worlds$ such that $\val'\subseteq\val$.
	
	We say that a subvaluation $\val$ {\em essentially satisfies} $\fml$ (modulo $\Worlds$), noted $\val\satEssent{\Worlds}\fml$, if and only if $\val\satMod{\Worlds}\fml$ and $\{ |\lit| : \lit\in\val \}\subseteq\EssentAt{\fml}$. If $\val\satEssent{\Worlds}\fml$, we call $\val$ an {\em essential 
subvaluation} of $\fml$ (modulo $\Worlds$).

\begin{definition}[Prime Subvaluation]
Let $\fml$ be a propositional formula and $\Worlds$ a set of valuations. A subvaluation $\val$ is a {\em prime subvaluation} of $\fml$ (modulo $\Worlds$) if and only if $\val\satEssent{\Worlds}\fml$ and there is no $\val'\subseteq\val$\ \suchthat\ $\val'\satEssent{\Worlds}\fml$.
\end{definition}

	Our notion of prime subvaluation is closely related to Veltman's definition of {\em basis} for a formula~\cite{Veltman05}.\footnote{The author is indebted to Andreas Herzig for pointing this out.} A prime subvaluation of a formula $\fml$ is thus one of the weakest states of truth in which $\fml$ is true. Hence, prime subvaluations are just another way of seeing prime implicants~\cite{Quine52} of $\fml$. By $\base{\fml,\Worlds}$ we will denote the set of all prime subvaluations of $\fml$ modulo $\Worlds$.

\begin{theorem}
Let $\fml\in\Fml$ and $\Worlds$ be a set of valuations. Then for all $w\in\Worlds$, $w\sat\fml$ if and only if $w\sat\bigvee_{\val\in\base{\fml,\Worlds}}\bigwedge_{\lit\in\val}\lit$.
\end{theorem}

\begin{proof} 
Right to left direction is straightforward. For the left to right direction, if $w\sat\fml$, then $w\sat\fml*$. Let $w'\subseteq w$ be the least subset of $w$ still satisfying $\fml*$. Clearly, $w'$ is a prime subvaluation of $\fml$ modulo $\Worlds$, and then because $w\sat\bigwedge_{\lit\in w'}\lit$, the result follows.
\end{proof}

\subsection{Closeness between Models}\label{Closeness}

	 When contracting a formula from a model, we will perform a change in its structure. Because there can be several different ways of modifying a model (not all of them minimal), we need a notion of distance between models to identify those that are closest to the original one.
	 
	As we are going to see in more depth in what follows, changing a model amounts to modifying its possible worlds or its accessibility relation. Hence, the distance between two \PDL-models will depend upon the distance between their sets of worlds and accessibility relations. These here will be based on the {\em symmetric difference} between sets, defined as $X\symdiff Y=(X\setminus Y)\cup(Y\setminus X)$.

\begin{definition}[Closeness between \PDL-Models]\label{ClosenessPDLModels}
Let $\model=\TransStruct{}$ be a model. Then $\model'=\TransStructP{}$ is {\em at least as close to} $\model$ as $\model''=\TransStructS{}$, noted $\model'\submodel{\model}\model''$, if~and~only~if
\begin{itemize}
\item either $\Worlds\symdiff\Worlds'\subseteq\Worlds\symdiff\Worlds''$
\item or $\Worlds\symdiff\Worlds'=\Worlds\symdiff\Worlds''$ and $\AccRel\symdiff
\AccRel'\subseteq
\AccRel\symdiff\AccRel''$
\end{itemize}
\end{definition}

	Although simple, this notion of closeness is sufficient for our purposes here, as we will see in the sequel. Notice that other distance notions could have been considered as well, like \eg\ the {\em cardinality} of symmetric differences. (See Section~\ref{Comments} for a discussion on this.)

%%%%%%%%%%%%%%%%%%%%%%%%%%%%%%%%%%%%%%%%%%%%%%%%%%%
\section{Semantics of Action Theory Change}\label{SemanticsErasure}
%%%%%%%%%%%%%%%%%%%%%%%%%%%%%%%%%%%%%%%%%%%%%%%%%%%

	When admitting the possibility of a law $\modfml$ failing, one must ensure that $\modfml$ becomes invalid, \ie, not true in at least one model of the dynamic domain. Because there can be lots of such models, we may have a {\em set} $\ModelSet$ of models in which $\modfml$ is (potentially) valid. Thus contracting $\modfml$ amounts to making it no longer valid in this set of models. What are the operations that must be carried out to achieve that? Throwing models out of $\ModelSet$ does not work, since $\modfml$ will keep on being valid in all models of the remaining set. Thus one should {\em add} new models to $\ModelSet$. Which models? Well, models in which $\modfml$ is not true. But not any of such models: taking models falsifying $\modfml$ that are too different from our original models will certainly violate minimal change.
	
	Hence, we shall take some model $\model\in\ModelSet$ as basis and manipulate it to get a new model $\model'$ in which $\modfml$ is not true. In dynamic logic, the removal of a law $\modfml$ from a model $\model=\TransStruct{}$ means modifying the possible worlds or the accessibility relation in $\model$ so that $\modfml$ becomes false. Such an operation gives as result a {\em set} $\ErasureModels{\model}{\modfml}$ of models each of which is no longer a model of $\modfml$. But if there are several candidates, which ones should we choose? We shall take those models that are {\em minimal} modifications of the original $\model$, \ie, those minimal \wrt\ $\submodel{\model}$. Note that there can be more than one $\model'$ that is minimal. Hence, because adding just one of these new models is enough to invalidate $\modfml$, we take all possible combinations $\ModelSet\cup\{\model'\}$ of expanding our original set of models by one of these minimal models. The result will be a {\em set of sets of models}. In each set of models there will be one $\model'$ falsifying $\modfml$.

\subsection{Model Contraction of Executability Laws}\label{ModelContractionExec}

	To contract an executability law $\antec\imp\poss\act\top$ from {\em one} model, one intuitively {\em removes arrows} leaving $\antec$-worlds. In order to succeed in the operation, we have to guarantee that in the resulting model there will be at least one $\antec$-world with no departing $\act$-arrow.

\begin{definition}
\label{DefErasureOneModelExe}
Let $\model=\TransStruct{}$. $\model'=\TransStructP{}\in\ErasureModels{\model}{\antec\imp\poss\act\top}$ if and only if 
\begin{itemize}
\item $\Worlds'=\Worlds$
\item $\AccRel'\subseteq\AccRel$
\item If $(w,w')\in\AccRel\setminus\AccRel'$, then $\wMvalid{w}{\model}\antec$
\item There is $w\in\Worlds'$\ \suchthat\ $\notwMvalid{w}{\model'}\antec\imp\poss\act\top$
\end{itemize}
\end{definition}	

	Observe that $\ErasureModels{\model}{\antec\imp\poss\act\top}\neq\emptyset$ if and only if $\antec$ is satisfiable in $\Worlds$. Moreover, $\model\in\ErasureModels{\model}{\antec\imp\poss\act\top}$ if and only if $\notMvalid{\model}\antec\imp\poss\act\top$.

\myskip

	To get minimal change, we want such an operation to be minimal \wrt\ the original model: one should remove a minimum set of arrows sufficient to get the desired result.

\begin{definition}\label{DefMinimalErasureExe}
$\erasure{\model}{\antec\imp\poss\act\top}=\bigcup\min\{\ErasureModels{\model}{\antec\imp
\poss\act\top},\submodel{\model}\}$
\end{definition}

	And now we define the sets of possible models resulting from the contraction of an executability law in a set of models:

\begin{definition}\label{DefErasureModelsExe}
Let $\ModelSet$ be a set of models, and $\antec\imp\poss\act\top$ an executability law. Then
\[
\ErasureModels{\ModelSet}{\antec\imp\poss\act\top}=\{\ModelSet' : \ModelSet'=\ModelSet\cup\{\model'\}, \model'\in\erasure{\model}{\antec\imp\poss\act\top}, \model\in\ModelSet\}
\]
\end{definition}

	In our running example, consider $\ModelSet=\{\model\}$, where $\model$ is the model in Figure~\ref{ExBigModelCoffeeMachine}. When the agent discovers that even with a token she does not manage to buy a coffee any more, she has to change her models in order to admit (new) models with states where $\Token$ is the case but from which there is no \buy-transition at all. Because having just one such world in each new model is enough, taking those resulting models whose accessibility relations are maximal guarantees minimal change. Hence we will have $\ErasureModels{\ModelSet}{\Token\imp\poss\buy\top}=\{\ModelSet\cup\{\model'_{1}\},\ModelSet\cup\{\model'_{2}\},\ModelSet\cup\{\model'_{3}\}\}$, where each $\model'_{i}$ is depicted in Figure~\ref{ExErasureExecModelCoffeeMachine}.

\begin{figure}[h]
\begin{center}
\parbox{1.5cm}{
$\model'_{1}$ :
}\parbox{5.3cm}{
\scriptsize{
\begin{picture}(35,35)(0,0) % (0,0) Ž embaixo ˆ esquerda, (50,50) Ž acima ˆ direita
    \gasset{Nw=9,Nh=4,linewidth=0.3}
    \thinlines
    \node(A1)(10,15){$\AToken,\ACoffee,\AHot$}
    \node[Nw=10](A2)(20,30){$\neg\AToken,\ACoffee,\AHot$}
    \node[Nw=10](A3)(30,15){$\AToken,\neg\ACoffee,\AHot$}
    \node[Nw=13](A4)(0,0){$\neg\AToken,\neg\ACoffee,\neg\AHot$}
    \node[Nw=12](A5)(40,0){$\neg\AToken,\neg\ACoffee,\AHot$}
    \node[Nw=12](A6)(20,0){$\AToken,\neg\ACoffee,\neg\AHot$}
 %\drawedge[ELside=l](A1,A2){$\Abuy$}
    \drawedge[ELside=r](A3,A2){$\Abuy$}
    \drawedge[ELside=r](A6,A2){$\Abuy$}
%    \drawedge[curvedepth=-3,linewidth=0.4,ELside=r](A1,A3){$\act_{1}$}
%    \drawedge[curvedepth=3,linewidth=0.4,ELside=l](A1,A3){$\act_{2}$}
%    \drawedge[linewidth=0.4,ELside=l](A2,A3){$\act_{1}$}
%    \drawloop[linewidth=0.4,loopangle=90](A2){$\act_{1}$}
%    \drawloop[linewidth=0.4,loopangle=270](A3){$\act_{2}$}
\end{picture}
} % End \scriptsize
}\parbox{1.5cm}{
$\model'_{2}$ :
}\parbox{5.3cm}{
\scriptsize{
\begin{picture}(35,35)(0,0) % (0,0) Ž embaixo ˆ esquerda, (50,50) Ž acima ˆ direita
    \gasset{Nw=9,Nh=4,linewidth=0.3}
    \thinlines
    \node(A1)(10,15){$\AToken,\ACoffee,\AHot$}
    \node[Nw=10](A2)(20,30){$\neg\AToken,\ACoffee,\AHot$}
    \node[Nw=10](A3)(30,15){$\AToken,\neg\ACoffee,\AHot$}
    \node[Nw=13](A4)(0,0){$\neg\AToken,\neg\ACoffee,\neg\AHot$}
    \node[Nw=12](A5)(40,0){$\neg\AToken,\neg\ACoffee,\AHot$}
    \node[Nw=12](A6)(20,0){$\AToken,\neg\ACoffee,\neg\AHot$}
    \drawedge[ELside=l](A1,A2){$\Abuy$}
 %\drawedge[ELside=r](A3,A2){$\Abuy$}
    \drawedge[ELside=r](A6,A2){$\Abuy$}
%    \drawedge[curvedepth=-3,linewidth=0.4,ELside=r](A1,A3){$\act_{1}$}
%    \drawedge[curvedepth=3,linewidth=0.4,ELside=l](A1,A3){$\act_{2}$}
%    \drawedge[linewidth=0.4,ELside=l](A2,A3){$\act_{1}$}
%    \drawloop[linewidth=0.4,loopangle=90](A2){$\act_{1}$}
%    \drawloop[linewidth=0.4,loopangle=270](A3){$\act_{2}$}
\end{picture}
} % End \scriptsize
}

\bigskip

\parbox{1.5cm}{
$\model'_{3}$ :
}\parbox{5.3cm}{
\scriptsize{
\begin{picture}(35,35)(0,0) % (0,0) Ž embaixo ˆ esquerda, (50,50) Ž acima ˆ direita
    \gasset{Nw=9,Nh=4,linewidth=0.3}
    \thinlines
    \node(A1)(10,15){$\AToken,\ACoffee,\AHot$}
    \node[Nw=10](A2)(20,30){$\neg\AToken,\ACoffee,\AHot$}
    \node[Nw=10](A3)(30,15){$\AToken,\neg\ACoffee,\AHot$}
    \node[Nw=13](A4)(0,0){$\neg\AToken,\neg\ACoffee,\neg\AHot$}
    \node[Nw=12](A5)(40,0){$\neg\AToken,\neg\ACoffee,\AHot$}
    \node[Nw=12](A6)(20,0){$\AToken,\neg\ACoffee,\neg\AHot$}
    \drawedge[ELside=l](A1,A2){$\Abuy$}
    \drawedge[ELside=r](A3,A2){$\Abuy$}
 %\drawedge[ELside=r](A6,A2){$\Abuy$}
%    \drawedge[curvedepth=-3,linewidth=0.4,ELside=r](A1,A3){$\act_{1}$}
%    \drawedge[curvedepth=3,linewidth=0.4,ELside=l](A1,A3){$\act_{2}$}
%    \drawedge[linewidth=0.4,ELside=l](A2,A3){$\act_{1}$}
%    \drawloop[linewidth=0.4,loopangle=90](A2){$\act_{1}$}
%    \drawloop[linewidth=0.4,loopangle=270](A3){$\act_{2}$}
\end{picture}
} % End \scriptsize
} % End the last \parbox
\end{center}
\caption{Models resulting from contracting $\Token\imp\poss\buy\top$ in the model $\model$ of 
Figure~\ref{ExBigModelCoffeeMachine}.}
\label{ExErasureExecModelCoffeeMachine}
\end{figure}

	Clearly, if $\antec$ is not satisfied in $\ModelSet$, \ie, $\Mvalid{\model}\neg\antec$ for all $\model\in\ModelSet$, then the contraction of $\antec\imp\poss\act\top$ does not succeed. In this case, $\neg\antec$ should be contracted from the set of models (see further in this section).

\subsection{Model Contraction of Effect Laws}\label{ModelContractionEff}

	When the agent discovers that there may be cases where after buying she gets no hot 
beverage, she must \eg\ give up the belief $\Token\imp\nec\buy\Hot$ in her set of models. This 
means that $\Token\et\poss\buy\neg\Hot$ shall now be admitted in at least one world of some of her new models of beliefs. Hence, to contract an effect law $\antec\imp\nec\act\conseq$ from a given model, intuitively we have to {\em add arrows} leaving $\antec$-worlds to worlds satisfying $\neg\conseq$. The challenge in such an operation is how to guarantee minimal change.

	In our example, when contracting $\Token\imp\nec\buy\Hot$ in the model of Figure~\ref{ExBigModelCoffeeMachine} we add arrows from $\Token$-worlds to $\neg\Hot$-worlds. Because $\Coffee\imp\Hot$, and then $\neg\Hot\imp\neg\Coffee$, this should also give $\poss\buy\neg\Coffee$ in some $\Token$-world ($\neg\Coffee$ is {\em relevant} to $\neg\Hot$, \ie, to have $\neg\Hot$ we must have $\neg\Coffee$). This means that if we allow for $\poss\buy\neg\Hot$ in some \Token-world, we also have to allow for $\poss\buy\neg\Coffee$ in that same world.
	
	Hence, in our example one can add arrows from $\Token$-worlds to $\neg\Hot\et\neg\Coffee\et\Token$-worlds, as well as to $\neg\Hot\et\neg\Coffee\et\neg\Token$ (Figure~\ref{ExErasureEffectModelCoffeeMachine1}). For instance, one can add a \buy-arrow from $\{\Token,\neg\Coffee,\neg\Hot\}$ to one of these candidates (Figure~\ref{ExErasureEffectModelCoffeeMachine3}).
	
\begin{figure}[h]
\begin{center}
\parbox{1.5cm}{
$\model$ :
}\parbox{4cm}{
\scriptsize{
\begin{picture}(35,35)(0,0) % (0,0) Ž embaixo ˆ esquerda, (50,50) Ž acima ˆ direita
    \gasset{Nw=9,Nh=4,linewidth=0.3}
    \thinlines
    \node(A1)(10,15){$\AToken,\ACoffee,\AHot$}
    \node[Nw=10](A2)(20,30){$\neg\AToken,\ACoffee,\AHot$}
    \node[Nw=10](A3)(30,15){$\AToken,\neg\ACoffee,\AHot$}
    \node[Nw=13,dash={1}0](A4)(0,0){$\neg\AToken,\neg\ACoffee,\neg\AHot$}
    \node[Nw=12](A5)(40,0){$\neg\AToken,\neg\ACoffee,\AHot$}
    \node[Nw=12,dash={1}0](A6)(20,0){$\AToken,\neg\ACoffee,\neg\AHot$}
    \drawedge[ELside=l](A1,A2){$\Abuy$}
    \drawedge[ELside=r](A3,A2){$\Abuy$}
    \drawedge[ELside=r](A6,A2){$\Abuy$}
%    \drawedge[curvedepth=-3,linewidth=0.4,ELside=r](A1,A3){$\act_{1}$}
%    \drawedge[curvedepth=3,linewidth=0.4,ELside=l](A1,A3){$\act_{2}$}
%    \drawedge[linewidth=0.4,ELside=l](A2,A3){$\act_{1}$}
%    \drawloop[linewidth=0.4,loopangle=90](A2){$\act_{1}$}
%    \drawloop[linewidth=0.4,loopangle=270](A3){$\act_{2}$}
\end{picture}
} % End \scriptsize
} % End \parbox
\end{center}
\caption{Candidate worlds to receive arrows from \Token-worlds.}
\label{ExErasureEffectModelCoffeeMachine1}
\end{figure}
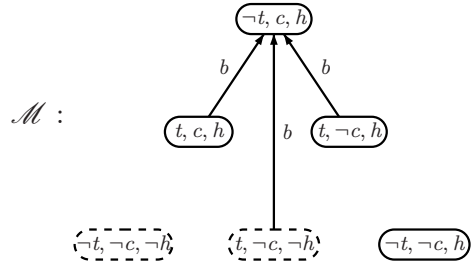

\begin{figure}[h]
\begin{center}
\parbox{1.5cm}{
$\model$ :
}\parbox{4cm}{
\scriptsize{
\begin{picture}(40,40)(0,0) % (0,0) Ž embaixo ˆ esquerda, (50,50) Ž acima ˆ direita
    \gasset{Nw=9,Nh=4,linewidth=0.3}
    \thinlines
    \node(A1)(10,18){$\AToken,\ACoffee,\AHot$}
    \node[Nw=10](A2)(20,33){$\neg\AToken,\ACoffee,\AHot$}
    \node[Nw=10](A3)(30,18){$\AToken,\neg\ACoffee,\AHot$}
    \node[Nw=13](A4)(0,3){$\neg\AToken,\neg\ACoffee,\neg\AHot$}
    \node[Nw=12](A5)(40,3){$\neg\AToken,\neg\ACoffee,\AHot$}
    \node[Nw=12](A6)(20,3){$\AToken,\neg\ACoffee,\neg\AHot$}
    \drawedge[ELside=l](A1,A2){$\Abuy$}
    \drawedge[ELside=r](A3,A2){$\Abuy$}
    \drawedge[ELside=r](A6,A2){$\Abuy$}
    \drawedge[ELside=r,dash={1}0](A6,A4){$\Abuy$}
    \drawloop[loopangle=270,dash={1}0](A6){$\Abuy$}
%    \drawedge[curvedepth=3,linewidth=0.4,ELside=l](A1,A3){$\act_{2}$}
%    \drawedge[linewidth=0.4,ELside=l](A2,A3){$\act_{1}$}
%    \drawloop[linewidth=0.4,loopangle=90](A2){$\act_{1}$}
%    \drawloop[linewidth=0.4,loopangle=270](A3){$\act_{2}$}
\end{picture}
} % End \scriptsize
} % End \parbox
\end{center}
\caption{Two candidate new \buy-arrows to falsify $\Token\imp\nec\buy\Hot$ in $\model$.}
\label{ExErasureEffectModelCoffeeMachine3}
\end{figure}
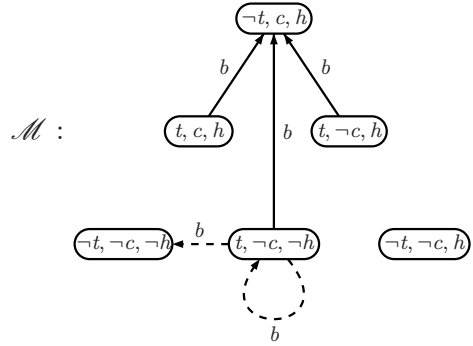

	Notice that adding the arrow to $\{\Token,\neg\Coffee,\neg\Hot\}$ itself would make us lose the effect $\neg\Token$, true after every execution of \buy\ in the original model ($\Mvalid{\model}\Token\imp\nec\buy\neg\Token$). How do we preserve this law while allowing for the new transition to a $\neg\Hot$-world? That is, how do we get rid of the effect $\Hot$ without losing effects that are not relevant for that? We here develop an approach for this issue.

\myskip

	When adding a new arrow leaving a world $w$ we intuitively want to preserve as many effects as we had before doing so. To achieve this, it is enough to preserve old effects only in $w$ (because the remaining structure of the model remains unchanged after adding the new arrow).  Of course, we cannot preserve effects that are inconsistent with $\neg\conseq$ (those will all be lost). So, it suffices to preserve only the effects that are consistent with $\neg\conseq$. To achieve that we must observe what is true in $w$ and in the target world $w'$:
\begin{itemize}
\item What changes from $w$ to $w'$ ($w'\setminus w$) must be what is obliged to do so: either because that is necessary to having $\neg\conseq$ in $w'$ or because that is necessary to having another effect (independent of $\neg\conseq$) in $w'$ that we want to preserve.
\item What does not change from $w$ to $w'$ ($w\cap w'$) should be what is allowed to do so: certain literals are never preserved (like $\Token$ in our example), then when pointing the arrow to a world where it does not change \wrt\ the leaving world ($\neg\Hot\et\neg\Coffee\et\Token$ in our example), we lose effects that held in $w$ before adding the arrow.
\end{itemize}

	This means that the only things allowed to change in the candidate target world must be those that are forced to change, either by some non-related law or because of having $\neg\conseq$ modulo a set of states $\Worlds$. In other words, we want the literals that change to be {\em at most} those that are sufficient to get $\neg\conseq$ modulo $\Worlds$, while preserving the maximum of effects. Every change outside that is not an intended one. Similarly, we want the literals that are preserved in the target world to be {\em at most} those that are usually preserved in a given set of models. Every preservation outside those may make us lose some law. This looks like prime implicants, and  that is where prime subvaluations play their role: the worlds to which the new arrow will point are those whose difference \wrt\ the departing world are literals that are relevant and whose similarity \wrt\ it are literals that we know do not change.

\begin{definition}[Relevant Target Worlds]
Let $\model=\TransStruct{}$ be a model, $w,w'\in\Worlds$, $\ModelSet$ a set of models such that $\model\in\ModelSet$, and $\antec\imp\nec\act\conseq$ an effect law. Then $w'$ is a {\em relevant target world of} $w$ \wrt\ $\antec\imp\nec\act\conseq$ for $\model$ in $\ModelSet$ if and only if
\begin{itemize}
\item $\wMvalid{w}{\model}\antec$, $\notwMvalid{w'}{\model}\conseq$
\item for all $\lit\in w'\setminus w$
\begin{itemize}
\item either there is $\val\in\base{\neg\conseq,\Worlds}$ \suchthat\ $\val\subseteq w'$ and $\lit\in\val$
\item or there is $\conseq'\in\Fml$ \suchthat\ there is $\val'\in\base{\conseq',\Worlds}$ \suchthat\ $\val'\subseteq w'$, $\lit\in\val'$, and for every $\model_{i}\in\ModelSet$, $\wMvalid{w}{\model_{i}}\nec\act\conseq'$
\end{itemize}
\item for all $\lit\in w\cap w'$
\begin{itemize}
\item either there is $\val\in\base{\neg\conseq,\Worlds}$ \suchthat\ $\val\subseteq w'$ and $\lit\in\val$
\item or there is $\model_{i}\in\ModelSet$ such that $\notwMvalid{w}{\model_{i}}\nec\act\neg\lit$
\end{itemize}
\end{itemize}
By $\RelTarget{w,\antec\imp\nec\act\conseq,\model,\ModelSet}$ we denote the set of all relevant target worlds of $w$ \wrt\ $\antec\imp\nec\act\conseq$ for $\model$ in $\ModelSet$.
\end{definition}

	Note that we need the set of models $\ModelSet$ (and here we can suppose it contains all models of the theory we want to change) because preserving effects depends on what other effects hold in the other models that interest us. We need to take them into account in the local operation of changing one model:\footnote{The reason we do not need $\ModelSet$ in the definition of the local (one model) contraction of executability laws $\ErasureModels{\model}{\antec\imp\poss\act\top}$ is that when removing arrows there is no way of losing effects, as every effect law that held in the world from which an arrow has been removed remains true in the same world in the resulting model.}

\begin{definition}\label{DefErasureOneModelEff}
Let $\model=\TransStruct{}$, and $\ModelSet$ be such that $\model\in\ModelSet$. Then $\model'=\TransStructP{}\in\ErasureModels{\model}{\antec\imp\nec\act\conseq}$ if and only if 
\begin{itemize}
\item $\Worlds'=\Worlds$
\item $\AccRel\subseteq\AccRel'$
\item If $(w,w')\in\AccRel'\setminus\AccRel$, then $w'\in\RelTarget{w,\antec\imp\nec\act\conseq,\model,\ModelSet}$
\item There is $w\in\Worlds'$\ \suchthat\ $\notwMvalid{w}{\model'}\antec\imp\nec\act\conseq$
\end{itemize}
\end{definition}

	Observe that $\ErasureModels{\model}{\antec\imp\nec\act\conseq}\neq\emptyset$ if and only if $\antec$ and $\neg\conseq$ are both satisfiable in $\Worlds$. Moreover, $\model\in\ErasureModels{\model}{\antec\imp\nec\act\conseq}$ if and only if $\notMvalid{\model}\antec\imp\nec\act\conseq$.

\myskip

	Because having just one world where the law is no longer true in each model is enough, taking those resulting models whose accessibility relations are minimal \wrt\ the original one guarantees minimal change.

\begin{definition}\label{DefMinimalErasureEff}
$\erasure{\model}{\antec\imp\nec\act\conseq}=\bigcup\min\{\ErasureModels{\model}{\antec\imp\nec\act\conseq},\submodel{\model}\}$
\end{definition}

	Now we can define the possible sets of models resulting from contracting an effect law from a set of models:

\begin{definition}\label{DefErasureModelsEff}
Let $\ModelSet$ be a set of models, and $\antec\imp\nec\act\conseq$ an effect law. Then
\[
\ErasureModels{\ModelSet}{\antec\imp\nec\act\conseq}=\{\ModelSet' : \ModelSet'=\ModelSet\cup\{\model'\}, \model'\in\erasure{\model}{\antec\imp\nec\act\conseq}, \model\in\ModelSet\}
\]
\end{definition}

	Taking again $\ModelSet=\{\model\}$, where $\model$ is the model in Figure~\ref{ExBigModelCoffeeMachine}, after contracting $\Token\imp\nec\buy\Hot$ from $\ModelSet$ we get $\ErasureModels{\ModelSet}{\Token\imp\nec\buy\Hot}=\{\ModelSet\cup\{\model'_{1}\},\ModelSet\cup\{\model'_{2}\},\ModelSet\cup\{\model'_{3}\}\}$, where all $\model'_{i}$s are as depicted in Figure~\ref{ExErasureEffectModelCoffeeMachine2}.

\begin{figure}[h]
\begin{center}
\parbox{1.5cm}{
$\model'_{1}$ :
}\parbox{5.3cm}{
\scriptsize{
\begin{picture}(35,35)(0,0) % (0,0) Ž embaixo ˆ esquerda, (50,50) Ž acima ˆ direita
    \gasset{Nw=9,Nh=4,linewidth=0.3}
    \thinlines
    \node(A1)(10,15){$\AToken,\ACoffee,\AHot$}
    \node[Nw=10](A2)(20,30){$\neg\AToken,\ACoffee,\AHot$}
    \node[Nw=10](A3)(30,15){$\AToken,\neg\ACoffee,\AHot$}
    \node[Nw=13](A4)(0,0){$\neg\AToken,\neg\ACoffee,\neg\AHot$}
    \node[Nw=12](A5)(40,0){$\neg\AToken,\neg\ACoffee,\AHot$}
    \node[Nw=12](A6)(20,0){$\AToken,\neg\ACoffee,\neg\AHot$}
    \drawedge[ELside=l](A1,A2){$\Abuy$}
    \drawedge[ELside=r](A3,A2){$\Abuy$}
    \drawedge[ELside=r](A6,A2){$\Abuy$}
    \drawedge[ELside=r](A1,A4){$\Abuy$}
%    \drawedge[curvedepth=-3,linewidth=0.4,ELside=r](A1,A3){$\act_{1}$}
%    \drawedge[curvedepth=3,linewidth=0.4,ELside=l](A1,A3){$\act_{2}$}
%    \drawedge[linewidth=0.4,ELside=l](A2,A3){$\act_{1}$}
%    \drawloop[linewidth=0.4,loopangle=90](A2){$\act_{1}$}
%    \drawloop[linewidth=0.4,loopangle=270](A3){$\act_{2}$}
\end{picture}
} % End \scriptsize
}\parbox{1.5cm}{
$\model'_{2}$ :
}\parbox{5.3cm}{
\scriptsize{
\begin{picture}(35,35)(0,0) % (0,0) Ž embaixo ˆ esquerda, (50,50) Ž acima ˆ direita
    \gasset{Nw=9,Nh=4,linewidth=0.3}
    \thinlines
    \node(A1)(10,15){$\AToken,\ACoffee,\AHot$}
    \node[Nw=10](A2)(20,30){$\neg\AToken,\ACoffee,\AHot$}
    \node[Nw=10](A3)(30,15){$\AToken,\neg\ACoffee,\AHot$}
    \node[Nw=13](A4)(0,0){$\neg\AToken,\neg\ACoffee,\neg\AHot$}
    \node[Nw=12](A5)(40,0){$\neg\AToken,\neg\ACoffee,\AHot$}
    \node[Nw=12](A6)(20,0){$\AToken,\neg\ACoffee,\neg\AHot$}
    \drawedge[ELside=l](A1,A2){$\Abuy$}
    \drawedge[ELside=r](A3,A2){$\Abuy$}
    \drawedge[ELside=r](A6,A2){$\Abuy$}
    \drawedge[ELside=r](A6,A4){$\Abuy$}
%    \drawedge[curvedepth=-3,linewidth=0.4,ELside=r](A1,A3){$\act_{1}$}
%    \drawedge[curvedepth=3,linewidth=0.4,ELside=l](A1,A3){$\act_{2}$}
%    \drawedge[linewidth=0.4,ELside=l](A2,A3){$\act_{1}$}
%    \drawloop[linewidth=0.4,loopangle=90](A2){$\act_{1}$}
%    \drawloop[linewidth=0.4,loopangle=270](A3){$\act_{2}$}
\end{picture}
} % End \scriptsize
}

\bigskip

\parbox{1.5cm}{
$\model'_{3}$ :
}\parbox{5.3cm}{
\scriptsize{
\begin{picture}(35,35)(0,0) % (0,0) Ž embaixo ˆ esquerda, (50,50) Ž acima ˆ direita
    \gasset{Nw=9,Nh=4,linewidth=0.3}
    \thinlines
    \node(A1)(10,15){$\AToken,\ACoffee,\AHot$}
    \node[Nw=10](A2)(20,30){$\neg\AToken,\ACoffee,\AHot$}
    \node[Nw=10](A3)(30,15){$\AToken,\neg\ACoffee,\AHot$}
    \node[Nw=13](A4)(0,0){$\neg\AToken,\neg\ACoffee,\neg\AHot$}
    \node[Nw=12](A5)(40,0){$\neg\AToken,\neg\ACoffee,\AHot$}
    \node[Nw=12](A6)(20,0){$\AToken,\neg\ACoffee,\neg\AHot$}
    \drawedge[ELside=l](A1,A2){$\Abuy$}
    \drawedge[ELside=r](A3,A2){$\Abuy$}
    \drawedge[ELside=r](A6,A2){$\Abuy$}
    \drawedge[ELside=r](A3,A4){$\Abuy$}
%    \drawedge[curvedepth=-3,linewidth=0.4,ELside=r](A1,A3){$\act_{1}$}
%    \drawedge[curvedepth=3,linewidth=0.4,ELside=l](A1,A3){$\act_{2}$}
%    \drawedge[linewidth=0.4,ELside=l](A2,A3){$\act_{1}$}
%    \drawloop[linewidth=0.4,loopangle=90](A2){$\act_{1}$}
%    \drawloop[linewidth=0.4,loopangle=270](A3){$\act_{2}$}
\end{picture}
} % End \scriptsize
} % End the last \parbox
\end{center}
\caption{Models resulting from contracting $\Token\imp\nec\buy\Hot$ in the model $\model$ of 
Figure~\ref{ExBigModelCoffeeMachine}.}
\label{ExErasureEffectModelCoffeeMachine2}
\end{figure}

\myskip

	In both cases where $\antec$ is not satisfiable in $\model$ or $\conseq$ is valid in $\model$, of course our operator does not succeed in falsifying $\antec\imp\nec\act\conseq$ (\cf\ end of Section~\ref{ModelContractionExec}). Intuitively, prior to doing that we have to change our set of possible states. This is what we address in the next section.

\subsection{Model Contraction of Static Laws}\label{ModelContractionStat}

	When contracting a static law from a model, we want to admit the existence of at least one (new) possible state falsifying it. This means that intuitively we should {\em add new worlds} to the original model. This is quite easy. A very delicate issue however is what to do with the accessibility relation: should new arrows leave/arrive at the new world? If no arrow leaves the new added world, we may lose some executability law. If some arrow leaves it, then we may lose some effect law, the same holding if we add an arrow pointing to the new world. On the other hand, if no arrow arrives at the new world, what about the intuition? Is it intuitive to have an unreachable state?

	All this discussion shows how drastic a change in the static laws may be: it is a change in the 
underlying structure (possible states) of the world! Changing it may have as consequence the loss of an effect law or an executability law. What we can do is choose which laws we accept to lose and postpone their change (by the other operators). Following the tradition in the RAA community which states that executability laws are, in general, more difficult to formalize than effect laws, and hence are more likely to be incorrect, here we prefer not to change the accessibility relation, which means preserving effect laws and postponing correction of executability laws, if needed. (\cf\ Sections~\ref{SyntaxStatErasure} and~\ref{Conclusion} below).

\begin{definition}\label{DefErasureOneModelStat}
Let $\model=\TransStruct{}$. $\model'=\TransStructP{}\in\ErasureModels{\model}{\fml}$ if and 
only if 
\begin{itemize}
\item $\Worlds\subseteq\Worlds'$
\item $\AccRel=\AccRel'$
\item There is $w\in\Worlds'$\ \suchthat\ $\notwMvalid{w}{\model'}\fml$
\end{itemize}
\end{definition}

	Notice that we have $\ErasureModels{\model}{\fml}=\emptyset$ if and only if $\fml$ is a tautology. Moreover, $\model\in\ErasureModels{\model}{\fml}$ if and only if $\notMvalid{\model}\fml$.

\myskip

	The minimal modifications of one model are defined as usual:

\begin{definition}\label{DefMinimalErasureStat}
$\erasure{\model}{\fml}=\bigcup\min\{\ErasureModels{\model}{\fml},\submodel{\model}\}$
\end{definition}

	And now we define the sets of models resulting from contracting a static law from a given set of models:

\begin{definition}\label{DefErasureModelsStat}
Let $\ModelSet$ be a set of models, and $\fml$ a static law. Then
\[
\ErasureModels{\ModelSet}{\fml}=\{\ModelSet' : \ModelSet'=\ModelSet\cup\{\model'\}, \model'\in\erasure{\model}{\fml}, \model\in\ModelSet\}
\]
\end{definition}

	In our scenario example, if $\ModelSet=\{\model\}$, where $\model$ is the model in Figure~\ref{ExBigModelCoffeeMachine}, then contracting $\Coffee\imp\Hot$ from $\ModelSet$ would give us  $\ErasureModels{\ModelSet}{\Coffee\imp\Hot}=\{\ModelSet\cup\{\model'_{1}\},\ModelSet\cup\{\model'_{2}\}\}$, where each $\model'_{i}$ is as depicted in Figure~\ref{ExErasureStatModelCoffeeMachine}.

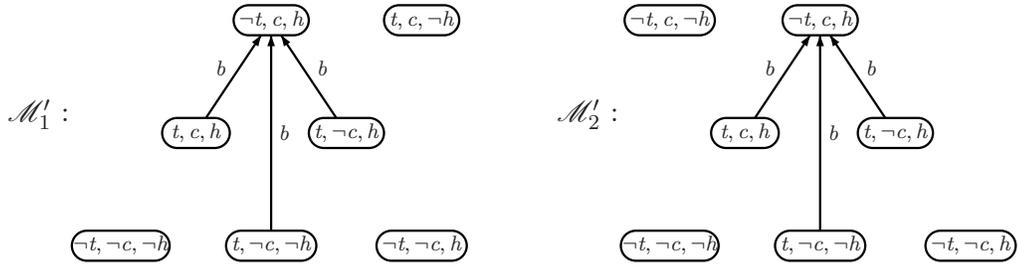
\begin{figure}[h]
\begin{center}
\parbox{1.5cm}{
$\model'_{1}$ :
}\parbox{5.8cm}{
\scriptsize{
\begin{picture}(35,35)(0,0) % (0,0) Ž embaixo ˆ esquerda, (50,50) Ž acima ˆ direita
    \gasset{Nw=9,Nh=4,linewidth=0.3}
    \thinlines
    \node(A1)(10,15){$\AToken,\ACoffee,\AHot$}
    \node[Nw=10](A2)(20,30){$\neg\AToken,\ACoffee,\AHot$}
    \node[Nw=10](A3)(30,15){$\AToken,\neg\ACoffee,\AHot$}
    \node[Nw=13](A4)(0,0){$\neg\AToken,\neg\ACoffee,\neg\AHot$}
    \node[Nw=12](A5)(40,0){$\neg\AToken,\neg\ACoffee,\AHot$}
    \node[Nw=12](A6)(20,0){$\AToken,\neg\ACoffee,\neg\AHot$}
    \node[Nw=10](A7)(40,30){$\AToken,\ACoffee,\neg\AHot$} 
    \drawedge[ELside=l](A1,A2){$\Abuy$}
    \drawedge[ELside=r](A3,A2){$\Abuy$}
    \drawedge[ELside=r](A6,A2){$\Abuy$}
%    \drawedge[curvedepth=-3,linewidth=0.4,ELside=r](A1,A3){$\act_{1}$}
%    \drawedge[curvedepth=3,linewidth=0.4,ELside=l](A1,A3){$\act_{2}$}
%    \drawedge[linewidth=0.4,ELside=l](A2,A3){$\act_{1}$}
%    \drawloop[linewidth=0.4,loopangle=90](A2){$\act_{1}$}
%    \drawloop[linewidth=0.4,loopangle=270](A3){$\act_{2}$}
\end{picture}
} % End \scriptsize
}\parbox{1.5cm}{
$\model'_{2}$ :
}\parbox{5.3cm}{
\scriptsize{
\begin{picture}(35,35)(0,0) % (0,0) Ž embaixo ˆ esquerda, (50,50) Ž acima ˆ direita
    \gasset{Nw=9,Nh=4,linewidth=0.3}
    \thinlines
    \node(A1)(10,15){$\AToken,\ACoffee,\AHot$}
    \node[Nw=10](A2)(20,30){$\neg\AToken,\ACoffee,\AHot$}
    \node[Nw=10](A3)(30,15){$\AToken,\neg\ACoffee,\AHot$}
    \node[Nw=13](A4)(0,0){$\neg\AToken,\neg\ACoffee,\neg\AHot$}
    \node[Nw=12](A5)(40,0){$\neg\AToken,\neg\ACoffee,\AHot$}
    \node[Nw=12](A6)(20,0){$\AToken,\neg\ACoffee,\neg\AHot$}
    \node[Nw=12](A7)(0,30){$\neg\AToken,\ACoffee,\neg\AHot$}    
    \drawedge[ELside=l](A1,A2){$\Abuy$}
    \drawedge[ELside=r](A3,A2){$\Abuy$}
    \drawedge[ELside=r](A6,A2){$\Abuy$}
%    \drawedge[curvedepth=-3,linewidth=0.4,ELside=r](A1,A3){$\act_{1}$}
%    \drawedge[curvedepth=3,linewidth=0.4,ELside=l](A1,A3){$\act_{2}$}
%    \drawedge[linewidth=0.4,ELside=l](A2,A3){$\act_{1}$}
%    \drawloop[linewidth=0.4,loopangle=90](A2){$\act_{1}$}
%    \drawloop[linewidth=0.4,loopangle=270](A3){$\act_{2}$}
\end{picture}
} % End \scriptsize
} % End \parbox
\end{center}
\caption{Models resulting from contracting $\Coffee\imp\Hot$ in the model $\model$ of Figure~
\ref{ExBigModelCoffeeMachine}.}
\label{ExErasureStatModelCoffeeMachine}
\end{figure}

	Notice that by not modifying the accessibility relation all the effect laws are preserved with minimal change. Moreover, our approach is in line with intuition: when learning that a new state is now possible, we do not necessarily know all the behavior of the actions in the new added state. We may expect some action laws to hold in the new state (see Section~\ref{Conclusion} for an alternative solution), but, with the information we dispose, not touching the accessibility relation is the safest way of contracting static laws.

%%%%%%%%%%%%%%%%%%%%%%%%%%%%%%%%%%%%%%%%%%%%%%%%%%%
\section{Syntactic Operators for Contraction of Laws}\label{SyntaxErasure}
%%%%%%%%%%%%%%%%%%%%%%%%%%%%%%%%%%%%%%%%%%%%%%%%%%%

	Now that we have defined the semantics of our theory change, we turn our attention to the definition of syntactic operators for changing sets of formulas.
	
	As Nebel~\cite{Nebel89} says, ``[\ldots] finite bases usually represent [\ldots] laws, and when we are forced to change the theory we would like to stay as close as possible to the original [\ldots] base.'' Hence, besides the definition of syntactical operators, we should also guarantee that they  perform minimal change.

\myskip

	By $\ErasureSet{\Theory{}}{\modfml}$ we denote in the sequel  the result of contracting a law $\modfml$ from the set of laws~$\Theory{}$.

\subsection{Contracting Executability Laws}\label{SyntacticContractionExec}

	For the case of contracting $\antec\imp\poss\act\top$ from an action theory, first we have to ensure that the action $\act$ is still executable in all those contexts where $\neg\antec$ is the case. Second, in order to get minimality, we must make $\act$ executable in {\em some} contexts where $\antec$ is true, \viz\ all $\antec$-worlds but one. This means that we can have several action theories as outcome. Algorithm~\ref{ExecOperator} gives a syntactical operator to achieve this.

\begin{algorithm}
\caption{Erasure of an executability law}
\label{ExecOperator}
\begin{algorithmic}[1]
\onehalfspacing
\INPUT $\Theory{}$, $\antec\imp\poss\act\top$
\OUTPUT $\ErasureSet{\Theory{}}{\antec\imp\poss\act\top}$
\STATE $\ErasureSet{\Theory{}}{\antec\imp\poss\act\top}\GETS\ \emptyset$
\IF{$\Theory{}\PDLvalid\antec\imp\poss\act\top$}
	\FORALL{$\term\in\IP{\STAT{}{}\et\antec}$}
		\FORALL{$\Set\subseteq\compatterm$}
			\STATE $\antec_{\Set}\GETS\ \bigwedge_{\stackscript{\prp_{i}\in\compatterm}{\prp_{i}\in\Set}}\prp_{i}\et\bigwedge_{\stackscript{\prp_{i}\in\compatterm}{\prp_{i}\notin\Set}}\neg\prp_{i}$
			\IF{$\STAT{}{}\notCPLvalid(\term\et\antec_{\Set})\imp\bot$}
				\STATE \parbox{4cm}{\[ \Theory{}'\GETS\ \left( \begin{array}{l}
				(\Theory{}\setminus\EXE{}{\act})\ \cup \\
				\{(\antec_{i}\et\neg(\term\et\antec_{\Set}))\imp\poss\act\top : 
				\antec_{i}\imp\poss\act\top \in\EXE{}{\act}\}
				\end{array}
				\right)
				\]
				} % End \parbox
			\ENDIF
			\STATE $\ErasureSet{\Theory{}}{\antec\imp\poss\act\top}\GETS\ 
			       \ErasureSet{\Theory{}}{\antec\imp\poss\act\top}\cup\{\Theory{}'\}$
		\ENDFOR
	\ENDFOR
	\ELSE 
		\STATE $\ErasureSet{\Theory{}}{\antec\imp\poss\act\top}\GETS\ \{\Theory{}\}$
\ENDIF
\RETURN $\ErasureSet{\Theory{}}{\antec\imp\poss\act\top}$
\end{algorithmic}
\end{algorithm}
\singlespacing

	Observe that from the finiteness of $\Theory{}$ and that of $\compatterm$, for any $\term\in\IP{\STAT{}{}\et\antec}$, and the decidability of \PDL~\cite{HarelEtAl00} and of classical propositional logic, it follows that Algorithm~\ref{ExecOperator} terminates.

\myskip

	In our running example, contracting the executability law $\Token\imp\poss\buy\top$ from the action theory $\Theory{}$ would give us $\ErasureSet{\Theory{}}{\Token\imp\poss\buy\top}=\{\Theory{1}',\Theory{2}',\Theory{3}'\}$, where:

\[
\Theory{1}'=\left\{
  \begin{array}{c}
	 \Coffee\imp\Hot, \neg\Coffee\imp\nec\buy\Coffee, \\
	  \Token\imp\nec\buy\neg\Token, \neg\Token\imp\nec\buy\bot, \\
	  \Coffee\imp\nec\buy\Coffee, \Hot\imp\nec\buy\Hot, \\
	  (\Token\et\neg\Coffee\et\Hot)\imp\poss\buy\top, \\
	 (\Token\et\neg\Coffee\et\neg\Hot)\imp\poss\buy\top
  \end{array}
  \right\}
\]

\[
\Theory{2}'=\left\{
  \begin{array}{c}
	 \Coffee\imp\Hot, \neg\Coffee\imp\nec\buy\Coffee, \\
	 \Token\imp\nec\buy\neg\Token, \neg\Token\imp\nec\buy\bot, \\
	 \Coffee\imp\nec\buy\Coffee, \Hot\imp\nec\buy\Hot, \\
	  (\Token\et\Coffee\et\Hot)\imp\poss\buy\top, \\
	 (\Token\et\neg\Coffee\et\neg\Hot)\imp\poss\buy\top
  \end{array}
  \right\}
\]

\[
\Theory{3}'=\left\{
  \begin{array}{c}
	 \Coffee\imp\Hot, \neg\Coffee\imp\nec\buy\Coffee, \\
	 \Token\imp\nec\buy\neg\Token, \neg\Token\imp\nec\buy\bot, \\
	 \Coffee\imp\nec\buy\Coffee, \Hot\imp\nec\buy\Hot, \\
	  (\Token\et\Coffee\et\Hot)\imp\poss\buy\top, \\
	 (\Token\et\neg\Coffee\et\Hot)\imp\poss\buy\top
  \end{array}
  \right\}
\]

\myskip

	Now the knowledge engineer has only to choose which theory is more in line with her 
intuitions and implement the changes (\cf\ Figure~\ref{ExErasureExecModelCoffeeMachine}).

\subsection{Contracting Effect Laws}

	When contracting $\antec\imp\nec\act\conseq$ from a theory $\Theory{}$, intuitively we should contract some effect laws that preclude $\neg\conseq$ in target worlds. In order to cope with minimality, we must change only those laws that are relevant to $\antec\imp\nec\act\conseq$.

	Let $\EFF{\antec,\conseq}{\act}$ denote a minimum subset of $\EFF{}{\act}$ such that $\STAT{}{},\EFF{\antec,\conseq}{\act}\PDLvalid\antec\imp\nec\act\conseq$. In the case the theory is modular~\cite{HerzigVarzinczak-Aiml04Proc05} (see further), such a set always exists. Moreover, note that there can be more than one such a set, in which case we denote them $(\EFF{\antec,\conseq}{\act})_{1},\ldots,(\EFF{\antec,\conseq}{\act})_{n}$. Let
\[
\EFF{-}{\act}=\bigcup_{1\leq i\leq n}(\EFF{\antec,\conseq}{\act})_{i}
\]
The laws in $\EFF{-}{\act}$ will serve as guidelines to get rid of $\antec\imp\nec\act\conseq$ in the theory.

\myskip

	The first thing we must do is to ensure that action $\act$ still has effect $\conseq$ in all 
those contexts in which $\antec$ does not hold. This means we shall weaken the laws in $\EFF{\antec,\conseq}{\act}$ specializing them to $\neg\antec$. Now, we need to preserve all old effects in all $\antec$-worlds but one. To achieve that we specialize the above laws to each possible valuation (maximal conjunction of literals) satisfying $\antec$ but one. Then, in the left $\antec$-valuation, we must ensure that action $\act$ has either its old effects or $\neg\conseq$ as outcome. We achieve that by weakening the {\em consequent} of the laws in $\EFF{-}{\act}$. Finally, in order to get minimal change, we must ensure that all literals in this $\antec$-valuation that are not forced to change in $\neg\conseq$-worlds should be preserved. We do this by stating an effect law of the form $(\antec_{k}\et\lit)\imp\nec\act(\conseq\ou\lit)$, where $\antec_{k}$ is the above $\antec$-valuation. The reason this is needed is clear: there can be several $\neg\conseq$-valuations, and as far as we want at most one to be reachable from the $\antec_{k}$-world, we should force it to be the one whose difference to this $\antec_{k}$-valuation is minimal.

	Again, the result will be a set of action theories. Algorithm~\ref{EffectOperator} below gives the operator.

\begin{algorithm}
\caption{Contraction of an effect law}
\label{EffectOperator}
\begin{algorithmic}[1]
\onehalfspacing
\INPUT $\Theory{}$, $\antec\imp\nec\act\conseq$
\OUTPUT $\ErasureSet{\Theory{}}{\antec\imp\nec\act\conseq}$
\STATE $\ErasureSet{\Theory{}}{\antec\imp\nec\act\conseq}\GETS\ \emptyset$
\IF{$\Theory{}\PDLvalid\antec\imp\nec\act\conseq$}
\FORALL{$\term\in\IP{\STAT{}{}\et\antec}$}
	\FORALL{$\Set\subseteq\compatterm$}
		\STATE $\antec_{\Set}\GETS\ 
				\bigwedge_{\stackscript{\prp_{i}\in\compatterm}{\prp_{i}\in\Set}}\prp_{i}\et
				\bigwedge_{\stackscript{\prp_{i}\in\compatterm}{\prp_{i}\notin\Set}}\neg\prp_{i}$
		\IF{$\STAT{}{}\notCPLvalid(\term\et\antec_{\Set})\imp\bot$}
			\FORALL{$\term'\in\IP{\STAT{}{}\et\neg\conseq}$}
				\STATE \parbox{4cm}{\[ \Theory{}'\GETS\ 
						\left(\begin{array}{l}
						(\Theory{}\setminus\EFF{-}{\act})\ \cup \\
						\{(\antec_{i}\et\neg(\term\et\antec_{\Set})) \imp\nec\act\conseq_{i} : 
						\antec_{i}\imp\nec\act\conseq_{i} \in\EFF{-}{\act}\}\ \cup \\
						\{(\antec_{i}\et\term\et\antec_{\Set})
						\imp\nec\act(\conseq_{i}\ou\term') : 
						\antec_{i}\imp\nec\act\conseq_{i} \in\EFF{-}{\act}\}								\end{array}\right)
					\]
				} % End \parbox
				
				\FORALL{$L\subseteq\Lit$}
					\IF{$\STAT{}{}\CPLvalid
					(\term\et\antec_{\Set})\imp\bigwedge_{\lit\in L}\lit$ \AND\ 
					$\STAT{}{}\notCPLvalid(\term'\et\bigwedge_{\lit\in L}\lit)\imp\bot$}
					\FORALL{$\lit\in L$}
						\IF{$\Theory{}\notPDLvalid(\term\et\antec_{\Set}\et\lit)\imp
								\nec\act\neg\lit$ \OR\ $\lit\in\term'$}
								\STATE $\Theory{}'\GETS\ \Theory{}' \cup
							\{(\term\et\antec_{\Set}\et\lit)\imp\nec\act(\conseq\ou\lit)\}$
						\ENDIF
					\ENDFOR
					\ENDIF
				\ENDFOR
			\STATE $\ErasureSet{\Theory{}}{\antec\imp\nec\act\conseq}
			\GETS\ 
		       	\ErasureSet{\Theory{}}{\antec\imp\nec\act\conseq}\cup\{\Theory{}'\}$
		\ENDFOR
	   \ENDIF
	\ENDFOR
\ENDFOR
\ELSE
	\STATE $\ErasureSet{\Theory{}}{\antec\imp\nec\act\conseq}\GETS\ \{\Theory{}\}$
\ENDIF
\RETURN $\ErasureSet{\Theory{}}{\antec\imp\nec\act\conseq}$
\end{algorithmic}
\end{algorithm}
\singlespacing

	The reader is invited to check that Algorithm~\ref{EffectOperator} always terminates (\cf\ Section~\ref{SyntacticContractionExec}).
	
\myskip

	For an example of execution of the algorithm, suppose we want to contract the effect law $\Token\imp\nec\buy\Hot$ from our theory $\Theory{}$. We first determine the minimum sets of effect laws that together with $\STAT{}{}$ entail $\Token\imp\nec\buy\Hot$. They are
\[
(\EFF{\Token,\Hot}{\buy})_{1}=\left\{
  \begin{array}{c}
	\Coffee\imp\nec\buy\Coffee, \\
	\neg\Coffee\imp\nec\buy\Coffee
  \end{array}
\right\} 
\]\[
(\EFF{\Token,\Hot}{\buy})_{2}=\left\{
  \begin{array}{c}
	\Hot\imp\nec\buy\Hot, \\
	\neg\Coffee\imp\nec\buy\Coffee
  \end{array}
\right\}
\]

Now for each context where $\Token$ is the case, we weaken the effect laws in $\EFF{-}{\buy}=(\EFF{\Token,\Hot}{\buy})_{1}\cup(\EFF{\Token,\Hot}{\buy})_{2}$. Given $\STAT{}{}=\{\Coffee\imp\Hot\}$, such contexts are $\Token\et\Coffee\et\Hot$, $\Token\et\neg\Coffee\et\neg\Hot$ and $\Token\et\neg\Coffee\et\Hot$.

	For $\Token\et\Coffee\et\Hot$: we replace in $\Theory{}$ the laws from $\EFF{-}{\buy}$ by
\[
\left\{
  \begin{array}{c}
	(\Coffee\et\neg(\Token\et\Coffee\et\Hot))\imp\nec\buy\Coffee, \\
	(\Hot\et\neg(\Token\et\Coffee\et\Hot))\imp\nec\buy\Hot, \\
	(\neg\Coffee\et\neg(\Token\et\Coffee\et\Hot))\imp\nec\buy\Coffee
  \end{array}
  \right\}
\]so that we preserve their effects in all possible contexts but $\Token\et\Coffee\et\Hot$. Now, in order to preserve some effects in $\Token\et\Coffee\et\Hot$-contexts while allowing for reachable $\neg\Hot$-worlds, we add the laws:
\[
\left\{
  \begin{array}{c}
	(\Token\et\Coffee\et\Hot)\imp\nec\buy(\Coffee\ou\neg\Hot), \\
	(\Token\et\Coffee\et\Hot)\imp\nec\buy(\Hot\ou\neg\Coffee)
  \end{array}
\right\}
\]
Now, we search all possible combinations of laws from $\EFF{}{\buy}$ that apply on $\Token\et\Coffee\et\Hot$ contexts and find $\Token\imp\nec\buy\neg\Token$. Because $\neg\Token$ must be true after every execution of \buy, we do not state the law $(\Token\et\Coffee\et\Hot)\imp\nec\buy(\Hot\ou\Token)$, and end up with the theory:
\[
\Theory{1}'=\left\{
  \begin{array}{c}
	\Coffee\imp\Hot, \Token\imp\poss\buy\top, \\
	 \Token\imp\nec\buy\neg\Token, \neg\Token\imp\nec\buy\bot, \\
	(\Coffee\et\neg(\Token\et\Coffee\et\Hot))\imp\nec\buy\Coffee, \\
	(\Hot\et\neg(\Token\et\Coffee\et\Hot))\imp\nec\buy\Hot, \\
	(\neg\Coffee\et\neg(\Token\et\Coffee\et\Hot))\imp\nec\buy\Coffee, \\
	(\Token\et\Coffee\et\Hot)\imp\nec\buy(\Coffee\ou\neg\Hot), \\
	(\Token\et\Coffee\et\Hot)\imp\nec\buy(\Hot\ou\neg\Coffee)
  \end{array}
\right\}
\]
On the other hand, if in our language we also had an atom $\prp$ with the same theory $\Theory{}$, then we should add a law $(\Token\et\Coffee\et\Hot\et\prp)\imp\nec\buy(\Hot\ou\prp)$ to meet minimal change by preserving effects that are not relevant to $\neg\conseq$.

\myskip

The execution for contexts $\Token\et\neg\Coffee\et\neg\Hot$ and $\Token\et\neg\Coffee\et\Hot$ are analogous and the algorithm ends with $\ErasureSet{\Theory{}}{\Token\imp\nec\buy\Hot}=\{\Theory{1}',\Theory{2}',\Theory{3}'\}$, where:
\[
\Theory{2}'=\left\{
  \begin{array}{c}
	\Coffee\imp\Hot, \Token\imp\poss\buy\top, \\
	 \Token\imp\nec\buy\neg\Token, \neg\Token\imp\nec\buy\bot, \\
	(\Coffee\et\neg(\Token\et\neg\Coffee\et\neg\Hot))\imp\nec\buy\Coffee, \\
	(\Hot\et\neg(\Token\et\neg\Coffee\et\neg\Hot))\imp\nec\buy\Hot, \\
	(\neg\Coffee\et\neg(\Token\et\neg\Coffee\et\neg\Hot))\imp\nec\buy\Coffee, \\
	(\Token\et\neg\Coffee\et\neg\Hot)\imp\nec\buy(\Coffee\ou\neg\Hot)
  \end{array}
  \right\}
\]

\[
\Theory{3}'=\left\{
  \begin{array}{c}
	\Coffee\imp\Hot, \Token\imp\poss\buy\top, \\
	 \Token\imp\nec\buy\neg\Token, \neg\Token\imp\nec\buy\bot, \\
	(\Coffee\et\neg(\Token\et\neg\Coffee\et\Hot))\imp\nec\buy\Coffee, \\
	(\Hot\et\neg(\Token\et\neg\Coffee\et\Hot))\imp\nec\buy\Hot, \\
	(\neg\Coffee\et\neg(\Token\et\neg\Coffee\et\Hot))\imp\nec\buy\Coffee, \\
	(\Token\et\neg\Coffee\et\Hot)\imp\nec\buy(\Hot\ou\neg\Coffee), \\
	(\Token\et\neg\Coffee\et\Hot)\imp\nec\buy(\Coffee\ou\neg\Hot)
  \end{array}
  \right\}
\]

Looking at Figure~\ref{ExErasureEffectModelCoffeeMachine2}, we can see the correspondence between these theories and their respective models.

% A ver, parece que não é verdade:
%***
%We can have frame axioms that may preclude $\poss\act\neg\conseq$ but that are not %necessary to have $\nec\act\conseq$. In our example, $\Hot\imp\nec\buy\Hot$ is not necessary %to $\Token\imp\nec\buy\Hot$, since this can be derived from $\Coffee\imp\nec\buy\Coffee$, $
%\neg\Coffee\imp\nec\buy\Coffee$, and $\Coffee\imp\Hot$. However, without weakening $\Hot
%\imp\nec\buy\Hot$, we do not get the intended result.
%***

\subsection{Contracting Static Laws}\label{SyntaxStatErasure}

	Finally, in order to contract a static law from a theory, we can use any contraction/erasure operator $\propcontract$ for classical logic. Because contracting static laws means {\em admitting} new possible states (\cf\ the semantics), just modifying the set $\STAT{}{}$ of static laws may not be enough for the dynamic logic case. Since we in general do not necessarily know the behavior of the actions in a new discovered state of the world, a careful approach is to change the theory so that all action laws remain the same in the contexts where the contracted law is the case. In our example, if when contracting the law $\Coffee\imp\Hot$ we are not sure whether \buy\ is still executable or not, then we should weaken our executability laws specializing them to the context $\Coffee\imp\Hot$, and make \buy\ a priori inexecutable in all $\neg(\Coffee\imp\Hot)$ contexts. The operator given in Algorithm~\ref{StatOperator} formalizes this.

\begin{algorithm}
\caption{Contraction of a static law}
\label{StatOperator}
\begin{algorithmic}[1]
\onehalfspacing
\INPUT $\Theory{}$, $\fml$
\OUTPUT $\ErasureSet{\Theory{}}{\fml}$
\STATE $\ErasureSet{\Theory{}}{\fml}\GETS\ \emptyset$
\IF{$\STAT{}{}\CPLvalid\fml$}
	\FORALL{$\STAT{-}{}\in\STAT{}{}\propcontract\fml$}
		\STATE \parbox{4cm}{\[ \Theory{}'\GETS\ \left( \begin{array}{c}
			((\Theory{}\setminus\STAT{}{}) \cup \STAT{-}{})\setminus\EXE{}{\act}\  \cup \\
			\{(\antec_{i}\et\fml)\imp\poss\act\top : 
			\antec_{i}\imp\poss\act\top \in\EXE{}{\act}\}\ \cup \\
			   \{\neg\fml\imp\nec\act\bot\}
			\end{array}
			\right)
			\]
		} % End \parbox
		\STATE $\ErasureSet{\Theory{}}{\fml}\GETS\ 
			       \ErasureSet{\Theory{}}{\fml}\cup\{\Theory{}'\}$
	\ENDFOR
\ELSE
	\STATE $\ErasureSet{\Theory{}}{\fml}\GETS\ \{\Theory{}\}$
\ENDIF
\RETURN $\ErasureSet{\Theory{}}{\fml}$
\end{algorithmic}
\end{algorithm}
\singlespacing

	In our running example, contracting the law $\Coffee\imp\Hot$ from $\Theory{}$ produces $
\ErasureSet{\Theory{}}{\Coffee\imp\Hot}=\{\Theory{1}',\Theory{2}'\}$, where

\[
\Theory{1}'=\left\{
  \begin{array}{c}
	  \neg(\neg\Token\et\Coffee\et\neg\Hot), \\
	   (\Token\et\Coffee\imp\Hot)\imp\poss\buy\top, \\
	  \neg\Coffee\imp\nec\buy\Coffee, \Token\imp\nec\buy\neg\Token, \\
	  \neg\Token\imp\nec\buy\bot, \Coffee\imp\nec\buy\Coffee, \\
	   \Hot\imp\nec\buy\Hot, (\Coffee\et\neg\Hot)\imp\nec\buy\bot
  \end{array}
  \right\}
\]

\[
\Theory{2}'=\left\{
  \begin{array}{c}
	  \neg(\Token\et\Coffee\et\neg\Hot), \\
	   (\Token\et\Coffee\imp\Hot)\imp\poss\buy\top, \\
	  \neg\Coffee\imp\nec\buy\Coffee, \Token\imp\nec\buy\neg\Token, \\
	  \neg\Token\imp\nec\buy\bot, \Coffee\imp\nec\buy\Coffee, \\
	   \Hot\imp\nec\buy\Hot, (\Coffee\et\neg\Hot)\imp\nec\buy\bot
  \end{array}
  \right\}
\]

\myskip

	Observe that the effect laws are not affected by the change: as far as we do not pronounce 
ourselves about the executability of some action in the new added world, all the effect laws remain true in it.

\myskip

	If the knowledge engineer is not happy with $(\Coffee\et\neg\Hot)\imp\nec\buy\bot$, she can contract this formula from the theory using Algorithm~\ref{EffectOperator}. Ideally, besides stating that \buy\ is executable in the context $\Coffee\et\neg\Hot$, we should want to specify its outcome in this context as well. For example, we could want $(\Coffee\et\neg\Hot)\imp\poss\buy\Hot$ to be true in the result. This would require theory {\em revision}. See Section~\ref{SemanticsRevision} for the semantics of such an operation.

%%%%%%%%%%%%%%%%%%%%%%%%%%%%%%%%%%%%%%%%%%%%%%%%%%%
\section{Correctness of the Operators}\label{Correctness}
%%%%%%%%%%%%%%%%%%%%%%%%%%%%%%%%%%%%%%%%%%%%%%%%%%%

	We here address the correctness of our algorithms \wrt\ our semantics for contraction.

\subsection{Two Counter-Examples}

Let the theory $\Theory{}=\{\prp_{1}\imp\poss\act\top,(\neg\prp_{1}\ou\prp_{2})\imp\nec\act\bot,\nec\act\neg\prp_{2}\}$ and consider its model $\model$ depicted in Figure~\ref{UnsoundnessErasure}. (Notice that $\Theory{}\PDLvalid\neg(\prp_{1}\et\prp_{2})$.) When contracting $\prp_{1}\imp\nec\act\neg\prp_{2}$ in $\model$, we get $\model'$ in Figure~\ref{UnsoundnessErasure}.

\begin{figure}[h]
\begin{center}
\parbox{1.5cm}{
$\model$ :
}\parbox{4.2cm}{
\scriptsize{
\begin{picture}(40,42)(0,0) % (0,0) Ž embaixo ˆ esquerda, (50,50) Ž acima ˆ direita
    \gasset{Nw=9,Nh=4,linewidth=0.3}
    \thinlines
    \node[Nw=10](A1)(0,30){$\prp_{1},\neg\prp_{2}$}
    \node[Nw=12](A2)(30,30){$\neg\prp_{1},\neg\prp_{2}$}
    \node[Nw=10](A3)(15,5){$\neg\prp_{1},\prp_{2}$}
    \drawedge(A1,A2){$\act$}
    \drawloop[loopangle=90](A1){$\act$}
\end{picture}
} % End \scriptsize
} % End \parbox
\parbox{1.5cm}{
$\model'$ :
}\parbox{4.2cm}{
\scriptsize{
\begin{picture}(40,42)(0,0) % (0,0) Ž embaixo ˆ esquerda, (50,50) Ž acima ˆ direita
    \gasset{Nw=9,Nh=4,linewidth=0.3}
    \thinlines
    \node[Nw=10](A1)(0,30){$\prp_{1},\neg\prp_{2}$}
    \node[Nw=12](A2)(30,30){$\neg\prp_{1},\neg\prp_{2}$}
    \node[Nw=10](A3)(15,5){$\neg\prp_{1},\prp_{2}$}
    \drawedge(A1,A2){$\act$}
    \drawloop[loopangle=90](A1){$\act$}
    \drawedge[ELside=r](A1,A3){$\act$}
\end{picture}
} % End \scriptsize
} % End \parbox
\end{center}
\caption{A model $\model$ of $\Theory{}$ and the result $\model'$ of contracting $\prp_{1}\imp\nec\act\neg\prp_{2}$ in it.}
\label{UnsoundnessErasure}
\end{figure}

Now contracting $\prp_{1}\imp\nec\act\neg\prp_{2}$ from $\Theory{}$ using Algorithm~\ref{EffectOperator} gives $\ErasureSet{\Theory{}}{\prp_{1}\imp\nec\act\neg\prp_{2}}=\{\Theory{}'\}$, where
\[
\Theory{}'=\left\{\begin{array}{c}
	\prp_{1}\imp\poss\act\top, (\neg\prp_{1}\ou\prp_{2})\imp\nec\act\bot, \\
	(\prp_{1}\et\neg\prp_{2})\imp\nec\act(\neg\prp_{2}\ou\prp_{2}), \\
	(\prp_{1}\et\neg\prp_{2})\imp\nec\act(\neg\prp_{2}\ou\prp_{1})
	\end{array}
	\right\}
\]
Notice that the formula $(\prp_{1}\et\neg\prp_{2})\imp\nec\act(\neg\prp_{2}\ou\prp_{1})$ is put in $\Theory{}'$ by Algorithm~\ref{EffectOperator} because there is $\{\prp_{1}\}\subseteq\Lit$ such that $\STAT{}{}\notCPLvalid(\prp_{1}\et\prp_{2})\imp\bot$ and $\Theory{}\notPDLvalid(\prp_{1}\et\neg\prp_{2})\imp\nec\act\neg\prp_{1}$. Clearly $\notMvalid{\model'}\Theory{}'$ and no theory in $\ErasureSet{\Theory{}}{\prp_{1}\imp\nec\act\neg\prp_{2}}$ has $\model'$ as model. This means that the contraction operators are not correct.

	This issue arises because Algorithm~\ref{EffectOperator} tries to allow an arrow from the $\prp_{1}\et\neg\prp_{2}$-world to a $\prp_{2}$-world that is closest to it, \viz\ $\{\prp_{1},\prp_{2}\}$, but has no way of knowing that such a world does not exist. A remedy for that is replacing the test $\Theory{}\notPDLtheorem(\term'\et\bigwedge_{\lit\in L}\lit)\imp\bot$ for $\STAT{}{}\notCPLtheorem(\term'\et\bigwedge_{\lit\in L}\lit)\imp\bot$, but that would increase even more the complexity of the algorithm. A better option would be to have $\STAT{}{}$ `complete enough' to allow the algorithm to determine the worlds to which a new transition could exist.

\myskip

	The other way round, it does not hold in general that the models of each $\Theory{}'\in\ErasureSet{\Theory{}}{\modfml}$ result from the semantic contraction of  models of $\Theory{}$ by $\modfml$. To see this suppose that there is only one atom $\prp$ and one action $\act$, and consider the action theory $\Theory{}=\{\prp\imp\nec\act\bot, \poss\act\top\}$. The only model of $\Theory{}$ is $\model=\tuple{\{\{\neg\prp\}\},\{(\{\neg\prp\},\{\neg\prp\})\}}$ in Figure~\ref{IncompletenessErasure}.
	
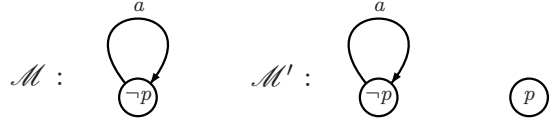
\begin{figure}[h]
\begin{center}
\parbox{1.7cm}{
$\model$ :
}\parbox{1.5cm}{
\scriptsize{
\begin{picture}(7,20)(0,0) % (0,0) Ž embaixo ˆ esquerda, (50,50) Ž acima ˆ direita
    \gasset{Nw=5,Nh=5,linewidth=0.3}
    \thinlines
    \node(A1)(0,7){$\neg\prp$}
    \drawloop[loopangle=90](A1){$\act$}
\end{picture}
}
}\parbox{1.7cm}{
$\model'$ :
}\parbox{2cm}{
\scriptsize{
\begin{picture}(7,20)(0,0) % (0,0) Ž embaixo ˆ esquerda, (50,50) Ž acima ˆ direita
    \gasset{Nw=5,Nh=5,linewidth=0.3}
    \thinlines
    \node(A1)(0,7){$\neg\prp$}
    \node(A2)(20,7){$\prp$}
    \drawloop[loopangle=90](A1){$\act$}
\end{picture}
}
}
\end{center}
\caption{Incompleteness of contraction: a model $\model$ of $\Theory{}$ and a model $\model'$ of the theory resulting from contracting $\prp\imp\poss\act\top$ from $\Theory{}$.}
\label{IncompletenessErasure}
\end{figure}

	By definition, $\erasure{\model}{\prp\imp\poss\act\top}=\{\model\}$. On the other hand, $\ErasureSet{\Theory{}}{\prp\imp\poss\act\top}$ is the singleton $\{\Theory{}'\}$ such that $\Theory{}'=\{\prp\imp\nec\act\bot, \neg\prp\imp\poss\act\top\}$. Then $\model'=\tuple{\{\{\neg\prp\},\{\prp\}\},(\{\neg\prp\},\{\neg\prp \})}$ in Figure~\ref{IncompletenessErasure} is a model of the contracted theory. Clearly, $\model'$ does not result from the semantic contraction of $\prp\imp\poss\act\top$ from $\model$: while $\neg\prp$ is valid in the contraction of the models of $\Theory{}$, it is not valid in the models of $\Theory{}'$. This means that the operators are not complete.

	This problem occurs because, in our example, the worlds that are forbidden by $\Theory{}$, \eg\ $\{\prp\}$, are not preserved as such in $\Theory{}'$. When contracting an executability or an effect law, we are not supposed to change the possible worlds of a theory (\cf\ Section~\ref{SemanticsErasure}).

\myskip

	Fortunately correctness of the algorithms \wrt\ our semantics can be guaranteed for those theories whose $\STAT{}{}$ is maximal, \ie, the set of static laws in $\STAT{}{}$ alone determine what worlds are authorized in the models of the theory. This is the principle of {\em modularity}~\cite{HerzigVarzinczak-Aiml04Proc05} and we briefly review it in the next section.

\subsection{Modular Theories}

\begin{definition}[Modularity~\cite{HerzigVarzinczak-Aiml04Proc05}]
An action theory $\Theory{}$ is {\em modular} if and only if for every $\fml\in\Fml$, if $\Theory{}\PDLvalid\fml$, then $\STAT{}{}\CPLvalid\fml$.
\end{definition}

	For an example of a non-modular theory, suppose that the action theory $\Theory{}$ of our coffee machine scenario were stated as
\[
\Theory{}=\left\{
  \begin{array}{c}
	 \Coffee\imp\Hot, \underline{\poss\buy\top}, \\
	  \neg\Coffee\imp\nec\buy\Coffee, \\
	   \Token\imp\nec\buy\neg\Token, \neg\Token\imp\nec\buy\bot, \\
	  \Coffee\imp\nec\buy\Coffee, \Hot\imp\nec\buy\Hot
  \end{array}
  \right\}
\]
The modified law is underlined: we have (in this case wrongly) stated that the agent can always buy at the machine. Then $\Theory{}\PDLvalid\Token$ and $\STAT{}{}\notCPLvalid\Token$.

	As the underlying multimodal logic is independently axiomatized (see Section~\ref{ActTheoriesPDL}), we can use the algorithms given by Herzig and Varzinczak~\cite{HerzigVarzinczak-Aiml04Proc05} to check whether an action theory satisfies the principle of modularity. Whenever this is not the case, the algorithms return the Boolean formulas entailed by the theory that are not consequences of $\STAT{}{}$ alone. For the theory $\Theory{}$ above, they would return $\{\Token\}$: as we stated $\poss\buy\top$, from this and $\neg\Token\imp\nec\buy\bot$ we get $\Theory{}\PDLvalid\Token$. Because $\STAT{}{}\notCPLvalid\Token$, $\Token$ is what is called an {\em implicit static law}~\cite{HerzigVarzinczak-ECAI04} of $\Theory{}$.\footnote{Implicit static laws are very closely related to veridical paradoxes~\cite{Quine62}. It turns out that they are not always intuitive. For a deep discussion on implicit static laws, see the article by Herzig and Varzinczak~\cite{HerzigVarzinczak-AIJ07}.}

\myskip

	Modular theories have interesting properties. For example, consistency can be checked by just checking consistency of the static laws in \STAT{}{}: if $\Theory{}$ is modular, then $\Theory{}\PDLvalid\bot$ if and only if $\STAT{}{}\CPLvalid\bot$. Deduction of effect laws does not need the executability ones and vice versa. Deduction of an effect of a sequence of actions $\act_{1};\ldots;\act_{n}$ (prediction) does not need to take into account the effect laws for actions other than $\act_{1},\ldots,\act_{n}$. This applies in particular to plan validation when deciding whether $\poss{\act_{1};\ldots;\act_{n}}\fml$ is the case.
	
	Similar notions to modularity have been investigated in the literature on regulation consistency~\cite{Cholvy99}, Situation Calculus~\cite{Amir-AAAI2000,HerzigVarzinczak-IJCAI05},  DL ontologies~\cite{CuencaGrauEtAl-KR2006,HerzigVarzinczak-JELIA06} and also in dynamic logic~\cite{ZhangEtAl02}. For more details on modularity in action theories, see the work by Varzinczak~\cite{Varzinczak06}.

\begin{theorem}
$\Theory{}$ is modular if and only if the big model of $\Theory{}$ is a model~of~$\Theory{}$.
\end{theorem}
\begin{proof} Let $\model_{\bigM}=\tuple{\Worlds_{\bigM},\AccRel_{\bigM}}$ be the big model of $\Theory{}$.\\

\noindent($\Rightarrow$): By definition, $\model_{\bigM}$ is such that $\Mvalid{\model_{\bigM}}\STAT{}{}\et\EFF{}{}$. It remains to show that $\Mvalid{\model_{\bigM}}\EXE{}{}$. Let $\antec_{i}\imp\poss\act\top\in\EXE{}{\act}$, and let $w\in\Worlds_{\bigM}$ be such that $\wMvalid{w}{\model_{\bigM}}\antec_{i}$. Therefore for all $\antec_{j}\in\Fml$ such that $\Theory{}\PDLvalid\antec_{j}\imp\nec\act\bot$, we must have $\notwMvalid{w}{\model_{\bigM}}\antec_{j}$, because $\Theory{}\PDLvalid\neg(\antec_{i}\et\antec_{j})$, and as $\Theory{}$ is modular, $\STAT{}{}\CPLvalid\neg(\antec_{i}\et\antec_{j})$, and hence $\Mvalid{\model_{\bigM}}\neg(\antec_{i}\et\antec_{j})$. Then by the construction of $\model_{\bigM}$, there is some $w'\in\Worlds_{\bigM}$ such that $\wMvalid{w'}{\model_{\bigM}}\conseq$ for all $\antec\imp\nec\act\conseq\in\EFF{}{\act}$ such that $\wMvalid{w}{\model_{\bigM}}\antec$. Thus $\ract(w)\neq\emptyset$ and $\Mvalid{\model_{\bigM}}\antec_{i}\imp\poss\act\top$.\\ % Hence $\Mvalid{\model_{\bigM}}\Theory{}$.
 
\noindent($\Leftarrow$): Suppose $\Theory{}$ is not modular. Then there must be some $\fml\in\Fml$ such that $\Theory{}\PDLvalid\fml$ and $\STAT{}{}\notCPLvalid\fml$. This means that there is $\val\in\valuations{\STAT{}{}}$ such that $v\notsat\fml$. As $\val\in\Worlds_{\bigM}$ (because $\Worlds_{\bigM}$ contains all possible valuations of $\STAT{}{}$), $\model_{\bigM}$ is not a model of $\Theory{}$.
\end{proof}

\subsection{Correctness Under Modularity}

	The following theorem establishes that the semantic contraction of a formula $\modfml$ from 
the set of models of an action theory $\Theory{}$ produces models of some contracted theory in~$\ErasureSet{\Theory{}}{\modfml}$.

\begin{theorem}\label{TheoremCompleteness}
Let $\Theory{}$ be modular, and $\modfml$ be a law. For all $\ModelSet'\in\ErasureModels{\ModelSet}{\modfml}$ such that $\Mvalid{\model}\Theory{}$ for every $\model\in\ModelSet$, there is $\Theory{}'\in\ErasureSet{\Theory{}}{\modfml}$ such that $\Mvalid{\model'}\Theory{}'$ for every $\model'\in\ModelSet'$.
\end{theorem}
\begin{proof}
See Appendix~\ref{ProofTheoremCompleteness}.
\end{proof}

\myskip

	The next theorem establishes the other way round: models of theories in 
$\ErasureSet{\Theory{}}{\modfml}$ are all models of the semantic contraction of $\modfml$ from 
models of $\Theory{}$.

\begin{theorem}\label{TheoremSoundness}
Let $\Theory{}$ be modular, $\modfml$ a law, and $\Theory{}'\in\ErasureSet{\Theory{}}{\modfml}$. For all $\model'$ such that $\Mvalid{\model'}\Theory{}'$, there is $\ModelSet'\in\ErasureModels{\ModelSet}{\modfml}$ such that $\model'\in\ModelSet'$ and $\Mvalid{\model}\Theory{}$ for every $\model\in\ModelSet$.
\end{theorem}
\begin{proof}
See Appendix~\ref{ProofTheoremSoundness}.
\end{proof}

\myskip

	With these two theorems one gets correctness of the operators:

\begin{corollary}\label{CorollaryCorrectness}
Let $\Theory{}$ be modular, $\modfml$ a law, and $\Theory{}'\in\ErasureSet{\Theory{}}{\modfml}$. Then $\Theory{}'\PDLvalid\varPsi$ if and only if $\Mvalid{\model'}\varPsi$ for every $\model'\in\ModelSet'$ such that $\ModelSet'\in\ErasureModels{\ModelSet}{\modfml}$ for some $\ModelSet$ such that $\Mvalid{\model}\Theory{}$ for all $\model\in\ModelSet$.
\end{corollary}
\begin{proof}

\noindent($\Rightarrow$): Let $\model'$ be such that $\Mvalid{\model'}\Theory{}'$. By Theorem~\ref{TheoremSoundness}, there is $\ModelSet'\in\ErasureModels{\ModelSet}{\modfml}$ such that $\model'\in\ModelSet'$ for some $\ModelSet$ such that $\Mvalid{\model}\Theory{}$ for all $\model\in\ModelSet$. From this and $\Theory{}'\PDLvalid\varPsi$, we have $\Mvalid{\model'}\varPsi$.\\

\noindent($\Leftarrow$): Suppose $\Theory{}'\notPDLvalid\varPsi$. (We show that there is some model $\model'\in\ModelSet'$ such that $\ModelSet'\in\ErasureModels{\ModelSet}{\modfml}$ for some $\ModelSet$ with $\Mvalid{\model}\Theory{}$ for all $\model\in\ModelSet$, and $\notMvalid{\model'}\varPsi$.)

Given that $\Theory{}$ is modular, by Lemma~\ref{PreservationModularity} $\Theory{}'$ is modular, too. Then, by Lemma~\ref{LemmaOnlyModelsValS}, there is $\model'=\tuple{\valuations{{\STAT{}{}}'},\AccRel'}$ such that $\notMvalid{\model'}\varPsi$. Clearly $\Mvalid{\model'}\Theory{}'$, and from Lemma~\ref{LemmaComesFromSemantics} the result follows.
\end{proof}

%%%%%%%%%%%%%%%%%%%%%%%%%%%%%%%%%%%%%%%%%%%%%%%%%%%
\section{Assessment of Postulates for Change}\label{Postulates}
%%%%%%%%%%%%%%%%%%%%%%%%%%%%%%%%%%%%%%%%%%%%%%%%%%%

	Do our action theory change operators satisfy the classical postulates for change? Before answering this question, one should ask: do our operators behave like revision or update operators? We here address this issue and then show which postulates for theory change are satisfied by our definitions.
	
\subsection{Contraction or Erasure?}

	The distinction between revision/contraction and update/erasure for classical theories is historically controversial in the literature. The same is true for the case of modal theories describing actions and their effects. We here rephrase Katsuno and Mendelzon's definitions~\cite{KatsunoMendelzon92} in our terms so that we can see to which one our method is closer.

\myskip

	In Katsuno and Mendelzon's view, contracting a law $\modfml$ from an action theory $\Theory{}$ intuitively means that the description of the possible behavior of the dynamic world $\Theory{}$ must be adjusted to the possibility of $\modfml$ being false. This amounts to selecting from the models of $\neg\modfml$ those that are closest to models of $\Theory{}$ and allow them as models of the result.
	
	In contrast, update methods select, for each model $\model$ of $\Theory{}$, the set of models of $\modfml$ that are closest to $\model$. Erasing $\modfml$ from $\Theory{}$ means adding models to $\Theory{}$; for each model $\model$, we add all those models closest to $\model$ in which $\modfml$ is false. Hence, from our constructions so far it seems that our operators are closer to update than to revision.

	Moreover, according to Katsuno and Mendelzon's view~\cite{KatsunoMendelzon92}, our change operators would also be classified as update because we make modifications in each model independently, \ie, without changing other models.\footnote{Even if when contracting an effect law from one particular model we need to check the other models of the theory, those are not modified.} Besides that, in our setting a different ordering on the resulting models is induced by each model of $\Theory{}$ (see Definitions~\ref{DefErasureModelsExe}, \ref{DefErasureModelsEff} and~\ref{DefErasureModelsStat}), which according to Katsuno and Mendelzon is a typical property of an update/erasure method.
	
	Nevertheless, things get quite different when it comes to the postulates for theory change.

\subsection{The Postulates}

	We here analyze the behavior of our action theory change operators \wrt\ Katsuno and Mendelzon's postulates and variants. Let $\Theory{}=\STAT{}{}\cup\EFF{}{}\cup\EXE{}{}$ denote an action theory and $\modfml$ denote a law.

\paragraph{Monotonicity Postulate:} $\Theory{}\PDLvalid\Theory{}'$, for all $\Theory{}'\in
\ErasureSet{\Theory{}}{\modfml}$.

\myskip

	This postulate is our version of Katsuno and Mendelzon's (C1) and (E1) postulates for contraction and erasure, respectively, and is satisfied by our change operators. The proof is in Lemma~\ref{LemmaMonotony}. Such a postulate is not satisfied by the operators proposed by Herzig~\etal~\cite{HerzigEtAl-ECAI06}: there when removing \eg\ an executability law $\antec\imp\poss\act\top$ one may make $\antec\imp\nec\act\bot$ valid in all models of the resulting theory.

\paragraph{Preservation Postulate:} If $\Theory{}\notPDLvalid\modfml$, then $\PDLvalid\Theory{}
\sii\Theory{}'$, for all $\Theory{}'\in\ErasureSet{\Theory{}}{\modfml}$.

\myskip

	This is Katsuno and Mendelzon's (C2) postulate. Our operators satisfy it as far as whenever $\Theory{}\notPDLvalid\modfml$, then the models of the resulting theory are exactly the models of $\Theory{}$, because these are the minimal models falsifying $\modfml$.

	The corresponding version of Katsuno and Mendelzon's (E2) postulate about erasure, \ie, if $\Theory{}\PDLvalid\neg\modfml$, then $\PDLvalid\Theory{}\sii\Theory{}'$, for all $\Theory{}'\in\ErasureSet{\Theory{}}{\modfml}$, is clearly also satisfied by our operators as a special case of the postulate above. Satisfaction of (C2) indicates that our operators are closer to contraction than to erasure.

\paragraph{Success Postulate:} If $\Theory{}\notPDLvalid\bot$ and $\notPDLvalid\modfml$, then $\Theory{}'\notPDLvalid\modfml$, for all $\Theory{}'\in\ErasureSet{\Theory{}}{\modfml}$.

\myskip

	This postulate is our version of Katsuno and Mendelzon's (C3) and (E3) postulates. If $\modfml$ is a propositional $\fml\in\Fml$, our operators satisfy it, as long as the classical propositional change operator satisfies it. For the general case, however, as stated the postulate is not always satisfied. This is shown by the following example: let $\Theory{}=\{\neg\prp,\poss\act\top,\prp\imp\nec\act\bot\}$. Note that $\Theory{}$ is modular and consistent. Now, contracting the (contingent) formula $\prp\imp\poss\act\top$ from $\Theory{}$ gives us $\Theory{}'=\Theory{}$. Clearly $\Theory{}'\PDLvalid\prp\imp\poss\act\top$. This happens because, despite not being a tautology, $\prp\imp\poss\act\top$ is a `trivial' formula \wrt\ $\Theory{}$: since $\neg\prp$ is valid in all $\Theory{}$-models, $\prp\imp\poss\act\top$ is trivially true in these models.
	
	Fortunately, for all those formulas that are non-trivial consequences of the theory, our operators guarantee success of contraction:

\begin{theorem}
Let $\Theory{}$ be consistent, and $\modfml$ be an executability or an effect law such that $\STAT{}{}\notPDLvalid\modfml$. If $\Theory{}$ is modular, then $\Theory{}'\notPDLvalid\modfml$ for every $\Theory{}'\in\ErasureSet{\Theory{}}{\modfml}$.
\end{theorem}
\begin{proof}
Suppose there is $\Theory{}'\in\ErasureSet{\Theory{}}{\modfml}$ such that $\Theory{}'\PDLvalid\modfml$. As $\Theory{}$ is modular, Corollary~\ref{CorollaryCorrectness} gives us $\Mvalid{\model'}\modfml$ for every $\model'\in\ModelSet'$ such that $\ModelSet'\in\ErasureModels{\ModelSet}{\modfml}$, where $\ModelSet=\{\model : \Mvalid{\model}\Theory{} \text{ and } \model=\tuple{\valuations{\STAT{}{}},\AccRel}\}$.

If $\Mvalid{\model'}\modfml$ for every $\model'\in\ModelSet'$, then even for $\model''\in\ModelSet'\setminus\ModelSet$ we have $\Mvalid{\model''}\modfml$. But $\model''\in\ErasureModels{\model}{\modfml}$ for some $\model\in\ModelSet$, and by definition $\notMvalid{\model''}\modfml$. Hence $\ErasureModels{\model}{\modfml}=\emptyset$, and then the truth of $\modfml$ in $\model$ does not depend on $\ract$. Then, whether $\modfml$ has the form $\antec\imp\poss\act\top$ or $\antec\imp\nec\act\conseq$, for $\antec,\conseq\in\Fml$, this holds only if $\STAT{}{}\CPLvalid\neg\antec$ (see Definitions~\ref{DefErasureOneModelExe} and~\ref{DefErasureOneModelEff}), in which case we get $\STAT{}{}\PDLvalid\modfml$.
\end{proof}

\paragraph{Equivalences Postulate:} If $\PDLvalid\Theory{1}\sii\Theory{2}$ and $\PDLvalid\modfml_{1}\sii\modfml_{2}$, then $\PDLvalid\Theory{1}'\sii\Theory{2}'$, for $\Theory{1}'\in\ErasureSet{(\Theory{1})}{\modfml{2}}$ and $\Theory{2}'\in\ErasureSet{(\Theory{2})}{\modfml_{1}}$.

\myskip

	This postulate corresponds to Katsuno and Mendelzon's (C4) and (E4) postulates. Under modularity and the assumption that the propositional change operator satisfies (C4)/(E4), our operations satisfy this postulate:

\begin{theorem}
Let $\Theory{1}$ and $\Theory{2}$ be modular. If $\PDLvalid\Theory{1}\sii\Theory{2}$ and $\PDLvalid\modfml_{1}\sii\modfml_{2}$, then for each $\Theory{1}'\in\ErasureSet{(\Theory{1})}{\modfml_{2}}$ there is $\Theory{2}'\in\ErasureSet{(\Theory{2})}{\modfml_{1}}$ such that $\PDLvalid\Theory{1}'\sii\Theory{2}'$, and vice-versa.
\end{theorem}
\begin{proof}
The proof follows straight from our results: since $\PDLvalid\Theory{1}\sii\Theory{2}$ and $\PDLvalid\modfml_{1}\sii\modfml_{2}$, they have pairwise the same models. Hence, given $\model$ such that $\Mvalid{\model}\Theory{1}$ and $\Mvalid{\model}\Theory{2}$, the semantic contraction of $\modfml_{1}$ and that of $\modfml_{2}$ from $\model$ have the same operations on $\model$. As $\Theory{1}$ and $\Theory{2}$ are modular, Corollary~\ref{CorollaryCorrectness} guarantees we get the same syntactical results. Moreover, as the classical operator $\propcontract$ satisfies (C4)/(E4), if follows that $\PDLvalid\Theory{1}'\sii\Theory{2}'$.
\end{proof}

\paragraph{Recovery Postulate:} $\Theory{}'\cup\{\modfml\}\PDLvalid\Theory{}$, for all $\Theory{}'\in\ErasureSet{\Theory{}}{\modfml}$.

\myskip

	This is the action theory counterpart of Katsuno and Mendelzon's (C5) and (E5) postulates. Again we rely on modularity in order to satisfy it.

\begin{theorem}
Let $\Theory{}$ be modular. $\Theory{}'\cup\{\modfml\}\PDLvalid\Theory{}$, for all $\Theory{}'\in\ErasureSet{\Theory{}}{\modfml}$.
\end{theorem}
\begin{proof}
If $\Theory{}\notPDLvalid\modfml$, because our operators satisfy the preservation postulate, $\Theory{}'=\Theory{}$, and then the result follows by monotonicity.

Let $\Theory{}\PDLvalid\modfml$, and let $\ModelSet'$ denote the set of all models of $\Theory{}'$. As $\Theory{}$ is modular, by Corollary~\ref{CorollaryCorrectness} every $\model'\in\ModelSet'$ is such that either $\Mvalid{\model'}\Theory{}$ (and then $\Mvalid{\model'}\modfml$) or $\model'\in\erasure{\model}{\modfml}$ (and then $\model'\in\ErasureModels{\model}{\modfml}$) for some $\model$ such that $\Mvalid{\model}\Theory{}$.

Let $\ModelSet''$ denote the set of all models of $\Theory{}'\cup\{\modfml\}$. Clearly $\ModelSet''\subseteq\ModelSet'$, by monotonicity. Moreover, every $\model''\in\ModelSet''$ is such that $\Mvalid{\model''}\modfml$, hence $\model''\notin\ErasureModels{\model}{\modfml}$ for every $\model$ such that $\Mvalid{\model}\Theory{}$, and then $\model''\notin\erasure{\model}{\modfml}$, for any $\model$ model of $\Theory{}$. Thus $\model''$ is a model of $\Theory{}$ and then $\Theory{}'\cup\{\modfml\}\PDLvalid\Theory{}$.
\end{proof}

\myskip

	Let $\bigvee\ErasureSet{\Theory{}}{\modfml}$ denote the disjunction of all $\Theory{}'$ in $\ErasureSet{\Theory{}}{\modfml}$.

\paragraph{Disjunctive rule:} $\ErasureSet{(\Theory{1}\ou\Theory{2})}{\modfml}$ is equivalent to $\bigvee\ErasureSet{(\Theory{1})}{\modfml}\ou\bigvee\ErasureSet{(\Theory{2})}{\modfml}$.

\myskip

	This is our version of (E8) erasure postulate by Katsuno and Mendelzon. Clearly our syntactical operators do not manage to contract a law from a disjunction of theories $\Theory{1}\ou\Theory{2}$. Nevertheless, by proving that it holds in the semantics, from the correctness of our operators, we get an equivalent operation. Again the fact that the theories under concern are modular gives us the result.

\begin{theorem}
Let $\Theory{1}$ and $\Theory{2}$ be modular, and $\modfml$ be a law. Then
\[
\PDLvalid\bigvee\ErasureSet{(\Theory{1}\ou\Theory{2})}{\modfml}\sii(\bigvee\ErasureSet{(\Theory{1})}{\modfml}\ou\bigvee\ErasureSet{(\Theory{2})}{\modfml})
\]
\end{theorem}
\begin{proof}

\noindent($\Leftarrow$): Let $\model'$ be such that $\Mvalid{\model'}\bigvee\ErasureSet{(\Theory{1})}{\modfml}\ou\bigvee\ErasureSet{(\Theory{2})}{\modfml}$. Then $\Mvalid{\model'}\bigvee\ErasureSet{(\Theory{1})}{\modfml}$ or $\Mvalid{\model'}\bigvee\ErasureSet{(\Theory{2})}{\modfml}$. Suppose $\Mvalid{\model'}\bigvee\ErasureSet{(\Theory{1})}{\modfml}$ (the other case is analogous). Then there is $(\Theory{1})'\in\ErasureSet{(\Theory{1})}{\modfml}$ such that $\Mvalid{\model'}(\Theory{1})'$. Then by Corollary~\ref{CorollaryCorrectness}, there is $\ModelSet'\in\ErasureModels{\ModelSet}{\modfml}$ such that $\model'\in\ModelSet'$, for $\ModelSet$ a set of models of $\Theory{1}$. Then $\model'$ is a model resulting from contracting $\modfml$ from models of $\Theory{1}$ , and then $\model'$ also results from contracting $\modfml$ in models of $\Theory{1}\ou\Theory{2}$, \viz\ those models of $\Theory{1}$. Then by Corollary~\ref{CorollaryCorrectness}, there is $(\Theory{1}\ou\Theory{2})'\in\ErasureSet{(\Theory{1}\ou\Theory{2})}{\modfml}$ such that $\Mvalid{\model'}(\Theory{1}\ou\Theory{2})'$, and then $\Mvalid{\model'}\bigvee\ErasureSet{(\Theory{1}\ou\Theory{2})}{\modfml}$.\\

\noindent($\Rightarrow$): Let $\model'$ be such that $\Mvalid{\model'}\bigvee\ErasureSet{(\Theory{1}\ou\Theory{2})}{\modfml}$. Then there is $(\Theory{1}\ou\Theory{2})'\in\ErasureSet{(\Theory{1}\ou\Theory{2})}{\modfml}$ such that $\Mvalid{\model'}(\Theory{1}\ou\Theory{2})'$. By Corollary~\ref{CorollaryCorrectness}, there is $\ModelSet'\in\ErasureModels{\ModelSet}{\modfml}$ such that $\model'\in\ModelSet'$, for $\ModelSet$ a set of models of $\Theory{1}\ou\Theory{2}$. Then $\model'$ is a model resulting from contracting $\modfml$ from models of $\Theory{1}\ou\Theory{2}$. Hence $\model'$ results from contracting $\modfml$ from models of $\Theory{1}$ or from models of $\Theory{2}$. Suppose the former is the case (the second is analogous). Then by Corollary~\ref{CorollaryCorrectness} there is $(\Theory{1})'\in\ErasureSet{(\Theory{1})}{\modfml}$ such that $\Mvalid{\model'}(\Theory{1})'$, and then $\Mvalid{\model'}\bigvee\ErasureSet{(\Theory{1})}{\modfml}$.
\end{proof}

\myskip

	We have thus shown that our constructions satisfy (E8) postulate. Nevertheless there is no evidence whether it is really expected here. This supports our position that our operators' behavior is closer to contraction than to erasure. To finish up we state a new postulate:

\paragraph{Preservation of modularity:} If $\Theory{}$ is modular, then every $\Theory{}'\in\ErasureSet{\Theory{}}{\modfml}$ is modular.

\myskip

	Changing a modular theory should not make it nonmodular. This is not a standard postulate, but we think that as a good property modularity should be preserved across changing an action theory. If so, this means that whether a theory is modular or not can be checked once for all and one does not need to care about it during the future evolution of the action theory, \ie, when other changes will be made on it. Our operators satisfy this postulate and the proof is given in Appendix~\ref{ProofTheoremSoundness}.

%%%%%%%%%%%%%%%%%%%%%%%%%%%%%%%%%%%%%%%%%%%%%%%%%%%
\section{A Semantics for Action Theory Revision}\label{SemanticsRevision}
%%%%%%%%%%%%%%%%%%%%%%%%%%%%%%%%%%%%%%%%%%%%%%%%%%%

	So far we have analyzed the case of contraction: when evolving a theory one realizes that it is too strong and hence it has to be weakened. Let's now take a look at the other way round, \ie, the theory is too liberal and the agent discovers new {\em laws} about the world that should be added to her beliefs, which amounts to strengthening them.
	
	Suppose the action theory of our scenario example were initially stated as follows:
\[
\Theory{}=\left\{
  \begin{array}{c}
	 \Coffee\imp\Hot, \Token\imp\poss\buy\top, \\
	 \neg\Coffee\imp\nec\buy\Coffee, \neg\Token\imp\nec\buy\bot, \\
	 \Coffee\imp\nec\buy\Coffee, \Hot\imp\nec\buy\Hot
  \end{array}
  \right\}
\]
Then the big-model of $\Theory{}$ is as shown in Figure~\ref{ExBigModelRevision}.

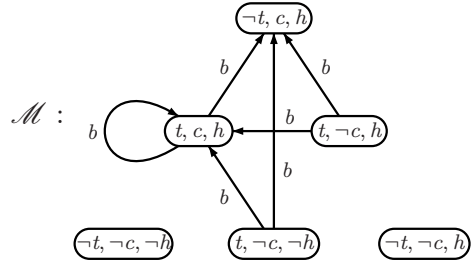
\begin{figure}[h]
\begin{center}
\parbox{1.5cm}{
$\model$ :
}\parbox{4cm}{
\scriptsize{
\begin{picture}(35,35)(0,0) % (0,0) Ž embaixo ˆ esquerda, (50,50) Ž acima ˆ direita
    \gasset{Nw=9,Nh=4,linewidth=0.3}
    \thinlines
    \node(A1)(10,15){$\AToken,\ACoffee,\AHot$}
    \node[Nw=10](A2)(20,30){$\neg\AToken,\ACoffee,\AHot$}
    \node[Nw=10](A3)(30,15){$\AToken,\neg\ACoffee,\AHot$}
    \node[Nw=13](A4)(0,0){$\neg\AToken,\neg\ACoffee,\neg\AHot$}
    \node[Nw=12](A5)(40,0){$\neg\AToken,\neg\ACoffee,\AHot$}
    \node[Nw=12](A6)(20,0){$\AToken,\neg\ACoffee,\neg\AHot$}
    \drawedge[ELside=l](A1,A2){$\Abuy$}
    \drawedge[ELside=r](A3,A2){$\Abuy$}
    \drawedge[ELside=r,ELpos=33](A6,A2){$\Abuy$}
    \drawloop[ELside=l,loopangle=180](A1){$\Abuy$}
    \drawedge[ELside=l](A6,A1){$\Abuy$}
    \drawedge[ELside=r,ELpos=40](A3,A1){$\Abuy$}
%    \drawedge[curvedepth=-3,linewidth=0.4,ELside=r](A1,A3){$\act_{1}$}
%    \drawedge[curvedepth=3,linewidth=0.4,ELside=l](A1,A3){$\act_{2}$}
%    \drawedge[linewidth=0.4,ELside=l](A2,A3){$\act_{1}$}
%    \drawloop[linewidth=0.4,loopangle=90](A2){$\act_{1}$}
%    \drawloop[linewidth=0.4,loopangle=270](A3){$\act_{2}$}
\end{picture}
} % End \scriptsize
} % End \parbox
\end{center}
\caption{Model of the new initial action domain description.}
\label{ExBigModelRevision}
\end{figure}

	Looking at model $\model$ in Figure~\ref{ExBigModelRevision} we can see that, for example, the agent does not know that she loses her token every time she buys coffee at the machine. This is a new law that she should incorporate to her knowledge base at some stage of her action theory evolution.

\myskip

	Contrary to contraction, where we want the negation of some law to become {\em satisfiable}, in revision we want to make a new law {\em valid}. This means that one has to eliminate all cases satisfying its negation. This depicts the duality between revision and contraction: whereas in the latter one invalidates a formula by making its negation satisfiable, in the former one makes a formula valid by forcing its negation to be unsatisfiable prior to adding the new law to the theory.

\myskip

	The idea behind our semantics is as follows: we initially have a set of models $\ModelSet$ in which a given formula $\modfml$ is (potentially) not valid, \ie, $\modfml$ is (possibly) not true in every model in $\ModelSet$. In the result we want to have only models of $\modfml$. Adding $\modfml$-models to $\ModelSet$ is of no help. Moreover, adding models makes us lose laws: the corresponding resulting theory would be more liberal.
	
	One solution amounts to deleting from $\ModelSet$ those models that are not $\modfml$-models. Of course removing only some of them does not solve the problem, we must delete every such a model. By doing that, all resulting models will be models of $\modfml$. (This corresponds to {\em theory expansion}, when the resulting theory is satisfiable.) However, if $\ModelSet$ contains no model of $\modfml$, we will end up with $\emptyset$. Consequence: the resulting theory is inconsistent. (This is the main revision problem.) In this case the solution is to {\em substitute} each model $\model$ in $\ModelSet$ by its {\em nearest modification} $\RevisionModels{\model}{\modfml}$ that makes $\modfml$ true. This lets us to keep as close as possible to the original models we had. But, what if for one model in $\ModelSet$ there are several minimal (incomparable) modifications of it validating $\modfml$? In that case we will consider all of them. The result will also be a {\em list of models} $\RevisionModels{\ModelSet}{\modfml}$, all~being~models~of~$\modfml$.

\myskip

	Before defining revision of sets of models, we present what modifications of (individual) models are.

\subsection{Revising a Model by a Static Law}\label{SemRevisionStatic}

	Suppose that our coffee deliverer agent discovers that the only hot beverage that is served on the machine is coffee. In this case, she might want to revise her beliefs with the new static law $\neg\Coffee\imp\neg\Hot$: she cannot hold a hot beverage that is not a coffee.

	Considering the model depicted in Figure~\ref{ExBigModelRevision}, one sees that the formula $\neg\Coffee\et\Hot$ is satisfiable. As we do not want this, the first step is to {\em remove} all worlds in which $\neg\Coffee\et\Hot$ is true. The second step is to guarantee that all the remaining worlds satisfy the new law. Such an issue has been largely addressed in the literature on propositional belief base revision and update~\cite{Gardenfors88,Winslett88,KatsunoMendelzon92,HerzigRifi-AIJ99}. Here we can achieve that with a semantics similar to that of classical revision operators: basically one can change the set of possible valuations, by removing or adding worlds.

	In our example, removing the possible worlds $\{\AToken,\neg\ACoffee,\AHot\}$ and $\{\neg\AToken,\neg\ACoffee,\AHot\}$ would do the job (there is no need to add new valuations since the new incoming law is satisfied in at least one world of the resulting model).

	The delicate point in removing worlds is that this may have as consequence the loss of some executability laws: in the example, if there were some arrow pointing from some world $w$ to say $\{\neg\AToken,\neg\ACoffee,\AHot\}$, then removing the latter from the model would make the action under concern no longer executable in $w$, if it was the only arrow labeled by that action leaving it. From a semantic point of view, this is intuitive: if the state of the world to which we could move is no longer possible, then we do not have a transition to that state anymore. Hence, if that transition was the only one we had, it is natural~to~lose~it.
	
	Similarly, one could ask what to do with the accessibility relation if new worlds are added, \ie, when expansion is not possible. Following the discussion in Section~\ref{ModelContractionStat}, we here prefer not to systematically add new arrows to the accessibility relation, and postpone correction of executability laws, if needed. This approach is debatable, but with the information we have at hand, this is the safest way of changing static~laws.

\myskip
	
	The semantics for revision of one model by a static law is as follows:

\begin{definition}
Let $\model=\TransStruct{}$. $\model'=\TransStructP{}\in\RevisionModels{\model}{\fml}$ if and 
only if:

\hspace{0.4cm}\parbox{7cm}{
\begin{itemize}
\item $\Worlds'=(\Worlds\setminus\valuations{\neg\fml})\cup\valuations{\fml}$
\item $\AccRel'\subseteq\AccRel$
\end{itemize}
} % end \parbox
\end{definition}

	Clearly $\Mvalid{\model'}\fml$ for each $\model'\in\RevisionModels{\model}{\fml}$. The minimal models resulting from revising a model $\model $ by $\fml$ are those closest to $\model$ \wrt\ $\submodel{\model}$:

\begin{definition}
Let $\model$ be a model and $\fml$ a static law. $\revise{\model}{\fml}=\bigcup\min\{\RevisionModels{\model}{\fml},\submodel{\model}\}$.
\end{definition}

	In the example of model $\model$ in Figure~\ref{ExBigModelRevision}, $\revise{\model}{\neg\Coffee\imp\neg\Hot}$ is the singleton $\{\model'\}$, where $\model'$ is as shown in Figure~\ref{RevisionModelStat}.

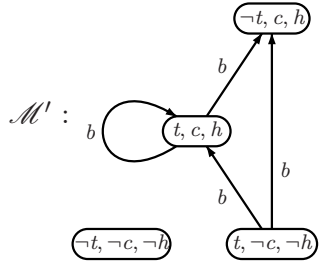
\begin{figure}[h]
\begin{center}
\parbox{1.5cm}{
$\model'$ :
}\parbox{4cm}{
\scriptsize{
\begin{picture}(35,35)(0,0) % (0,0) Ž embaixo ˆ esquerda, (50,50) Ž acima ˆ direita
    \gasset{Nw=9,Nh=4,linewidth=0.3}
    \thinlines
    \node(A1)(10,15){$\AToken,\ACoffee,\AHot$}
    \node[Nw=10](A2)(20,30){$\neg\AToken,\ACoffee,\AHot$}
    %\node[Nw=10](A3)(30,15){$\AToken,\neg\ACoffee,\AHot$}
    \node[Nw=13](A4)(0,0){$\neg\AToken,\neg\ACoffee,\neg\AHot$}
    %\node[Nw=12](A5)(40,0){$\neg\AToken,\neg\ACoffee,\AHot$}
    \node[Nw=12](A6)(20,0){$\AToken,\neg\ACoffee,\neg\AHot$}
    \drawedge[ELside=l](A1,A2){$\Abuy$}
    %\drawedge[ELside=r](A3,A2){$\Abuy$}
    \drawedge[ELside=r,ELpos=33](A6,A2){$\Abuy$}
    \drawloop[ELside=l,loopangle=180](A1){$\Abuy$}
    \drawedge[ELside=l](A6,A1){$\Abuy$}
    %\drawedge[ELside=r,ELpos=40](A3,A1){$\Abuy$}
%    \drawedge[curvedepth=-3,linewidth=0.4,ELside=r](A1,A3){$\act_{1}$}
%    \drawedge[curvedepth=3,linewidth=0.4,ELside=l](A1,A3){$\act_{2}$}
%    \drawedge[linewidth=0.4,ELside=l](A2,A3){$\act_{1}$}
%    \drawloop[linewidth=0.4,loopangle=90](A2){$\act_{1}$}
%    \drawloop[linewidth=0.4,loopangle=270](A3){$\act_{2}$}
\end{picture}
} % End \scriptsize
} % End \parbox
\end{center}
\caption{Model resulting from revising the model $\model$ in Figure~\ref{ExBigModelRevision} with $\neg\Coffee\imp\neg\Hot$.}
\label{RevisionModelStat}
\end{figure}

\subsection{Revising a Model by an Effect Law}

	Let's suppose now that our agent eventually discovers that after buying coffee she does not keep her token. This means that her theory should now be revised by the new effect law $\Token\imp\nec\buy\neg\Token$. Looking at model $\model$ in Figure~\ref{ExBigModelRevision}, this amounts to guaranteeing that the formula $\Token\et\poss\buy\Token$ is satisfiable in none of its worlds. To do that, we have to look at all the worlds satisfying this formula (if any) and
\begin{itemize}
\item either make $\Token$ false in each of these worlds, 
\item or make $\poss\buy\Token$ false in all of them.
\end{itemize}

	If we chose the first option, we will essentially flip the truth value of literal \Token\ in the respective worlds, which changes the set of valuations of the model. If we chose the latter, we will basically remove \buy-arrows leading to \Token-worlds. In that case, a change in the accessibility relation will be made.

	In our example, we have that the possible worlds $\{\Token,\Coffee,\Hot\}$, $\{\Token,\neg\Coffee,\Hot\}$ and $\{\Token,\neg\Coffee,\neg\Hot\}$ satisfy $\Token\et\poss\buy\Token$ and all they have to change.
	
	Flipping $\Token$ in all these worlds to $\neg\Token$ would do the job, but would also have as consequence the introduction of a new static law: $\neg\Token$ would now be valid, \ie, the agent never has a token.
	
	Here we think that changing action laws should not have as side effect a change in the static laws. Given their special status, these should change only if explicitly required (see above). In this case, each world satisfying $\Token\et\poss\buy\Token$ has to be changed so that $\poss\buy\Token$ is no longer true in it. In our example, we should remove the arrows $(\{\Token,\Coffee,\Hot\},\{\Token,\Coffee,\Hot\})$, $(\{\Token,\neg\Coffee,\Hot\},\{\Token,\Coffee,\Hot\})$ and $(\{\Token,\neg\Coffee,\neg\Hot\},\{\Token,\Coffee,\Hot\})$.

\myskip

	The semantics of one model revision for the case of a new effect law is:
	
\begin{definition}
Let $\model=\TransStruct{}$. $\model'=\TransStructP{}\in\RevisionModels{\model}{\antec\imp\nec
\act\conseq}$ if and only if:

\hspace{0.4cm}\parbox{8cm}{
\begin{itemize}
\item $\Worlds'=\Worlds$
\item $\AccRel'\subseteq\AccRel$
\item If $(w,w')\in\AccRel\setminus\AccRel'$, then $\wMvalid{w}{\model}\antec$
\item $\Mvalid{\model'}\antec\imp\nec\act\conseq$
\end{itemize}
} % end \parbox
\end{definition}

	The minimal models resulting from the revision of a model $\model$ by a new effect law are those that are closest to $\model$ \wrt\ $\submodel{\model}$:

\begin{definition}
Let $\model$ be a model and $\antec\imp\nec\act\conseq$ an effect law. $\revise{\model}{\antec\imp\nec\act\conseq}=\bigcup\min\{\RevisionModels{\model}{\antec\imp\nec\act\conseq},\submodel{\model}\}$.
\end{definition}

	Taking again $\model$ as in Figure~\ref{ExBigModelRevision}, $\revise{\model}{\Token\imp\nec\buy\neg\Token}$ will be the singleton $\{\model'\}$ (Figure~\ref{RevisionModelEff}).

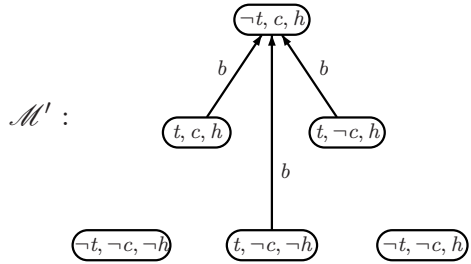
\begin{figure}[h]
\begin{center}
\parbox{1.5cm}{
$\model'$ :
}\parbox{4cm}{
\scriptsize{
\begin{picture}(35,35)(0,0) % (0,0) Ž embaixo ˆ esquerda, (50,50) Ž acima ˆ direita
    \gasset{Nw=9,Nh=4,linewidth=0.3}
    \thinlines
    \node(A1)(10,15){$\AToken,\ACoffee,\AHot$}
    \node[Nw=10](A2)(20,30){$\neg\AToken,\ACoffee,\AHot$}
    \node[Nw=10](A3)(30,15){$\AToken,\neg\ACoffee,\AHot$}
    \node[Nw=13](A4)(0,0){$\neg\AToken,\neg\ACoffee,\neg\AHot$}
    \node[Nw=12](A5)(40,0){$\neg\AToken,\neg\ACoffee,\AHot$}
    \node[Nw=12](A6)(20,0){$\AToken,\neg\ACoffee,\neg\AHot$}
    \drawedge[ELside=l](A1,A2){$\Abuy$}
    \drawedge[ELside=r](A3,A2){$\Abuy$}
    \drawedge[ELside=r,ELpos=33](A6,A2){$\Abuy$}
    %\drawloop[ELside=l,loopangle=180](A1){$\Abuy$}
    %\drawedge[ELside=l](A6,A1){$\Abuy$}
    %\drawedge[ELside=r,ELpos=40](A3,A1){$\Abuy$}
%    \drawedge[curvedepth=-3,linewidth=0.4,ELside=r](A1,A3){$\act_{1}$}
%    \drawedge[curvedepth=3,linewidth=0.4,ELside=l](A1,A3){$\act_{2}$}
%    \drawedge[linewidth=0.4,ELside=l](A2,A3){$\act_{1}$}
%    \drawloop[linewidth=0.4,loopangle=90](A2){$\act_{1}$}
%    \drawloop[linewidth=0.4,loopangle=270](A3){$\act_{2}$}
\end{picture}
} % End \scriptsize
} % End \parbox
\end{center}
\caption{Model resulting from revising the model $\model$ in Figure~\ref{ExBigModelRevision} with the new effect law $\Token\imp\nec\buy\neg\Token$.}
\label{RevisionModelEff}
\end{figure}

\subsection{Revising a Model by an Executability Law}\label{SemRevisionExe}

	Let us now suppose that in some stage it has been decided to grant free coffee to everybody. Faced with this information, the agent will now revise her laws to reflect the fact that \buy\ can also be executed in $\neg\Token$-contexts: $\neg\Token\imp\poss\buy\top$ is a new executability law (and hence we will have $\poss\buy\top$ in all new models of the agent's beliefs).

	Considering again the model in Figure~\ref{ExBigModelRevision}, we observe that $\neg(\neg\Token\imp\poss\buy\top)$ is satisfiable in $\model$. Hence we must throw $\neg\Token\et\nec\buy\bot$ away to ensure the new formula becomes true.

	To remove $\neg\Token\et\nec\buy\bot$ we have to look at all worlds satisfying it and modify $\model$ so that they no longer satisfy that formula. Given worlds $\{\neg\Token,\neg\Coffee,\neg\Hot\}$ and $\{\neg\Token,\neg\Coffee,\Hot\}$, we have two options: change the interpretation of $\Token$ or add new arrows leaving these worlds. A question that arises is `what choice is more drastic: change a world or an arrow'? Again, here we think that changing the world's content (the valuation) is more drastic, as the existence of such a world was foreseen by some static law and is hence assumed to be as it is, unless we have enough information supporting the contrary, in which case we explicitly change the static laws (see above). Thus we shall add a new \buy-arrow from each of $\{\neg\Token,\neg\Coffee,\neg\Hot\}$ and $\{\neg\Token,\neg\Coffee,\Hot\}$.

	Having agreed on that, the issue now is: which worlds should the new arrows point to? Recalling the reasoning developed in Section~\ref{ModelContractionEff}, in order to comply with minimal change, the new arrows shall point to worlds that are relevant targets of each of the $\neg\Token$-worlds in question. In our example, $\{\neg\Token,\Coffee,\Hot\}$ is the only relevant target world here: the two other $\neg\Token$-worlds violate the effect \Coffee\ of \buy, while the three \Token-worlds would make us violate the frame axiom $\neg\Token\imp\nec\buy\neg\Token$.
	
	The semantics for one model revision by a new executability law is as follows:

\begin{definition}
Let $\model=\TransStruct{}$. $\model'=\TransStructP{}\in\RevisionModels{\model}{\antec\imp\poss\act\top}$ if and only if:

\hspace{0.4cm}\parbox{10cm}{
\begin{itemize}
\item $\Worlds'=\Worlds$
\item $\AccRel\subseteq\AccRel'$
\item If $(w,w')\in\AccRel'\setminus\AccRel$, then $w'\in\RelTarget{w,\antec\imp\nec\act\bot,\model,\ModelSet}$
\item $\Mvalid{\model'}\antec\imp\poss\act\top$
\end{itemize}
} % end \parbox
\end{definition}

	The minimal models resulting from revising a model $\model $ by a new executability law are those closest to $\model$ \wrt\ $\submodel{\model}$:

\begin{definition}
Let $\model$ be a model and $\antec\imp\poss\act\top$ be an executability law. $\revise{\model}{\antec\imp\poss\act\top}=\bigcup\min\{\RevisionModels{\model}{\antec\imp\poss\act\top},\submodel{\model}\}$.
\end{definition}

	In our running example, $\revise{\model}{\neg\Token\imp\poss\buy\top}$ is the singleton $\{\model'\}$, where $\model'$ is as shown in Figure~\ref{RevisionModelExe}.

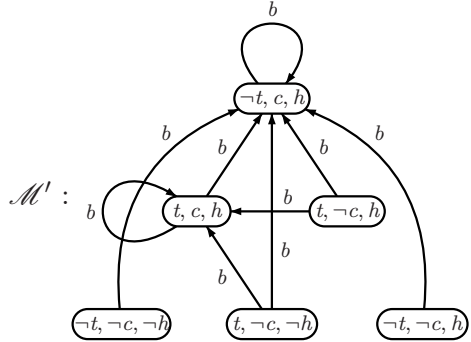
\begin{figure}[h]
\begin{center}
\parbox{1.5cm}{
$\model'$ :
}\parbox{4cm}{
\scriptsize{
\begin{picture}(35,35)(0,0) % (0,0) Ž embaixo ˆ esquerda, (50,50) Ž acima ˆ direita
    \gasset{Nw=9,Nh=4,linewidth=0.3}
    \thinlines
    \node(A1)(10,15){$\AToken,\ACoffee,\AHot$}
    \node[Nw=10](A2)(20,30){$\neg\AToken,\ACoffee,\AHot$}
    \node[Nw=10](A3)(30,15){$\AToken,\neg\ACoffee,\AHot$}
    \node[Nw=13](A4)(0,0){$\neg\AToken,\neg\ACoffee,\neg\AHot$}
    \node[Nw=12](A5)(40,0){$\neg\AToken,\neg\ACoffee,\AHot$}
    \node[Nw=12](A6)(20,0){$\AToken,\neg\ACoffee,\neg\AHot$}
    \drawedge[ELside=l](A1,A2){$\Abuy$}
    \drawedge[ELside=r](A3,A2){$\Abuy$}
    \drawedge[ELside=r,ELpos=33](A6,A2){$\Abuy$}
    \drawloop[ELside=l,loopangle=180](A1){$\Abuy$}
    \drawedge[ELside=l](A6,A1){$\Abuy$}
    \drawedge[ELside=r,ELpos=40](A3,A1){$\Abuy$}
    \drawedge[ELside=l,curvedepth=8,ELpos=66](A4,A2){$\Abuy$}
    \drawedge[ELside=r,curvedepth=-8,ELpos=66](A5,A2){$\Abuy$}
    \drawloop[ELside=l,loopangle=90](A2){$\Abuy$}
%    \drawedge[curvedepth=-3,linewidth=0.4,ELside=r](A1,A3){$\act_{1}$}
%    \drawedge[curvedepth=3,linewidth=0.4,ELside=l](A1,A3){$\act_{2}$}
%    \drawedge[linewidth=0.4,ELside=l](A2,A3){$\act_{1}$}
%    \drawloop[linewidth=0.4,loopangle=90](A2){$\act_{1}$}
%    \drawloop[linewidth=0.4,loopangle=270](A3){$\act_{2}$}
\end{picture}
} % End \scriptsize
} % End \parbox
\end{center}
\caption{The result of revising model $\model$ in Figure~\ref{ExBigModelRevision} by the new executability law $\neg\Token\imp\poss\buy\top$.}
\label{RevisionModelExe}
\end{figure}

\subsection{Revising Sets of Models}

	Up until now we have seen what the revision of single models means. This is needed when expansion by the new law is not possible due to inconsistency. We here give a unified definition of revision of a set of models $\ModelSet$ by a new law $\modfml$:

\begin{definition}\label{ModelSetRevision}
Let $\ModelSet$ be a set of models and $\modfml$ a law. Then
\[
\RevisionModels{\ModelSet}{\modfml}=\left\{\begin{array}{l}
	\ModelSet\setminus\{\model : \notMvalid{\model}\modfml\}, \text{ if there is }\model\in\ModelSet\ \suchthat\ \Mvalid{\model}\modfml \\[0.2cm]
	\bigcup_{\model\in\ModelSet}\revise{\model}{\modfml}, \text{ otherwise}
								\end{array}
							\right.
\]
\end{definition}
Observe that Definition~\ref{ModelSetRevision} comprises both {\em expansion} and {\em revision}: in the first one, simple addition of the new law gives a satisfiable theory; in the latter a deeper change is needed to get rid of inconsistency.

%%%%%%%%%%%%%%%%%%%%%%%%%%%%%%%%%%%%%%%%%%%%%%%%%%%
\section{Related Work}\label{RelatedWork}
%%%%%%%%%%%%%%%%%%%%%%%%%%%%%%%%%%%%%%%%%%%%%%%%%%%

	To the best of our knowledge, the first work on updating an action domain description is that by Li and Pereira~\cite{LiPereira96} in a narrative-based action description language~\cite{GelfondLifschitz93}. Contrary to us, however, they mainly investigate the problem of updating the narrative with new observed {\em facts} and (possibly) with occurrences of actions that explain those facts. This amounts to updating a given state/configuration of the world (in our terms, what is true in a possible world) and focusing on the models of the narrative in which some actions took place (in our terms, the models of the action theory with a particular sequence of action executions). Clearly the models of the action laws remain the same.

\myskip

	Baral and Lobo~\cite{BaralLobo-IJCAI97} introduce extensions of action languages that allow for some causal laws to be stated as defeasible. Their work is similar to ours in that they also allow for weakening of laws: in their setting, effect propositions can be replaced by what they call defeasible (weakened versions of) effect propositions. Our approach is different from theirs in the way executability laws are dealt with. Here executability laws are explicit and we are also able to contract them. This feature is important when the qualification problem~\cite{McCarthy77} is considered: we may always discover contexts that preclude the execution of a given action (\cf\ the Introduction).

\myskip

	Liberatore~\cite{Liberatore-JELIA2000} proposes a framework for reasoning about actions in which it is possible to express a given semantics of belief update, like Winslett's~\cite{Winslett88} and Katsuno and Mendelzon's~\cite{KatsunoMendelzon92}. This means it is the formalism, essentially an action description language, that is used to describe updates (the change of propositions from one state of the world to another) by expressing them as laws in the action theory.
	
	The main difference between Liberatore's work and Li and Pereira's is that, despite not being concerned, at least a priori, with changing action laws, Liberatore's framework allows for abductively introducing in the action theory new effect propositions (effect laws, in our terms) that consistently explain the occurrence of an event.
	
\myskip

	The work by Eiter~\etal~\cite{EiterEtAl-IJCAI05,EiterEtAl-ECAI06} is similar to ours in that they also propose a framework that is oriented to updating action laws. They mainly investigate the case where \eg\ a new effect law is added to the description (and then has to be true in all models of the modified theory). This problem is the dual of contraction and is then closer to our definition of revision (\cf\ Section~\ref{SemanticsRevision}).

	In Eiter \etal's framework, action theories are described in a variant of a narrative-based action description language. Like in the present work, the semantics is also in terms of transition systems: directed graphs having arrows (action occurrences) linking nodes (configurations of the world). Contrary to us, however, the minimality condition on the outcome of the update is in terms of inclusion of sets of laws, which means the approach is more syntax oriented.
	
	In their setting, during an update an action theory $\Theory{}$ is seen as composed of two pieces, $\Theory{u}$ and $\Theory{m}$, where $\Theory{u}$ stands for the part of $\Theory{}$ that is not supposed to change and $\Theory{m}$ contains the laws that may be modified. In our terms, when contracting a static law we would have $\Theory{m}=\STAT{}{}\cup\EXE{}{\act}$, when contracting an executability $\Theory{m}=\EXE{}{\act}$, and when contracting effects laws $\Theory{m}=\EFF{-}{\act}$. The difference here is that in our approach it is always clear what laws should not change in a given type of contraction, and $\Theory{u}$ and $\Theory{m}$ do not need to be explicitly specified prior to the update.

	Their approach and ours can both be described as {\em constraint-based} update, in that the theory change is carried out relative to some restrictions (a set of laws that we want to hold in the result). In our framework, for example, all changes in the action laws are relative to the set of static laws $\STAT{}{}$ (and that is why we concentrate on models of $\Theory{}$ having $\valuations{\STAT{}{}}$ as worlds). When changing a law, we want to keep the same set of states. The difference \wrt\ Eiter~\etal's approach is that there it is also possible to update a theory relatively to \eg\ executability laws: when expanding $\Theory{}$ with a new effect law, one may want to constrain the change so that the action under concern is guaranteed to be executable in the result.\footnote{We could simulate that in our approach with two successive modifications of $\Theory{}$: first adding the effect law and then an executability law (\cf\ Section~\ref{SemanticsRevision}).} As shown in the referred work, this may require the withdrawal of some static law. Hence, in Eiter \etal's framework, static laws do not have the same status as in ours.

\myskip

	Herzig~\etal~\cite{HerzigEtAl-ECAI06} define a method for action theory contraction that, despite the similarity with the current work and the common underlying motivations, is more limited than the present constructions.
	
	First, with the referred approach we do not get minimal change. For example, in the referred work the operator for contracting executability laws is such that in the resulting theory the modified set of executabilities is given by
\[
\EXE{-}{\act}=\{(\antec_{i}\et\neg\antec)\imp\poss\act\top : \antec_{i}\imp\poss\act\top\in\EXE{}{\act}\}
\]
which, according to its semantics, gives theories among whose models are those resulting from removing arrows from {\em all} $\antec$-worlds. A similar comment can be made \wrt\ contraction of effect laws.

	Second, Herzig~\etal's contraction method does not satisfy most of the postulates for theory change that we have addressed in Section~\ref{Postulates}. Besides not satisfying the monotonicity postulate, it does not satisfy the preservation one. To witness, suppose we have a language with only one atom $\prp$, and the model $\model$ depicted in Figure~\ref{NonPreservationECAI06}.

\begin{figure}[h]
\begin{center}
\parbox{1.5cm}{
$\model$ :
}\parbox{4.2cm}{
\scriptsize{
\begin{picture}(6,20)(0,0) % (0,0) Ž embaixo ˆ esquerda, (50,50) Ž acima ˆ direita
    \gasset{Nw=5,Nh=5,linewidth=0.3}
    \thinlines
    \node(A1)(0,6){$\prp$}
    \node(A2)(20,6){$\neg\prp$}
    \drawedge[curvedepth=3](A1,A2){$\act$}
    \drawedge[curvedepth=3](A2,A1){$\act$}
    \drawloop(A2){$\act$}
\end{picture}
} % End \scriptsize
} % End \parbox
\parbox{1.7cm}{
$\model'$ :
}\parbox{4.2cm}{
\scriptsize{
\begin{picture}(6,20)(0,0) % (0,0) Ž embaixo ˆ esquerda, (50,50) Ž acima ˆ direita
    \gasset{Nw=5,Nh=5,linewidth=0.3}
    \thinlines
    \node(A1)(0,6){$\prp$}
    \node(A2)(20,6){$\neg\prp$}
    \drawedge[curvedepth=3](A1,A2){$\act$}
    \drawedge[curvedepth=3](A2,A1){$\act$}
    \drawloop(A2){$\act$}
    \drawloop(A1){$\act$}
\end{picture}
} % End \scriptsize
} % End \parbox
\end{center}
\caption{Counter-example to preservation in the method of contraction by Herzig \etal~\cite{HerzigEtAl-ECAI06}.}
\label{NonPreservationECAI06}
\end{figure}
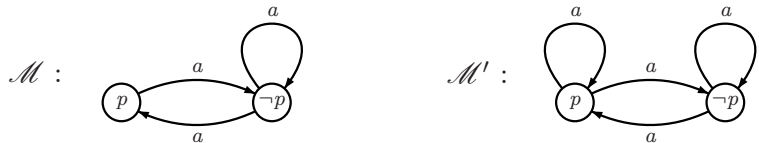

	Then $\Mvalid{\model}\prp\imp\nec\act\neg\prp$ and $\notMvalid{\model}\nec\act\neg
\prp$. Now the contraction operator defined there is such that when removing $\nec\act\neg\prp$ from $\model$ yields the model $\model'$ in Figure~\ref{NonPreservationECAI06} such that $\ract'=\Worlds\times\Worlds$. Then $\notMvalid{\model'}\prp\imp\nec\act\neg\prp$, \ie, the effect law $\prp\imp\nec\act\neg\prp$ is not preserved.

%%%%%%%%%%%%%%%%%%%%%%%%%%%%%%%%%%%%%%%%%%%%%%%%%%%
\section{Comments}\label{Comments}
%%%%%%%%%%%%%%%%%%%%%%%%%%%%%%%%%%%%%%%%%%%%%%%%%%%

	In this section we make some comments about possible modifications or improvements in our constructions so far.

\subsection{Other Distance Notions}

	Here we have used a model distance based on symmetric differences between sets. This distance is quite close to Winslett's~\cite{Winslett88} notion of closeness between interpretations in the Possible Models Approach (PMA). Instead of it, however, we could have considered other distance notions as well, like \eg\ Dalal's~\cite{Dalal88} distance, Hamming distance~\cite{Hamming50}, or weighted distance. Due to space limitations, we do not develop a through comparison among all these distances here. We nevertheless do show that with a cardinality-based distance, for example, we may not always get the intended result.
% (the interested reader can find that in~\cite{})

\myskip

	Let $\card{X}$ denote the number of elements in set $X$. Then suppose our closeness between \PDL-models was defined as follows:

\begin{definition}[Cardinality-based closeness between \PDL-Models]
Let $\model=\TransStruct{}$ be a model. Then $\model'=\TransStructP{}$ is {\em at least as close to} $\model$ as $\model''=\TransStructS{}$, noted $\model'\leq_{\model}\model''$, if~and~only~if

\hspace{0.4cm}\parbox{14cm}{
\begin{itemize}
\item either $\card{\Worlds\symdiff\Worlds'}\leq\card{\Worlds\symdiff\Worlds''}$
\item or $\card{\Worlds\symdiff\Worlds'}=\card{\Worlds\symdiff\Worlds''}$ and $\card{\AccRel\symdiff
\AccRel'}\leq\card{\AccRel\symdiff\AccRel''}$
\end{itemize}
} % end \parbox
\end{definition}

	Such a notion of distance is closely related to Dalal's~\cite{Dalal88} closeness.

\myskip

	Since when contracting a static law $\fml$ from a model $\model$ we usually add one new possible world, it is easy to see that with this cardinality-based distance we get the same result in $\erasure{\model}{\fml}$ as with the distance from Definition~\ref{ClosenessPDLModels}. 
	
	When it comes to contraction of action laws, and then changing the accessibility relations, however, this cardinality-based distance does not seem to fit with the intuitions. To witness, consider the model $\model$ in Figure~\ref{CardinalityContraction1}, which satisfies the executability law $\prp_{1}\imp\poss\act\top$.

\begin{figure}[h]
\begin{center}
\parbox{1.5cm}{
$\model$ :
}\parbox{4.2cm}{
\scriptsize{
\begin{picture}(40,42)(0,0) % (0,0) Ž embaixo ˆ esquerda, (50,50) Ž acima ˆ direita
    \gasset{Nw=9,Nh=4,linewidth=0.3}
    \thinlines
    \node[Nw=10](A1)(0,30){$\prp_{1},\prp_{2}$}
    \node[Nw=11](A2)(30,30){$\prp_{1},\neg\prp_{2}$}
    \node[Nw=12](A3)(15,5){$\neg\prp_{1},\neg\prp_{2}$}
    \drawedge[curvedepth=3,ELside=l](A1,A2){$\act$}
    \drawedge[curvedepth=3,ELside=l](A2,A1){$\act$}
    \drawedge[ELside=l](A2,A3){$\act$}
\end{picture}
} % End \scriptsize
} % End \parbox
\end{center}
\caption{A model $\model$ satisfying $\prp_{1}\imp\poss\act\top$.}
\label{CardinalityContraction1}
\end{figure}

	Then, $\ErasureModels{\model}{\prp_{1}\imp\poss\act\top}=\{\model',\model''\}$, where $\model'$ and $\model''$ are as depicted in Figure~\ref{CardinalityContraction2}.

\begin{figure}[h]
\begin{center}
\parbox{1.5cm}{
$\model'$ :
}\parbox{4.2cm}{
\scriptsize{
\begin{picture}(40,42)(0,0) % (0,0) Ž embaixo ˆ esquerda, (50,50) Ž acima ˆ direita
    \gasset{Nw=9,Nh=4,linewidth=0.3}
    \thinlines
    \node[Nw=10](A1)(0,30){$\prp_{1},\prp_{2}$}
    \node[Nw=11](A2)(30,30){$\prp_{1},\neg\prp_{2}$}
    \node[Nw=12](A3)(15,5){$\neg\prp_{1},\neg\prp_{2}$}
    %\drawedge[curvedepth=3,ELside=l](A1,A2){$\act$}
    \drawedge[curvedepth=3,ELside=l](A2,A1){$\act$}
    \drawedge[ELside=l](A2,A3){$\act$}
\end{picture}
} % End \scriptsize
} % End \parbox
\parbox{1.5cm}{
$\model''$ :
}\parbox{4.2cm}{
\scriptsize{
\begin{picture}(40,42)(0,0) % (0,0) Ž embaixo ˆ esquerda, (50,50) Ž acima ˆ direita
    \gasset{Nw=9,Nh=4,linewidth=0.3}
    \thinlines
    \node[Nw=10](A1)(0,30){$\prp_{1},\prp_{2}$}
    \node[Nw=11](A2)(30,30){$\prp_{1},\neg\prp_{2}$}
    \node[Nw=12](A3)(15,5){$\neg\prp_{1},\neg\prp_{2}$}
    \drawedge[curvedepth=3,ELside=l](A1,A2){$\act$}
    %\drawedge[curvedepth=3,ELside=l](A2,A1){$\act$}
    %\drawedge[ELside=l](A2,A3){$\act$}
\end{picture}
} % End \scriptsize
} % End \parbox
\end{center}
\caption{Models resulting from contracting $\prp_{1}\imp\poss\act\top$ in the model $\model$ of Figure~\ref{CardinalityContraction1}.}
\label{CardinalityContraction2}
\end{figure}

	Note that $\model''$ is an intended contracted model. However, with the cardinality-based distance above we will get $\ErasureModels{\{\model\}}{\prp_{1}\imp\poss\act\top}=\{\{\model,\model'\}\}$. We do not have $\{\model,\model''\}$ in the result since $\model'\leq_{\model}\model''$: in $\model'$ only {\em one} arrow has been removed, while in $\model''$ {\em two}.

%\myskip
%
%	For more details on these and other notions of distance, the reader is referred to~\cite{}.

\subsection{Inducing Executability}

	Regarding the semantics for contracting static laws, we could try to go further and at least make a guess about what executability laws we should preserve. Before doing that, we need a definition.

\begin{definition}[Closeness between Valuations]
Let $\val$ be a propositional valuation. The valuation $\val'$ is {\em as close to} $\val$ as $\val''$, 
noted $\val'\submodel{\val}\val''$, if and only if $\val\symdiff\val'\subseteq\val\symdiff\val''$.
\end{definition}
So the distance between valuations $\val_{1}$ and $\val_{2}$ is the set of literals on which they differ: $\val_{1}\symdiff\val_{2}=\{\lit : \val_{1}\sat\lit \text{ and } \val_{2}\notsat\lit\}\cup\{\lit :\val_{2}\sat\lit \text{ and } \val_{1}\notsat\lit\}$.	

\myskip

	Our argument now is as follows: when adding a new world, we can look at its contents and see what happens in worlds that are similar to it (by similar here we mean the possible worlds that are closest to it). A priori and intuitively we can expect that if we put a new arrow leaving the new world, it will neither point to a world that is the target of no other world, nor point to a world that is not closest to it. It is reasonable to expect that in the new world a given action may have a behavior that is quite similar to that which it has in the worlds that are closest to the new one. Hence we select the worlds whose distance to the new one is minimal, look at where the arrows leaving them point to, and then point the new arrow there. With a similar argument, we can decide which arrows targeting the new world add to the model. The definition below formalizes this.

%*** R' deve ser subconjunto da reuniao dos outros R, e a mesma coisa para R-1 ***
\begin{definition}\label{ReDefinitionErasureStat}
Let $\model=\TransStruct{}$. $\model'=\TransStructP{}\in\ErasureModels{\model}{\fml}$ if and 
only if

\hspace{0.4cm}\parbox{11cm}{
\begin{itemize}
\item $\Worlds\subseteq\Worlds'$
\item $\AccRel\subseteq\AccRel'$
\item If $w'\in\Worlds'\setminus\Worlds$, then $\AccRel'(w')\subseteq\AccRel(w)$ and $
\AccRel'^{-1}
(w')\subseteq\AccRel^{-1}(w)$, where $w\in\bigcup\min\{\Worlds,\submodel{w'}\}$
\item There is $w\in\Worlds'$\ \suchthat\ $\notwMvalid{w}{\model'}\fml$
\end{itemize}
} % end \parbox
\end{definition}

	With this new definition, what we do is suppose that some of the known laws for the other worlds can still be true in the new state, by analogy to the other possible states. In a similar way, when facing a new situation, we may wonder how we got there. Again, by analogy with known states, we could expect that we get to the new state coming from a state that usually produces something similar to what we have now in front of us. In this case we have a kind of abduction-like reasoning that may of course be wrong but that is not illegal.

	Although intuitive, at least in its motivation, adopting Definition~\ref{ReDefinitionErasureStat} could have some undesirable side effects. For example, if in the semantics we decide to add new arrows pointing from and to the new added world, then our corresponding operator may not satisfy the monotonicity postulate. To see,  let
\[
\Theory{}=\left\{\begin{array}{c}
	\prp_{1}\otimes\prp_{2}, \prp_{1}\imp\nec\act\prp_{2}, \\
	\prp_{1}\imp\poss\act\top, \prp_{2}\imp\nec\act\bot
	\end{array}
\right\}
\]

\myskip

The only model of $\Theory{}$ is $\model=\TransStruct{}$ such that $\Worlds=\{\{\prp_{1},\neg\prp_{2}\},\{\neg\prp_{1},\prp_{2}\}\}$ and $\AccRel=\{(\{\prp_{1},\neg\prp_{2}\},\{\neg\prp_{1},\prp_{2}\})\}$ (Figure~\ref{NonMonotonyErasure}).

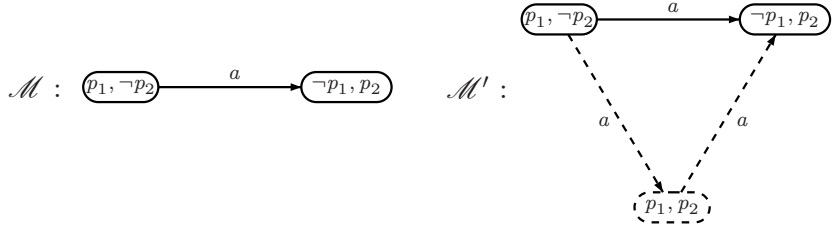
\begin{figure}[h]
\begin{center}
\parbox{1.5cm}{
$\model$ :
}\parbox{4.2cm}{
\scriptsize{
\begin{picture}(40,42)(0,0) % (0,0) Ž embaixo ˆ esquerda, (50,50) Ž acima ˆ direita
    \gasset{Nw=9,Nh=4,linewidth=0.3}
    \thinlines
    \node[Nw=10](A1)(0,21){$\prp_{1},\neg\prp_{2}$}
    \node[Nw=12](A2)(30,21){$\neg\prp_{1},\prp_{2}$}
    \drawedge(A1,A2){$\act$}
\end{picture}
} % End \scriptsize
} % End \parbox
\parbox{1.5cm}{
$\model'$ :
}\parbox{4.2cm}{
\scriptsize{
\begin{picture}(40,42)(0,0) % (0,0) Ž embaixo ˆ esquerda, (50,50) Ž acima ˆ direita
    \gasset{Nw=9,Nh=4,linewidth=0.3}
    \thinlines
    \node[Nw=10](A1)(0,30){$\prp_{1},\neg\prp_{2}$}
    \node[Nw=12](A2)(30,30){$\neg\prp_{1},\prp_{2}$}
    \node[Nw=10,dash={1}0](A3)(15,5){$\prp_{1},\prp_{2}$}
    \drawedge(A1,A2){$\act$}
    \drawedge[ELside=r,dash={1}0](A1,A3){$\act$}
    \drawedge[ELside=r,dash={1}0](A3,A2){$\act$}
\end{picture}
} % End \scriptsize
} % End \parbox
\end{center}
\caption{Counter-example to monotonicity when adding arrows to/from new worlds in the semantics of static law contraction. $\model$ denotes the original model of $\Theory{}$, while $\model'$ shows the new added world and the candidate arrows to add to $\ract$.}
\label{NonMonotonyErasure}
\end{figure}

	If we contract $\prp_{1}\imp\neg\prp_{2}$ from $\Theory{}$, in the semantic result we have only the model $\model'$ in Figure~\ref{NonMonotonyErasure} such that $\Mvalid{\model'}\prp_{1}\imp\poss\act\poss\act\prp_{2}$. Then, we would have $\Theory{}'\PDLvalid\prp_{1}\imp\poss\act\poss\act\prp_{2}$, and then $\Theory{}\notPDLvalid\Theory{}'$.

\myskip

	The very issue with such a semantic characterization however would be how to capture it at the syntactic level: what syntax operator for change should we have in order to capture this closeness between possible worlds? More importantly, since we may be wrong about a guess regarding the executability or an effect of a given action, how can it be rolled back in the new theory? These are open questions that we leave for further investigation.

%%%%%%%%%%%%%%%%%%%%%%%%%%%%%%%%%%%%%%%%%%%%%%%%%%%
\section{Concluding Remarks}\label{Conclusion}
%%%%%%%%%%%%%%%%%%%%%%%%%%%%%%%%%%%%%%%%%%%%%%%%%%%

	In this work we have given a semantics for action theory change in terms of distances between models that captures the notion of minimal change. We have given algorithms to contract a formula from a theory that terminate and are correct \wrt\ the semantics~(Corollary~\ref{CorollaryCorrectness}). We have shown the importance that modularity has in this result and in others. %With the results in Sections~\ref{Correctness} and~\ref{Postulates}, and with a notion of minimal change, we get a better solution than that proposed in~\cite{HerzigEtAl-ECAI06}.

\myskip

	Under modularity, our operators satisfy all the postulates for contraction. This supports the thesis that our modularity notion is fruitful.

	By forcing formulas to be explicitly stated in their respective modules (and thus possibly making them inferable in independently different ways), modularity intuitively could be seen to diminish elaboration tolerance~\cite{McCarthy98}. For instance, when contracting a Boolean formula $\fml$ in a non-modular theory, it seems reasonable to expect not to change the set of static laws $\STAT{}{}$, while the theory being modular surely forces changing such a module.
	
	It is not difficult, however, to conceive non-modular theories in which contraction of a formula $\fml$ may demand a change in $\STAT{}{}$ as well. As an example, suppose $\STAT{}{}=\{\fml_{1}\imp\fml_{2}\}$ in an action theory from whose dynamic part we (implicitly) infer $\neg\fml_{2}$. In this case, contracting $\neg\fml_{1}$ while keeping $\neg\fml_{2}$ would necessarily ask for a change in $\STAT{}{}$.
	
	We point out nevertheless that in both cases (modular and non-modular) the extra work in changing other modules stays in the mechanical level, \ie, in the algorithms that carry out the modification, and does not augment in a significant way the amount of work the knowledge engineer is expected to do. Moreover, considering the evolution of the theory, \ie, future modifications one should perform in it, modularity has to be checked/ensured only once, since it is preserved by our operators (\cf\ Lemma~\ref{PreservationModularity}).

\myskip

	While terminating, our algorithms come with a considerable computational cost: the entailment test in $\Kn$ with global axioms is known to be \pspace-complete. Although this may be acceptable (theory change can be carried out offline), the computation of $\IP{.}$ might result in exponential growth.

\myskip

	We have also extended Varzinczak's studies~\cite{Varzinczak-KR08} by defining a semantics for action theory revision based on minimal modifications of models. For the corresponding revision algorithms, the reader is referred to the work by Varzinczak~\cite{Varzinczak-NMR08}. One of our ongoing researches is on assessing our revision operators' behavior \wrt\ the AGM postulates for revision~\cite{AlchourronEtAl85}.
	
\myskip

	Another issue that drives our future research on the subject is how to contract not only laws but {\em any} \PDL-formula. As defined, the order of application of our operators matter in the final result: if we contract $\fml$ and then $\antec\imp\nec\act\conseq$ from a theory $\Theory{}$, the    result may not be the same as contracting $\antec\imp\nec\act\conseq$ first and then removing $\fml$. This problem would not appear in a more general framework in which any formula could be contracted: removing $\fml\et(\antec\imp\nec\act\conseq)$ should give the same result as $(\antec\imp\nec\act\conseq)\et\fml$.

	Definitions~\ref{DefErasureOneModelExe}, \ref{DefErasureOneModelEff} and~\ref{DefErasureOneModelStat} appear to be important for better understanding the problem of contracting general formulas: basically the set of modifications to perform in a given model in order to force it to falsify a general formula will comprise removal/addition of arrows/worlds. The definition of a general revision/contraction method will then benefit from our constructions.

\myskip

	Given the connection between multimodal logics and Description Logics~\cite{BaaderEtAl-HandBook-DL2003}, we believe that the definitions here given may also contribute to ontology evolution and debugging in DLs.

%%%%%%%%%%%%%%%%%%%%%%%%%%%%%%%%%%%%%%%%%%%%%%%%%%%
\section*{Acknowledgements}
%%%%%%%%%%%%%%%%%%%%%%%%%%%%%%%%%%%%%%%%%%%%%%%%%%%

	The author is grateful to Arina Britz and Ken Halland for proofreading an earlier version of this article. Their comments helped very much in improving the text.

%%%%%%%%%%%%%%%%%%%%%%%%%%%%%%%%%%%%%%%%%%%%%%%%%%%
%\addcontentsline{toc}{section}{References}
%\bibliography{References}
%\bibliographystyle{plain}

%%%%%%%%%%%%%%%%%%%%%%%%%%%%%%%%%%%%%%%%%%%%%%%%%%%

\appendix

%%%%%%%%%%%%%%%%%%%%%%%%%%%%%%%%%%%%%%%%%%%%%%%%%%%
\section{Proof of Theorem~\ref{TheoremCompleteness}}\label{ProofTheoremCompleteness}
%%%%%%%%%%%%%%%%%%%%%%%%%%%%%%%%%%%%%%%%%%%%%%%%%%%

{\em Let $\Theory{}$ be modular, and $\modfml$ be a law. For all $\ModelSet'\in\ErasureModels{\ModelSet}{\modfml}$ such that $\Mvalid{\model}\Theory{}$ for every $\model\in\ModelSet$, there is $\Theory{}'\in\ErasureSet{\Theory{}}{\modfml}$ such that $\Mvalid{\model'}\Theory{}'$ for every $\model'\in\ModelSet'$.}

\myskip

\begin{lemma}\label{LemmaMonotony}
$\Theory{}\PDLvalid\Theory{}'$.
\end{lemma}
\begin{proof}
Let $\Theory{}$ be an action theory, and $\Theory{}'\in\ErasureSet{\Theory{}}{\modfml}$, for $\modfml$ a law. We analyze each case.

\myskip

Let $\modfml$ be of the form $\antec\imp\poss\act\top$, for some $\antec\in\Fml$. Then $\Theory{}'$ is such that
\[
\Theory{}'=(\Theory{}\setminus\EXE{}{\act})\cup\{(\antec_{i}\et\neg(\term\et\antec_{\Set}))\imp\poss\act\top : \antec_{i}\imp\poss\act\top\in\EXE{}{\act}\}
\]
where $\term\in\IP{\STAT{}{}\et\antec}$ and  $\antec_{\Set}=\bigwedge_{\stackscript{\prp_{i}\in\compatterm}{\prp_{i}\in\Set}}\prp_{i}\et\bigwedge_{\stackscript{\prp_{i}\in\compatterm}{\prp_{i}\notin\Set}}\neg\prp_{i}$, for some $\Set\subseteq\compatterm$.

	Let $\model=\TransStruct{}$ be such that $\Mvalid{\model}{\Theory{}}$. It is enough to show that $\model$ is a model of the new laws. For every $(\antec_{i}\et\neg(\term\et\antec_{\Set}))\imp\poss\act\top$, for every $w\in\Worlds$, if $\wMvalid{w}{\model}\antec_{i}\et\neg(\term\et\antec_{\Set})$, then $\wMvalid{w}{\model}\antec_{i}$. Because $\Theory{}\PDLvalid\antec_{i}\imp\poss\act\top$, $\Mvalid{\model}\antec_{i}\imp\poss\act\top$, and then $\ract(w)\neq\emptyset$.
	
	Hence $\Mvalid{\model}\Theory{}'$.

\myskip

	Let now $\modfml$ have the form $\antec\imp\nec\act\conseq$, for $\antec,\conseq\in\Fml$. Then $\Theory{}'$ is such that
\[
\Theory{}'=\begin{array}{l}
		(\Theory{}\setminus\EFF{-}{\act})\ \cup \\[0.2cm]
		\{(\antec_{i}\et\neg(\term\et\antec_{\Set}))\imp\nec\act\conseq_{i} : 
			\antec_{i}\imp\nec\act\conseq_{i} \in\EFF{-}{\act}\}\ \cup \\[0.2cm]
		\{(\antec_{i}\et\term\et\antec_{\Set})
						\imp\nec\act(\conseq_{i}\ou\term') : 
						\antec_{i}\imp\nec\act\conseq_{i} \in\EFF{-}{\act}\}\ \cup \\[0.2cm]
		\left\{\begin{array}{cl}
					& \lit\in L, \text{ for some } L\subseteq\Lit\ \suchthat \\
	(\term\et\antec_{\Set}\et\lit)\imp\nec\act(\conseq\ou\lit) : & \STAT{}{}\not\vdash(\term'\et\bigwedge_{\lit\in L}\lit)\imp\bot, \text{ and } \lit\in\term'\\
					& \text{ or } \Theory{}\notPDLtheorem(\term\et\antec_{\Set}\et\lit)\imp\nec\act\neg\lit
			\end{array}
		\right\}
		\end{array}
\]
where $\EFF{-}{\act}=\bigcup_{1\leq i\leq n}(\EFF{\antec,\conseq}{\act})_{i}$, $\term\in\IP{\STAT{}{}\et\antec}$, $\antec_{\Set}=\bigwedge_{\stackscript{\prp_{i}\in\compatterm}{\prp_{i}\in\Set}}\prp_{i}\et\bigwedge_{\stackscript{\prp_{i}\in\compatterm}{\prp_{i}\notin\Set}}\neg\prp_{i}$, for some $\Set\subseteq\compatterm$, and $\term'\in\IP{\STAT{}{}\et\neg\conseq}$.

	Let $\model=\TransStruct{}$ be such that $\Mvalid{\model}{\Theory{}}$. It is enough to show that $\model$ is a model of the added laws. Given $(\antec_{i}\et\neg(\term\et\antec_{\Set}))\imp\nec\act\conseq_{i}$, for every $w\in\Worlds$, if $\wMvalid{w}{\model}\antec_{i}\et\neg(\term\et\antec_{\Set})$, then $\wMvalid{w}{\model}\antec_{i}$. Because $\Theory{}\PDLvalid\antec_{i}\imp\nec\act\conseq_{i}$, $\Mvalid{\model}\antec_{i}\imp\nec\act\conseq_{i}$, and then $\wMvalid{w'}{\model}\conseq_{i}$ for every $w'\in\Worlds$ such that $(w,w')\in\ract$.

	For $(\antec_{i}\et\term\et\antec_{\Set})	\imp\nec\act(\conseq_{i}\ou\term')$, for every $w\in\Worlds$, if $\wMvalid{w}{\model}\antec_{i}\et\term\et\antec_{\Set}$, then again $\wMvalid{w'}{\model}\conseq_{i}$ for every $w'\in\Worlds$ such that $(w,w')\in\ract$.

	Now, given $(\term\et\antec_{\Set}\et\lit)\imp\nec\act(\conseq\ou\lit)$, for every $w\in\Worlds$, if $\wMvalid{w}{\model}\term\et\antec_{\Set}\et\lit$, then $\wMvalid{w}{\model}\term$, and then $\wMvalid{w}{\model}\antec$. Since $\Theory{}\PDLvalid\antec\imp\nec\act\conseq$, we have $\Mvalid{\model}\antec\imp\nec\act\conseq$, and then $\wMvalid{w'}{\model}\conseq$ for every $w'\in\Worlds$ such that $(w,w')\in\ract$.
	
	Hence $\Mvalid{\model}\Theory{}'$.
	
\myskip

	Let $\modfml$ be a propositional $\fml$. Then $\Theory{}'$ is such that
\[
\Theory{}'=\begin{array}{l}
		((\Theory{}\setminus\STAT{}{}) \cup \STAT{-}{})\setminus\EXE{}{\act}\  \cup \\
		\{(\antec_{i}\et\fml)\imp\poss\act\top : 
		\antec_{i}\imp\poss\act\top \in\EXE{}{\act}\}\ \cup \\
		   \{\neg\fml\imp\nec\act\bot\}
		\end{array}
\]
for some $\STAT{-}{}\in\STAT{}{}\propcontract\fml$.

Let $\model=\TransStruct{}$ be such that $\Mvalid{\model}{\Theory{}}$. It suffices to show that $\model$ satisfies the added laws.

	Since we assume $\propcontract$ behaves like a classical contraction operator, like \eg\ Katsuno and Mendelzon's~\cite{KatsunoMendelzon92}, we have $\CPLvalid\STAT{}{}\imp\STAT{-}{}$, and then, because $\Mvalid{\model}\STAT{}{}$, we have $\Mvalid{\model}\STAT{-}{}$.

	Now given $(\antec_{i}\et\fml)\imp\poss\act\top$, for every $w\in\Worlds$, if $\wMvalid{w}{\model}\antec_{i}\et\fml$, then $\wMvalid{w}{\model}\antec_{i}$, and because $\Mvalid{\model}\antec_{i}\imp\poss\act\top$, we have $\ract(w)\neq\emptyset$.
	
	Finally, for $\neg\fml\imp\nec\act\bot$, because $\Mvalid{\model}\fml$, $\model$ trivially satisfies $\neg\fml\imp\nec\act\bot$.
	
	Hence, $\Mvalid{\model}\Theory{}'$.
\end{proof}

\bigskip

\noindent{\bf Proof of Theorem~\ref{TheoremCompleteness}}

\bigskip

	Let $\ModelSet=\{\model : \Mvalid{\model}\Theory{}\}$, and $\ModelSet'\in\ErasureModels{\ModelSet}{\modfml}$. We show that there is $\Theory{}'\in\ErasureSet{\Theory{}}{\modfml}$ such that $\Mvalid{\model'}\Theory{}'$ for every $\model'\in\ModelSet'$.
	
\myskip

	By definition, each $\model'\in\ModelSet'$ is such that either $\Mvalid{\model'}\Theory{}$ or $\notMvalid{\model'}\modfml$. Because $\ErasureSet{\Theory{}}{\modfml}\neq\emptyset$, there must be $\Theory{}'\in\ErasureSet{\Theory{}}{\modfml}$. If $\Mvalid{\model'}\Theory{}$, by Lemma~\ref{LemmaMonotony} $\Mvalid{\model'}\Theory{}'$ and we are done. Let's then suppose that $\notMvalid{\model'}\modfml$.	We analyze each case.

\bigskip

	Let $\modfml$ have the form $\antec\imp\poss\act\top$ for some $\antec\in\Fml$. Then $\model'=\TransStructP{}$, where $\Worlds'=\Worlds$, $\AccRel'=\AccRel\setminus\AccRel_{\act}^{\antec}$, with $\AccRel_{\act}^{\antec}=\{(w,w') : \wMvalid{w}{\model}\antec \text{ and }(w,w')\in\ract\}$, for some $\model\in\ModelSet$.
	
	Let $u\in\Worlds'$ be such that $\notwMvalid{u}{\model'}\antec\imp\poss\act\top$, \ie, $\wMvalid{u}{\model'}\antec$ and $\AccRel'_{\act}(u)=\emptyset$.
	
	Because $u\sat\antec$, there must be $v\in\base{\antec,\Worlds'}$ such that $v\subseteq u$.  Let $\term=\bigwedge_{\lit\in v}\lit$. Clearly $\term$ is a prime implicant of $\STAT{}{}\et\antec$. Let also $\antec_{\Set}=\bigwedge_{\lit\in u\setminus v}\lit$, and consider
\[
\Theory{}'=(\Theory{}\setminus\EXE{}{\act}) \cup
	\{(\antec_{i}\et\neg(\term\et\antec_{\Set}))\imp\poss\act\top : 
			\antec_{i}\imp\poss\act\top\in\EXE{}{\act}\}
\]
(Clearly, $\Theory{}'$ is a theory produced by Algorithm~\ref{ExecOperator}.)

\myskip

	It is enough to show that $\model'$ is a model of the new added laws. Given $(\antec_{i}\et\neg(\term\et\antec_{\Set}))\imp\poss\act\top\in\Theory{}'$, for every $w\in\Worlds'$, if $\wMvalid{w}{\model'}\antec_{i}\et\neg(\term\et\antec_{\Set})$, then $\wMvalid{w}{\model'}\antec_{i}$, from what it follows $\wMvalid{w}{\model}\antec_{i}$. Because $\Mvalid{\model}\antec_{i}\imp\poss\act\top$, there is $w'\in\Worlds$ such that $w'\in\ract(w)$. We need to show that $(w,w')\in\AccRel'_{\act}$. If $\notwMvalid{w}{\model}\antec$, then $\AccRel_{\act}^{\antec}=\emptyset$, and $(w,w')\in\AccRel'_{\act}$. If $\wMvalid{w}{\model}\antec$, either $w=u$, and then from $\wMvalid{u}{\model'}\term\et\antec_{\Set}$ we conclude $\wMvalid{u}{\model'}(\antec_{i}\et\neg(\term\et\antec_{\Set}))\imp\poss\act\top$, or $w\neq u$ and then we must have $(w,w')\in\AccRel'_{\act}$, otherwise there is $\AccRelS_{\act}^{\antec}\subset\AccRel_{\act}^{\antec}$ such that $\AccRel\symdiff(\AccRel\setminus\AccRelS_{\act}^{\antec})\subset\AccRel\symdiff(\AccRel\setminus\AccRel_{\act}^{\antec})$, and then $\model''=\tuple{\Worlds',\AccRel\setminus\AccRelS_{\act}^{\antec}}$ is such that $\notMvalid{\model''}\antec\imp\poss\act\top$ and $\model''\submodel{\model}\model'$, a contradiction because $\model'$ is minimal \wrt\ $\submodel{\model}$. Thus $(w,w')\in\AccRel'_{\act}$, and then $\wMvalid{w}{\model'}\poss\act\top$. Hence $\Mvalid{\model'}\Theory{}'$.
	
\myskip

	Now let $\modfml$ be of the form $\antec\imp\nec\act\conseq$, for $\antec,\conseq$ both Boolean. Then $\model'=\TransStructP{}$, where $\Worlds'=\Worlds$, $\AccRel'=\AccRel\cup\AccRel_{\act}^{\antec,\neg\conseq}$, with 
\[
\AccRel_{\act}^{\antec,\neg\conseq}=\{(w,w') : w'\in\RelTarget{w,\antec\imp\nec\act\conseq,\model,\ModelSet}\}
\]
for some $\model=\TransStruct{}\in\ModelSet$.

	Let $u\in\Worlds'$ be such that $\notwMvalid{u}{\model'}\antec\imp\nec\act\conseq$. Then there is $u'\in\Worlds'$ such that $(u,u')\in\AccRel'_{\act}$ and $\notwMvalid{u'}{\model'}\conseq$. Because $u\sat\antec$, there is $v\in\base{\antec,\Worlds'}$ such that $v\subseteq u$, and as $u'\sat\neg\conseq$, there must be $v'\in\base{\neg\conseq,\Worlds'}$ such that $v'\subseteq u'$. Let $\term=\bigwedge_{\lit\in v}\lit$, $\antec_{\Set}=\bigwedge_{\lit\in u\setminus v}\lit$, and $\term'=\bigwedge_{\lit\in v'}\lit$. Clearly $\term$ (resp.\ $\term'$) is a prime implicant of $\STAT{}{}\et\antec$ (resp.\ $\STAT{}{}\et\neg\conseq$).

\myskip

	Now let $\EFF{-}{\act}=\bigcup_{1\leq i\leq n}(\EFF{\antec,\conseq}{\act})_{i}$ and let the theory
\[
\Theory{}'=\begin{array}{l}
		(\Theory{}\setminus\EFF{-}{\act})\ \cup \\[0.2cm]
		\{(\antec_{i}\et\neg(\term\et\antec_{\Set}))\imp\nec\act\conseq_{i} : 
			\antec_{i}\imp\nec\act\conseq_{i} \in\EFF{-}{\act}\}\ \cup \\[0.2cm]
		\{(\antec_{i}\et\term\et\antec_{\Set})
						\imp\nec\act(\conseq_{i}\ou\term') : 
						\antec_{i}\imp\nec\act\conseq_{i} \in\EFF{-}{\act}\}\ \cup \\[0.2cm]
		\left\{\begin{array}{cl}
					& \lit\in L, \text{ for some } L\subseteq\Lit\ \suchthat \\
	(\term\et\antec_{\Set}\et\lit)\imp\nec\act(\conseq\ou\lit) : & \STAT{}{}\not\vdash(\term'\et\bigwedge_{\lit\in L}\lit)\imp\bot, \text{ and } \lit\in\term'\\
					& \text{or } \Theory{}\notPDLtheorem(\term\et\antec_{\Set}\et\lit)\imp\nec\act\neg\lit
			\end{array}
		\right\}
		\end{array}
\]
(Clearly, $\Theory{}'$ is a theory produced by Algorithm~\ref{EffectOperator}.)

\myskip

	In order to show that $\model'$ is a model of $\Theory{}'$, it is enough to show that it is a model of the added laws. Given $(\antec_{i}\et\neg(\term\et\antec_{\Set}))\imp\nec\act\conseq_{i}\in\Theory{}'$, for every $w\in\Worlds'$, if $\wMvalid{w}{\model'}\antec_{i}\et\neg(\term\et\antec_{\Set})$, then  $\wMvalid{w}{\model'}\antec_{i}$, and then  $\wMvalid{w}{\model}\antec_{i}$. Because $\Mvalid{\model}\antec_{i}\imp\nec\act\conseq_{i}$, $\wMvalid{w'}{\model}\conseq_{i}$ for all $w'\in\Worlds$ such that $(w,w')\in\ract$. We need to show that $\AccRel'_{\act}(w)=\ract(w)$. If $\notwMvalid{w}{\model}\antec$, then $\AccRel_{\act}^{\antec,\neg\conseq}=\emptyset$, and then $\AccRel'_{\act}(w)=\ract(w)$. If $\wMvalid{w}{\model}\antec$, then either $w=u$, and from $\wMvalid{u}{\model'}\term\et\antec_{\Set}$ we conclude $\wMvalid{u}{\model'}(\antec_{i}\et\neg(\term\et\antec_{\Set}))\imp\nec\act\conseq_{i}$, or $w\neq u$, and then we must have $\AccRel_{\act}^{\antec,\neg\conseq}=\emptyset$, otherwise there would be $\AccRelS_{\act}^{\antec,\neg\conseq}\subset\AccRel_{\act}^{\antec,\neg\conseq}$ such that $\AccRel\symdiff(\AccRel\cup\AccRelS_{\act}^{\antec,\neg\conseq})\subset\AccRel\symdiff(\AccRel\cup\AccRel_{\act}^{\antec,\neg\conseq})$, and then $\model''=\tuple{\Worlds',\AccRel\cup\AccRelS_{\act}^{\antec,\neg\conseq}}$ would be such that $\notMvalid{\model''}\antec\imp\nec\act\conseq$ and $\model''\submodel{\model}\model'$, a contradiction since $\model'$ is minimal \wrt\ $\submodel{\model}$. Hence  $\AccRel'_{\act}(w)=\ract(w)$, and $\wMvalid{w'}{\model'}\conseq_{i}$ for all $w'$ such that $(w,w')\in\AccRel'_{\act}$.

\myskip

	Now, given $(\antec_{i}\et\term\et\antec_{\Set})\imp\nec\act(\conseq_{i}\ou\term')$, for every $w\in\Worlds'$, if $\wMvalid{w}{\model'}\antec_{i}\et\term\et\antec_{\Set}$, then $\wMvalid{w}{\model'}\antec_{i}$, and then $\wMvalid{w}{\model}\antec_{i}$. Because, $\Mvalid{\model}\antec_{i}\imp\nec\act\conseq_{i}$, we have $\wMvalid{w'}{\model}\conseq_{i}$ for all $w'\in\Worlds$ such that $(w,w')\in\ract$, and then $\wMvalid{w'}{\model'}\conseq_{i}$ for every $w'\in\Worlds'$ such that $(w,w')\in\AccRel'_{\act}\setminus\AccRel_{\act}^{\antec,\neg\conseq}$. Now, given $(w,w')\in\AccRel_{\act}^{\antec,\neg\conseq}$, $\wMvalid{w'}{\model'}\term'$, and the result follows.
	
\myskip

	Now, for each $(\term\et\antec_{\Set}\et\lit)\imp\nec\act(\conseq\ou\lit)$, for every $w\in\Worlds'$, if $\wMvalid{w}{\model'}\term\et\antec_{\Set}\et\lit$, then $\wMvalid{w}{\model'}\antec$, and then $\wMvalid{w}{\model}\antec$. Because $\Mvalid{\model}\antec\imp\nec\act\conseq$, we have $\wMvalid{w'}{\model}\conseq$ for every $w'\in\Worlds$ such that $(w,w')\in\ract$, and then $\wMvalid{w'}{\model'}\conseq$ for all $w'\in\Worlds'$ such that $(w,w')\in\AccRel'_{\act}\setminus\AccRel_{\act}^{\antec,\neg\conseq}$. It remains to show that $\wMvalid{w'}{\model'}\lit$ for every $w'\in\Worlds'$ such that $(w,w')\in\AccRel_{\act}^{\antec,\neg\conseq}$. Since $\model'$ is minimal, it is enough to show that $\wMvalid{u'}{\model'}\lit$ for every $\lit\in\Lit$ such that $\wMvalid{u}{\model'}\term\et\antec_{\Set}\et\lit$. If $\lit\in\term'$, the result follows. Otherwise, suppose $\notwMvalid{u'}{\model'}\lit$. Then
\begin{itemize}
\item either $\neg\lit\in\term'$, then $\term'$ and $\lit$ are unsatisfiable, and in this case Algorithm~\ref{EffectOperator} has not put the law $(\term\et\antec_{\Set}\et\lit)\imp\nec\act(\conseq\ou\lit)$ in $\Theory{}'$, a contradiction;
\item or $\neg\lit\in u'\setminus v'$. In this case, there is a valuation $u''=(u'\setminus\{\neg\lit\})\cup\{\lit\}$ such that $u''\not\sat\conseq$. We must have $u''\in\Worlds'$, otherwise there will be $L'=\{\lit_{i} : \lit_{i}\in u''\}$ such that $\Theory{}\PDLvalid(\term'\et\bigwedge_{\lit_{i}\in L'}\lit_{i})\imp\bot$, and, because $\Theory{}$ is modular, $\STAT{}{}\CPLvalid(\term'\et\bigwedge_{\lit_{i}\in L'}\lit_{i})\imp\bot$, and then Algorithm~\ref{EffectOperator} has not put the law $(\term\et\antec_{\Set}\et\lit)\imp\nec\act(\conseq\ou\lit)$ in $\Theory{}'$, a contradiction. Then $u''\in\Worlds'$, and moreover $u''\notin\AccRel_{\act}^{\antec,\neg\conseq}(u)$, otherwise $\model'$ is not minimal. As $u''\setminus u\subset u'\setminus u$, the only reason why $u''\notin\AccRel_{\act}^{\antec,\neg\conseq}(u)$ is that there is $\lit'\in u\cap u''$ such that 
$\Mvalid{\model_{i}}\bigwedge_{\lit_{j}\in u}\lit_{j}\imp\nec\act\neg\lit'$ for every $\model_{i}\in\ModelSet$ if and only if $\lit'\notin\val'$ for any $\val'\in\base{\neg\conseq,W'}$ such that $\val'\subseteq u''$. Clearly $\lit'=\lit$, and because $\lit\notin\term'$, we have $\Mvalid{\model_{i}}\bigwedge_{\lit_{j}\in u}\lit_{j}\imp\nec\act\neg\lit$ for every $\model_{i}\in\ModelSet$. Then $\Theory{}\PDLvalid(\term\et\antec_{\Set}\et\lit)\imp\nec\act\neg\lit$, and then Algorithm~\ref{EffectOperator} has not put the law $(\term\et\antec_{\Set}\et\lit)\imp\nec\act(\conseq\ou\lit)$ in $\Theory{}'$, a contradiction.
\end{itemize}
Hence we have $\wMvalid{w'}{\model'}\conseq\ou\lit$ for every $w'\in\Worlds'$ such that $(w,w')\in\AccRel'_{\act}$.

	Putting the above results together, we get $\Mvalid{\model'}\Theory{}'$.

\myskip

	Let now $\modfml$ be some propositional $\fml$. Then $\model'=\TransStructP{}$, where $\Worlds\subseteq\Worlds'$, $\AccRel'=\AccRel$, is minimal \wrt\ $\submodel{\model}$, \ie, $\Worlds'$ is a minimum superset of $\Worlds$ such that there is $u\in\Worlds'$ with $u\notsat\fml$. Because we have assumed the syntactical classical contraction operator is sound and complete \wrt\ its semantics and is moreover minimal, then there must be $\STAT{-}{}\in\STAT{}{}\propcontract\fml$ such that $\Worlds'=\valuations{\STAT{-}{}}$. Hence $\Mvalid{\model'}\STAT{-}{}$.
	
	Because $\AccRel'=\AccRel$, every effect law of $\Theory{}$ remains true in $\model'$.
	
	Now, let
\[
\Theory{}'=\begin{array}{l}
		((\Theory{}\setminus\STAT{}{}) \cup \STAT{-}{})\setminus\EXE{}{\act}\  \cup \\
		\{(\antec_{i}\et\fml)\imp\poss\act\top : 
		\antec_{i}\imp\poss\act\top \in\EXE{}{\act}\}\ \cup \\
		   \{\neg\fml\imp\nec\act\bot\}
		\end{array}
\]
(Clearly, $\Theory{}'$ is a theory produced by Algorithm~\ref{StatOperator}.)

	For every $(\antec_{i}\et\fml)\imp\poss\act\top\in\Theory{}'$ and every $w\in\Worlds'$, if $\wMvalid{w}{\model'}\antec_{i}\et\fml$, then $\ract(w)\neq\emptyset$, because $\wMvalid{w}{\model}\antec_{i}\imp\poss\act\top$. Given $\neg\fml\imp\nec\act\bot$, for every $w\in\Worlds'$, if $\wMvalid{w}{\model'}\neg\fml$, then $w=u$, and $\ract(w)=\emptyset$.
	
	Putting all these results together, we have $\Mvalid{\model'}\Theory{}'$.{\hfill\qed}

%%%%%%%%%%%%%%%%%%%%%%%%%%%%%%%%%%%%%%%%%%%%%%%%%%%
\section{Proof of Theorem~\ref{TheoremSoundness}}\label{ProofTheoremSoundness}
%%%%%%%%%%%%%%%%%%%%%%%%%%%%%%%%%%%%%%%%%%%%%%%%%%%

{\em Let $\Theory{}$ be modular, $\modfml$ a law, and $\Theory{}'\in\ErasureSet{\Theory{}}{\modfml}$. For all $\model'$ such that $\Mvalid{\model'}\Theory{}'$, there is $\ModelSet'\in\ErasureModels{\ModelSet}{\modfml}$ such that $\model'\in\ModelSet'$ and $\Mvalid{\model}\Theory{}$ for every $\model\in\ModelSet$.}

\myskip

\begin{lemma}\label{PreservationModularity}
Let $\modfml$ be a law. If $\Theory{}$ is modular, then every $\Theory{}'\in\ErasureSet{\Theory{}}{\modfml}$ is modular.
\end{lemma}
\begin{proof}
Let $\modfml$ be nonclassical, and suppose there is $\Theory{}'\in\ErasureSet{\Theory{}}{\modfml}$ such that $\Theory{}'$ is not modular. Then there is some $\fml'\in\Fml$ such that $\Theory{}'\PDLvalid\fml'$ and ${\STAT{}{}}'\notCPLvalid\fml'$, where ${\STAT{}{}}'$ is the set of static laws in $\Theory{}'$. By Lemma~\ref{LemmaMonotony}, $\Theory{}\PDLvalid\Theory{}'$, and then we have $\Theory{}\PDLvalid\fml'$. Because $\modfml$ is nonclassical, ${\STAT{}{}}'=\STAT{}{}$. Thus $\STAT{}{}\notCPLvalid\fml'$, and hence $\Theory{}$ is not modular.\\

Let now $\modfml$ be some $\fml\in\Fml$. Then
\[
\Theory{}'=\begin{array}{l}
		((\Theory{}\setminus\STAT{}{}) \cup \STAT{-}{})\setminus\EXE{}{\act}\  \cup \\
		\{(\antec_{i}\et\fml)\imp\poss\act\top : 
		\antec_{i}\imp\poss\act\top \in\EXE{}{\act}\}\ \cup \\
		   \{\neg\fml\imp\nec\act\bot\}
		\end{array}
\]
for some $\STAT{-}{}\in\STAT{}{}\propcontract\fml$.

Suppose $\Theory{}$ is modular, and let $\fml'\in\Fml$ be such that $\Theory{}'\PDLvalid\fml'$ and $\STAT{-}{}\notCPLvalid\fml'$.

As $\STAT{-}{}\notCPLvalid\fml'$, there is $\val\in\valuations{\STAT{-}{}}$ such that $\val\notsat\fml'$. If $\val\in\valuations{\STAT{}{}}$, then $\STAT{}{}\notCPLvalid\fml'$, and as $\Theory{}$ is modular, $\Theory{}\notPDLvalid\fml'$. By Lemma~\ref{LemmaMonotony}, $\Theory{}\PDLvalid\Theory{}'$, and we have $\Theory{}'\notPDLvalid\fml'$, a contradiction. Hence $\val\notin\valuations{\STAT{}{}}$. Moreover, we must have $\val\notsat\fml$, otherwise $\propcontract$ has not worked as expected.

Let $\model=\TransStruct{}$ be such that $\Mvalid{\model}\Theory{}'$. (We extend $\model$ to another model of $\Theory{}'$.) Let $\model'=\TransStructP{}$ be such that $\Worlds'=\Worlds\cup\{\val\}$ and $\AccRel'=\AccRel$. To show that $\model'$ is a model of $\Theory{}'$, it suffices to show that $\val$ satisfies every law in $\Theory{}'$. As $\val\in\valuations{\STAT{-}{}}$, $\wMvalid{v}{\model'}\STAT{-}{}$. Given $\neg\fml\imp\nec\act\bot\in\Theory{}'$, as $\val\notsat\fml$ and $\ract'(\val)=\emptyset$, $\wMvalid{\val}{\model'}\neg\fml\imp\nec\act\bot$. Now, for every $\antec_{i}\imp\nec\act\conseq_{i}\in\Theory{}'$, if $\wMvalid{\val}{\model'}\antec_{i}$, then we trivially have $\wMvalid{\val'}{\model'}\conseq_{i}$ for every $\val'$ such that $(\val,\val')\in\ract'$. Finally, given $(\antec_{i}\et\fml)\imp\poss\act\top\in\Theory{}'$, as $\val\notsat\fml$, the formula trivially holds in $\val$. Hence $\Mvalid{\model'}\Theory{}'$, and because there is $\val\in\Worlds'$ such that $\notwMvalid{\val}{\model'}\fml'$, we have $\Theory{}'\notPDLvalid\fml'$, a contradiction. Hence for all $\fml'\in\Fml$ such that $\Theory{}'\PDLvalid\fml'$, $\STAT{-}{}\CPLvalid\fml'$, and then $\Theory{}'$ is modular.
\end{proof}

\myskip

\begin{lemma}\label{LemmaMinimalExtension}
If $\model_{\bigM}=\tuple{\Worlds_{\bigM},\AccRel_{\bigM}}$ is a model of $\Theory{}$, then for every $\model=\TransStruct{}$ such that $\Mvalid{\model}\Theory{}$ there is a minimal (\wrt\ set inclusion) extension $\AccRel'\subseteq\AccRel_{\bigM}\setminus\AccRel$ such that $\model'=\tuple{\valuations{\STAT{}{}},\AccRel\cup\AccRel'}$ is a model of $\Theory{}$.
\end{lemma}
\begin{proof}
Let $\model_{\bigM}=\tuple{\Worlds_{\bigM},\AccRel_{\bigM}}$ be a model of $\Theory{}$, and let $\model=\TransStruct{}$ be such that $\Mvalid{\model}\Theory{}$. Consider $\model'=\tuple{\valuations{\STAT{}{}},\AccRel}$. If $\Mvalid{\model'}\Theory{}$, we have $\AccRel'=\emptyset\subseteq\AccRel_{\bigM}\setminus\AccRel$ that is minimal. Suppose then $\notMvalid{\model'}\Theory{}$. We extend $\model'$ to a model of $\Theory{}$ that is a minimal extension of $\model$. As $\notMvalid{\model'}\Theory{}$, there is $\val\in\valuations{\STAT{}{}}\setminus\Worlds$ such that $\notwMvalid{\val}{\model'}\Theory{}$. Then there is $\modfml\in\Theory{}$ such that $\notwMvalid{\val}{\model'}\modfml$. If $\modfml$ is some $\fml\in\Fml$, as $\val\in\Worlds_{\bigM}$, $\model_{\bigM}$ is not a model of $\Theory{}$. If $\modfml$ is of the form $\antec\imp\nec\act\conseq$, for $\antec,\conseq\in\Fml$, there is $\val'\in\valuations{\STAT{}{}}$ such that $(\val,\val')\in\ract$ and $\val'\notsat\conseq$, a contradiction since $\ract(\val)=\emptyset$. Let now $\modfml$ have the form $\antec\imp\poss\act\top$ for some $\antec\in\Fml$. Then $\wMvalid{\val}{\model'}\antec$. As $\val\in\Worlds_{\bigM}$, if $\notwMvalid{\val}{\model_{\bigM}}\antec\imp\poss\act\top$, then $\notMvalid{\model_{\bigM}}\Theory{}$. Hence, ${\AccRel_{\bigM}}_{\act}(\val)\neq\emptyset$. Thus taking any $(\val,\val')\in{\AccRel_{\bigM}}_{\act}$ gives us a minimal $\AccRel'=\{(\val,\val')\}$ such that $\model''=\tuple{\valuations{\STAT{}{}},\AccRel\cup\AccRel'}$ is a model of $\Theory{}$.
\end{proof}

\myskip

\begin{lemma}\label{LemmaOnlyModelsValS}
Let $\Theory{}$ be modular, and $\modfml$ be a law. Then $\Theory{}\PDLvalid\modfml$ if and only if every $\model'=\tuple{\valuations{\STAT{}{}},\AccRel'}$ such that $\Mvalid{\TransStruct{}}\Theory{}$ and $\AccRel\subseteq\AccRel'$ is a model of $\modfml$.
\end{lemma}
\begin{proof}

\noindent($\Rightarrow$): Straightforward, as $\Theory{}\PDLvalid\modfml$ implies $\Mvalid{\model}\modfml$ for every $\model$ such that $\Mvalid{\model}\Theory{}$, in particular for those that are extensions of some model of $\Theory{}$.\\

\noindent($\Leftarrow$): Suppose $\Theory{}\notPDLvalid\modfml$. Then there is $\model=\TransStruct{}$ such that $\Mvalid{\model}\Theory{}$ and $\notMvalid{\model}\modfml$. As $\Theory{}$ is modular, the big model $\model_{\bigM}=\tuple{\Worlds_{\bigM},\AccRel_{\bigM}}$ of $\Theory{}$ is a model of $\Theory{}$. Then by Lemma~\ref{LemmaMinimalExtension} there is a minimal extension $\AccRel'$ of $\AccRel$ \wrt\ $\AccRel_{\bigM}$ such that $\model'=\tuple{\valuations{\STAT{}{}},\AccRel\cup\AccRel'}$ is a model of $\Theory{}$. Because $\notMvalid{\model}\modfml$, there is $w\in\Worlds$ such that $\notwMvalid{w}{\model}\modfml$. If $\modfml$ is some propositional $\fml\in\Fml$ or an effect law, any extension $\model'$ of $\model$ is such that $\notwMvalid{w}{\model'}\modfml$. If $\modfml$ is of the form $\antec\imp\poss\act\top$, then $\wMvalid{w}{\model}\fml$ and $\ract(w)=\emptyset$. As any extension of $\model$ is such that $(u,v)\in\AccRel'$ if and only if $u\in\valuations{\STAT{}{}}\setminus\Worlds$, 
only worlds other than those in $\Worlds$ get a new leaving arrow. Thus $(\AccRel\cup\AccRel')_{\act}(w)=\emptyset$, and then $\notwMvalid{w}{\model'}\modfml$.
\end{proof}

\myskip

\begin{lemma}\label{LemmaComesFromSemantics}
Let $\Theory{}$ be modular, $\modfml$ a law, and $\Theory{}'\in\ErasureSet{\Theory{}}{\modfml}$. If $\model'=\tuple{\valuations{{\STAT{}{}}'},\AccRel'}$ is a model of $\Theory{}'$, then there is $\ModelSet=\{\model : \model=\tuple{\valuations{\STAT{}{}},\AccRel} \text{ and } \Mvalid{\model}\Theory{}\}$ such that $\model'\in\ModelSet'$ for some $\ModelSet'\in\ErasureModels{\ModelSet}{\modfml}$.
\end{lemma}
\begin{proof}
Let $\model'=\tuple{\valuations{{\STAT{}{}}'},\AccRel'}$ be such that $\Mvalid{\model'}\Theory{}'$. If $\Mvalid{\model'}\Theory{}$, the result follows. Let's suppose then $\notMvalid{\model'}\Theory{}$. We analyze each case.

\myskip

Let $\modfml$ be of the form $\antec\imp\poss\act\top$, for some $\antec\in\Fml$. Let $\ModelSet=\{\model : \model=\tuple{\valuations{\STAT{}{}},\AccRel}\}$. As $\Theory{}$ is modular, by Lemmas~\ref{LemmaMinimalExtension} and~\ref{LemmaOnlyModelsValS}, $\ModelSet$ is non-empty and contains only models of $\Theory{}$.

Suppose $\model'$ is not a minimal model of $\Theory{}'$, \ie, there is $\model''$ such that $\model''\submodel{\model}\model'$ for some $\model\in\ModelSet$. Then $\model'$ and $\model''$ differ only in the executability of $\act$ in a given $\antec$-world, \viz\ a $\term\et\antec_{\Set}$-context, for some $\term\in\IP{\STAT{}{}\et\antec}$ and $\antec_{\Set}=\bigwedge_{\stackscript{\prp_{i}\in\compatterm}{\prp_{i}\in\Set}}\prp_{i}\et\bigwedge_{\stackscript{\prp_{i}\in\compatterm}{\prp_{i}\notin\Set}}\neg\prp_{i}$ such that $\Set\subseteq\compatterm$. Because $\notMvalid{\model'}(\term\et\antec_{\Set})\imp\poss\act\top$, we must have $\Mvalid{\model''}(\term\et\antec_{\Set})\imp\poss\act\top$ and then $\Mvalid{\model''}\Theory{}$. Hence $\model'$ is minimal \wrt\ $\submodel{\model}$.

When contracting executability laws, ${\STAT{}{}}'=\STAT{}{}$. Hence taking the right $\AccRel$ and a minimal $\ract^{\antec}$ such that $\model=\tuple{\valuations{\STAT{}{}},\AccRel}$ and $\AccRel'=\AccRel\setminus\ract^{\antec}$, for some $\ract^{\antec}\subseteq\{(w,w') : \wMvalid{w}{\model}\antec \text{ and } (w,w')\in\ract\}$, we construct $\ModelSet'=\ModelSet\cup\{\model'\}\in\ErasureModels{\ModelSet}{\antec\imp\poss\act\top}$.

\myskip

Let $\modfml$ be of the form $\antec\imp\nec\act\conseq$, for $\antec,\conseq\in\Fml$. Let $\ModelSet=\{\model : \model=\tuple{\valuations{\STAT{}{}},\AccRel}\}$. As $\Theory{}$ is modular, by Lemmas~\ref{LemmaMinimalExtension} and~\ref{LemmaOnlyModelsValS}, $\ModelSet$ is non-empty and contains only models of $\Theory{}$.

We claim that $\model'$ has only one arrow linking a $\antec$-world, \viz\ a context $\antec_{i}\et\term\et\antec_{\Set}$ for some $\term\in\IP{\STAT{}{}\et\antec}$ and $\antec_{\Set}=\bigwedge_{\stackscript{\prp_{i}\in\compatterm}{\prp_{i}\in\Set}}\prp_{i}\et\bigwedge_{\stackscript{\prp_{i}\in\compatterm}{\prp_{i}\notin\Set}}\neg\prp_{i}$, such that $\Set\subseteq\compatterm$, to a $\term'$-world, where $\term'\in\IP{\STAT{}{}\et\neg\conseq}$. The proof is as follows: given $\lit\in\Lit$ such that $\lit$ holds in this $\antec_{i}\et\term\et\antec_{\Set}$-world
\begin{itemize}
\item if $(\term\et\antec_{\Set}\et\lit)\imp\nec\act(\conseq\ou\lit)\notin\Theory{}'$, then $\lit\notin\term'$ and $\Theory{}\PDLvalid(\term\et\antec_{\Set}\et\lit)\imp\nec\act\neg\lit$. Then this world has only $\neg\lit$-successors.
\item if $(\term\et\antec_{\Set}\et\lit)\imp\nec\act(\conseq\ou\lit)\in\Theory{}'$, then every $\term'$-successor is an $\lit$-world.
\end{itemize}
By successively applying this reasoning to each $\lit$ that holds in this $\antec_{i}\et\term\et\antec_{\Set}$-world, we will end up with only one $\term'$-successor.

Suppose now that $\model'$ is not a minimal model of $\Theory{}'$, \ie, there is $\model''$ such that $\Mvalid{\model''}\Theory{}'$ and $\model''\submodel{\model}\model'$ for some $\model\in\ModelSet$. Then $\model'$ and $\model''$ differ only in the effects on that $\antec_{i}\et\term\et\antec_{\Set}$-world: $\model''$ has no arrow linking it to a $\term'$-world. Then we have $\Mvalid{\model''}(\antec_{i}\et\term\et\antec_{\Set})\imp\nec\act\conseq_{i}$, and then $\Mvalid{\model''}\Theory{}$. Hence $\model'$ is a minimal model of $\Theory{}'$ \wrt\ $\submodel{\model}$.

When contracting effect laws, ${\STAT{}{}}'=\STAT{}{}$. Thus taking the right $\AccRel$ and a minimal $\ract^{\antec,\conseq}$ such that $\model=\tuple{\valuations{\STAT{}{}},\AccRel}$ and
$\AccRel'=\AccRel\cup\ract^{\antec,\conseq}$, for some $\ract^{\antec,\conseq}\subseteq\{(w,w') : \wMvalid{w}{\model}\antec \text{ and } w'\in\RelTarget{w,\antec\imp\nec\act\conseq,\model,\ModelSet}\}$, we construct $\ModelSet'=\ModelSet\cup\{\model'\}\in\ErasureModels{\ModelSet}{\antec\imp\nec\act\conseq}$.

\myskip

Let now $\modfml$ be $\fml$ for some $\fml\in\Fml$. Since $\Theory{}$ is modular, by Lemmas~\ref{LemmaMinimalExtension} and~\ref{LemmaOnlyModelsValS} there is $\model=\tuple{\valuations{\STAT{}{}},\AccRel}$ such that $\Mvalid{\model}\Theory{}$. We know $\valuations{\STAT{}{}}\subseteq\valuations{\STAT{-}{}}$. Because $\neg\fml\imp\nec\act\bot\in\Theory{}'$, $\ract'(\val)=\emptyset$ for every $\neg\fml$-world $\val$ added in $\model'$. Hence, because $\propcontract$ is minimal, taking $\ModelSet=\{\model\}$ gives us the result.
\end{proof}

\bigskip

\noindent{\bf Proof of Theorem~\ref{TheoremSoundness}}

From the hypothesis that $\Theory{}$ is modular and Lemma~\ref{PreservationModularity}, it follows that $\Theory{}'$ is modular, too. Then $\model'=\tuple{\valuations{{\STAT{}{}}'},\AccRel}$ is a model of $\Theory{}'$, by Lemma~\ref{LemmaOnlyModelsValS}. From this and Lemma~\ref{LemmaComesFromSemantics} the result follows.{\hfill\qed}

\end{document}